\documentclass[preprint]{elsarticle}
\pdfoutput=1
\usepackage{graphicx}

\usepackage{tikz}
\usetikzlibrary{fit}
\usetikzlibrary{calc}
\usepackage{color}
\usepackage{colortbl}

\usepackage[cmex10]{amsmath}
\usepackage{amssymb}

\DeclareMathOperator*{\argmin}{\mathrm{argmin}}
\DeclareMathOperator*{\argmax}{\mathrm{argmax}}

\usepackage{amsthm}
\newtheorem{thm}{Theorem}
\newtheorem{lemma}{Lemma}

\usepackage{algorithm}
\usepackage{algorithmic}
\usepackage{array}
\usepackage{subfig}
\usepackage{url}
\usepackage{multirow}
\usepackage{mwe}

\usepackage[normalem]{ ulem}
\definecolor{blue-violet}{rgb}{0.54, 0.17, 0.89}
\definecolor{mygreen}{rgb}{0.0, 0.5, 0.0}
\definecolor{awesome}{rgb}{1.0, 0.13, 0.32}
\definecolor{bostonuniversityred}{rgb}{0.8, 0.0, 0.0}

\newcounter{ToDo}
\newcounter{guocomm}

\definecolor{purple}{HTML}{4C2188}
\newcolumntype{p}{>{\columncolor{purple}}c}

\definecolor{blush}{HTML}{FF8AD8}
\newcolumntype{b}{>{\columncolor{blush}}c}

\definecolor{celadon}{rgb}{0.67, 0.88, 0.69}
\newcolumntype{g}{>{\columncolor{celadon}}c}

\newcommand{\tikzmark}[1]{\tikz[overlay,remember picture] \node (#1) {};}
\newcommand{\DrawBox}[4][]{%
    \tikz[overlay,remember picture]{%
        \coordinate (TopLeft)     at ($(#3)+(-0.2em,0.9em)$);
        \coordinate (BottomRight) at ($(#4)+(0.2em,-0.3em)$);
        \path (TopLeft); \pgfgetlastxy{\XCoord}{\IgnoreCoord};
        \path (BottomRight); \pgfgetlastxy{\IgnoreCoord}{\YCoord};
        \coordinate (LabelPoint) at ($(\XCoord,\YCoord)!0.5!(BottomRight)$);
        \draw [#2,#1, fill=#2] (TopLeft) rectangle (BottomRight);
    }
}

\journal{Computer Vision and Image Understanding}

\begin{document}

\begin{frontmatter}

\title{Segmentation of Subspaces in Sequential Data}

\author[csu,csiro]{Stephen Tierney}
\ead{stierney@csu.edu.au}

\author[csiro]{Yi Guo}
\ead{yi.guo@csiro.au}

\author[csu]{Junbin Gao}
\ead{jbgao@csu.edu.au}

\address[csu]{School of Computing and Mathematics, Charles Sturt University, Bathurst, NSW 2795, Australia}
\address[csiro]{Digital Productivity Flagship, CSIRO, North Ryde, NSW 2113, Australia}

\begin{abstract}
We propose Ordered Subspace Clustering (OSC) to segment data drawn from a sequentially ordered union of subspaces. Similar to Sparse Subspace Clustering (SSC) we formulate the problem as one of finding a sparse representation but include an additional penalty term to take care of sequential data. We test our method on data drawn from infrared hyper spectral, video and motion capture data. Experiments show that our method, OSC, outperforms the state of the art methods: Spatial Subspace Clustering (SpatSC), Low-Rank Representation (LRR) and SSC.
\end{abstract}


\end{frontmatter}


\section{Introduction}

In many areas such as machine learning and image processing, high dimensional data are ubiquitous. This high dimensionality has adverse effects on the computation time and memory requirements of many algorithms. Fortunately, it has been shown that high dimensional data often lie in a space of much lower dimension than the ambient space \cite{elhamifar2012sparse, vidal2011subspace}. This has motivated the creation of many dimension reduction techniques. These techniques, such as Principal Component Analysis (PCA), assume that the data belongs to a single low dimensional subspace \cite{candes2011robust}. However in reality the data often lies in a union of multiple subspaces. Therefore it is desirable to determine the subspaces in the data so that one can apply dimension reduction to each subspace separately. The problem of assigning data points to subspaces is known as subspace segmentation.

\begin{figure*}[]
\centering
\subfloat[Observed data lies in disjoint sets of subspaces.]{%
\begin{minipage}[c][1\width]{0.3\textwidth}%
\includegraphics[width=1\textwidth]{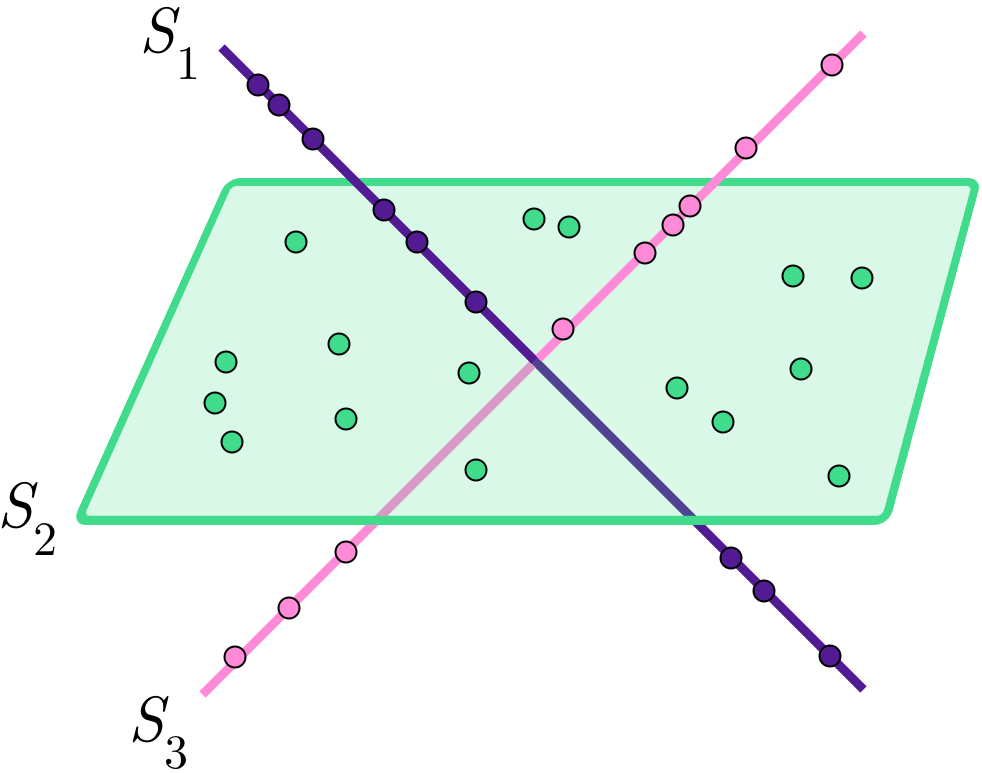}
\end{minipage}}\quad
\subfloat[The self expressive property, $\mathbf{X = XZ}$, is used to learn the subspace structure.]{\scalebox{0.85}{$\left(\begin{array}{ppbbgg}&&&&&\\
&&&&&\\
&&&&&\\
&&&&&\\
&&&&&\\
&&&&&\\
\end{array}\right)
= 
\left(\begin{array}{ppbbgg}
&&&&&\\
&&&&&\\
&&&&&\\
&&&&&\\
&&&&&\\
&&&&&\\
\end{array}\right)
\left(\begin{array}{ccccccc}
\tikzmark{left1}\\
& \tikzmark{right1}{}\\
& & \tikzmark{left2}{}\\
& & & \tikzmark{right2}{}\\
& & & & \tikzmark{left3}{}\\
& & & & & \tikzmark{right3}
\end{array}\right)$\DrawBox[thick]{purple}{left1}{right1}
\DrawBox[thick]{blush}{left2}{right2}
\DrawBox[thick]{celadon}{left3}{right3}}}\\
\subfloat[Final labels for each sample are obtained through spectral clustering on $\mathbf Z$.]{\shortstack{%
\fcolorbox{purple}{white}{\includegraphics[width=0.05\linewidth]{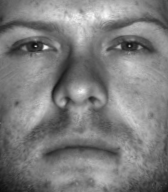}
\includegraphics[width=0.05\linewidth]{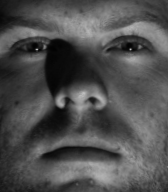}
\includegraphics[width=0.05\linewidth]{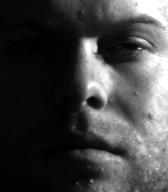}
\includegraphics[width=0.05\linewidth]{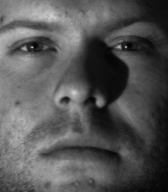}
\includegraphics[width=0.05\linewidth]{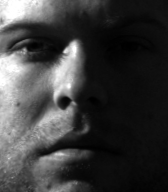}}
\fcolorbox{blush}{white}{\includegraphics[width=0.05\linewidth]{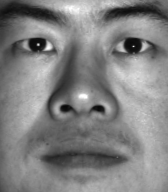}
\includegraphics[width=0.05\linewidth]{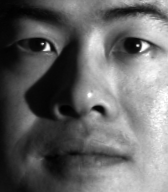}
\includegraphics[width=0.05\linewidth]{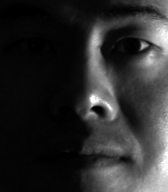}
\includegraphics[width=0.05\linewidth]{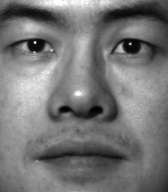}
\includegraphics[width=0.05\linewidth]{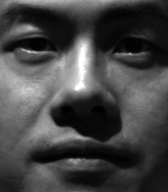}}
\fcolorbox{celadon}{white}{\includegraphics[width=0.05\linewidth]{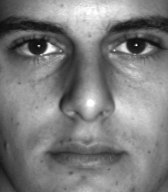}
\includegraphics[width=0.05\linewidth]{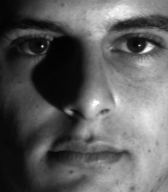}
\includegraphics[width=0.05\linewidth]{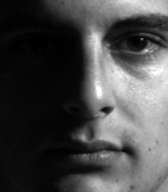}
\includegraphics[width=0.05\linewidth]{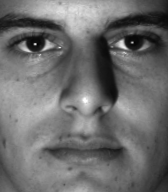}
\includegraphics[width=0.05\linewidth]{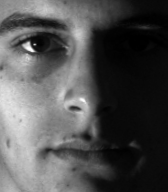}}}}
\caption{An overview of the subspace clustering procedure. The observed data such as face images are assumed to lie in a union of lower dimensional subspaces. The self expressive property is then used to learn the coefficients that best represent the subspace structure. Lastly spectral clustering is applied to the learnt coefficients, which are treated as similarities, to obtain the  final subspace labels.}
\label{Fig:sub_overview}
\end{figure*} 

Given a data matrix of $N$ observed column-wise samples $\mathbf A = [\mathbf a_1, \mathbf a_2, \dots, \mathbf a_N ]$  $\in \mathbb{R}^{D \times N}$, where $D$ is the dimension of the data, the objective of subspace segmentation is to learn corresponding subspace labels $\mathbf l = [ l_1, l_2, \dots, l_N ] \in \mathbb{N}^{N}$. Data within $\mathbf A$ is assumed to be drawn from a union of $k$ subspaces $\{S_i\}^k_{i=1}$ of dimensions $\{d_i\}^k_{i=1}$. Both the number of subspaces $k$ and the dimension of each subspace are unknown. To further complicate the problem it is rarely the case that clean data $\mathbf A$ is observed. Instead we usually observe data which has been corrupted by noise. Subspace segmentation is a difficult task since one must produce accurate results quickly while contending with numerous unknown parameters and large volume of potentially noisy data.

The use of subspace segmentation as a pre-processing method has not been limited to dimensionality reduction. For example it has been used in other applications such as image compression \cite{hong2006multiscale}, image classification \cite{zhang2013learning, DBLP:conf/dicta/BullG12}, feature extraction \cite{liu2012fixed, liu2011latent}, image segmentation \cite{yang2008unsupervised, cheng2011multi}. Furthermore state-of-the-art subspace segmentation has shown impressive results for pure segmentation tasks such as   identifying individual rigidly moving objects in video \cite{tomasi1992shape, costeira1998multibody, kanatani2002motion, jacquet2013articulated}, identifying face images of a subject under varying illumination \cite{basri2003lambertian, georghiades2001few}, segmentation of human activities \cite{zhu2014complex} and temporal video segmentation \cite{vidal2005generalized}.

This paper is concerned with a variant of subspace segmentation in which the data has a sequential structure. The data is assumed to be sampled at uniform intervals in either space or time in a single direction. For example video data which as a function of time has a sequential structure \cite{vidal2005generalized, tierney2014subspace} where it is assumed that frames are similar to their consecutive frames (neighbours) until the scene ends.  Another example is hyper-spectral drill core data \cite{elhamifar2012sparse}, which is obtained by sampling the infrared reflectance along the length of the core. The mineralogy is typically stratified meaning segments of mineral compounds congregate together \cite{Guo.Y;Gao.J;Li.F-2013, guo2014spatial}. The sequential structure implies that consecutive data samples are likely to share the same subspace label i.e.\ $l_i = l_{i+1}$, until of course a boundary point is reached.

\begin{figure*}[]
\centering
\begin{tikzpicture}[scale = 0.7]
    \node (term) at (1,0) [yslant=0.5,xscale=0.4]     
        {\includegraphics[width=40mm]{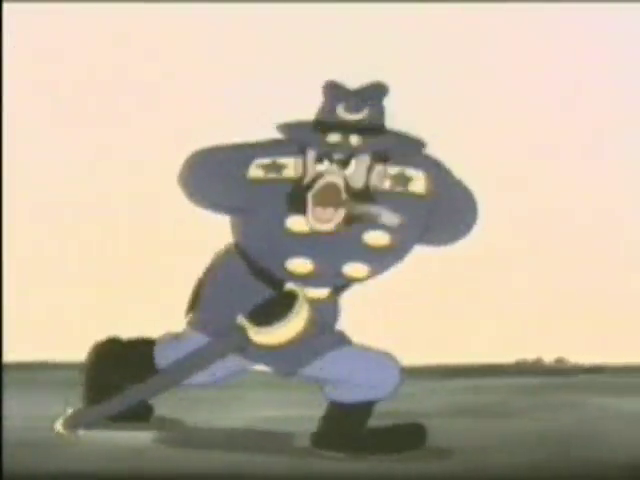}};
    \node (term) at (2,0) [yslant=0.5,xscale=0.4]      
        {\includegraphics[width=40mm]{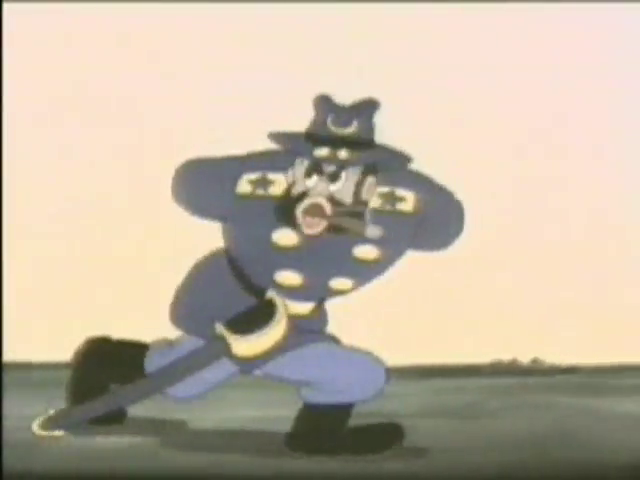}};
    \node (term) at (3,0) [yslant=0.5,xscale=0.4]      
        {\includegraphics[width=40mm]{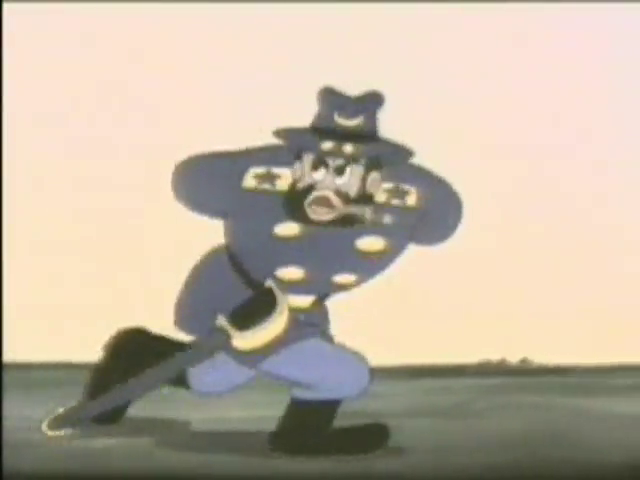}};
    \node (term) at (4,0) [yslant=0.5,xscale=0.4]      
        {\includegraphics[width=40mm]{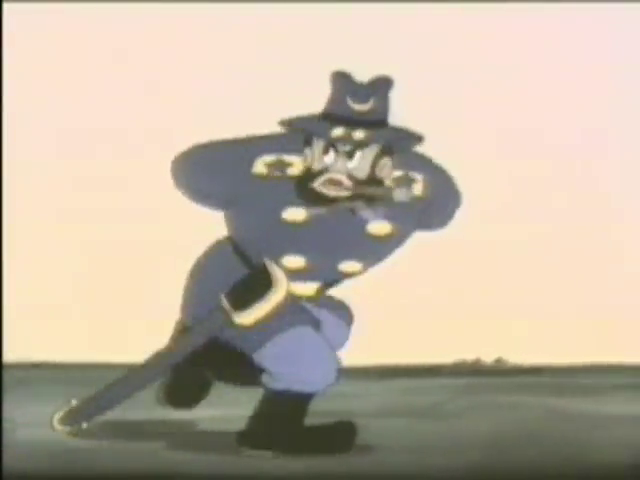}};
    \node (term) at (5,0) [yslant=0.5,xscale=0.4]      
        {\includegraphics[width=40mm]{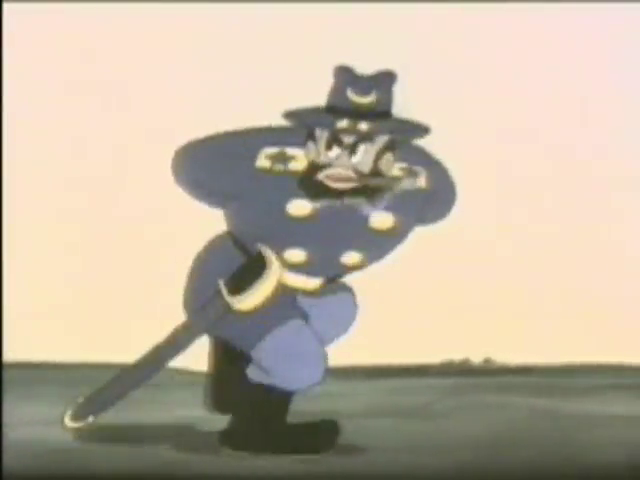}};
    \node (term) at (6,0) [yslant=0.5,xscale=0.4] 
        {\includegraphics[width=40mm]{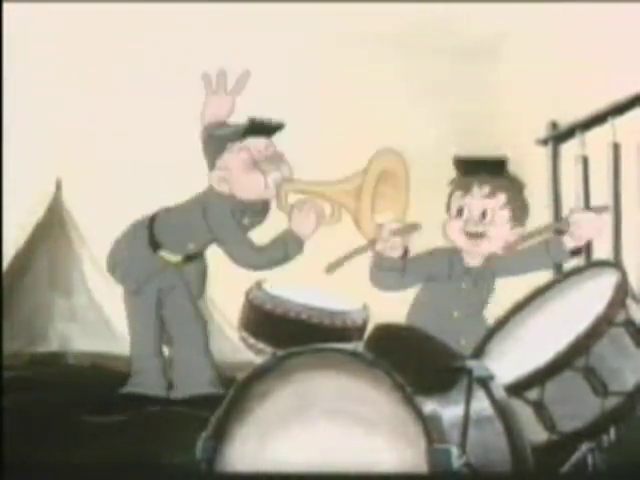}};
    \node (term) at (7,0) [yslant=0.5,xscale=0.4]     
        {\includegraphics[width=40mm]{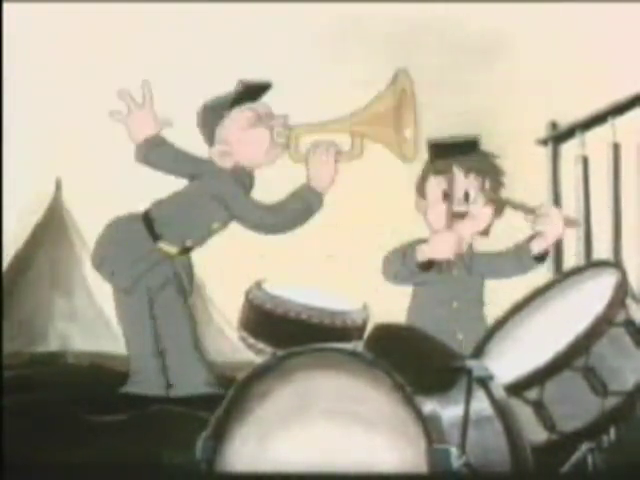}};
    \node (term) at (8,0) [yslant=0.5,xscale=0.4]      
        {\includegraphics[width=40mm]{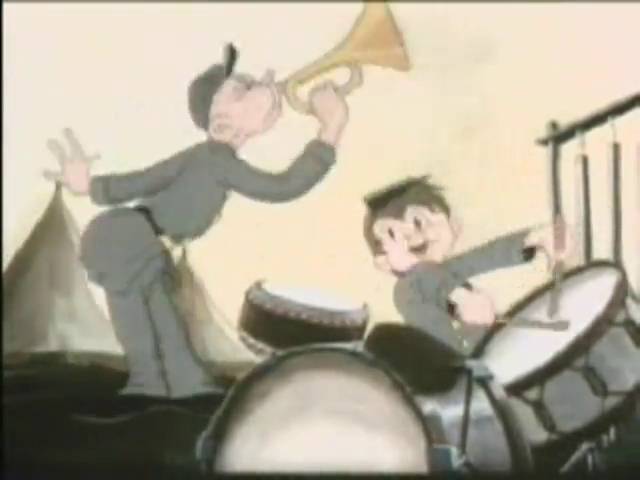}};
    \node (term) at (9,0) [yslant=0.5,xscale=0.4]      
        {\includegraphics[width=40mm]{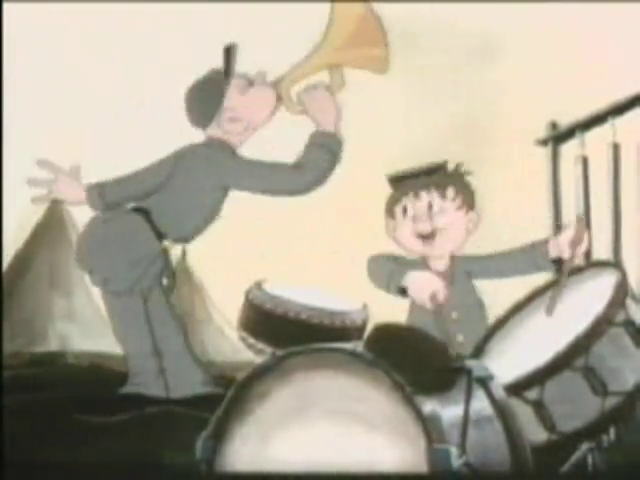}};
    \node (term) at (10,0) [yslant=0.5,xscale=0.4]      
        {\includegraphics[width=40mm]{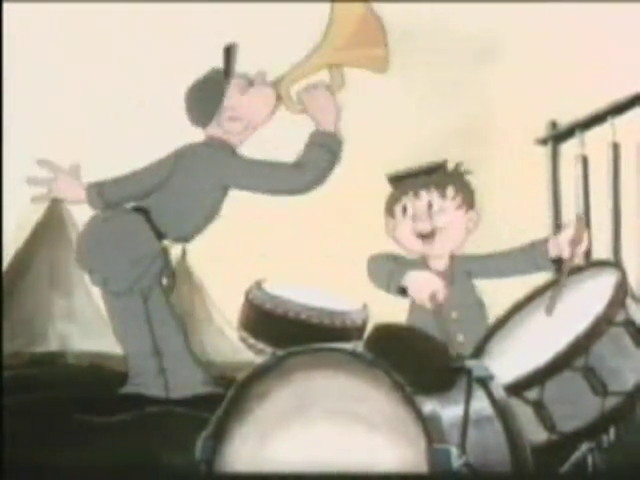}};
    \node (term) at (11,0) [yslant=0.5,xscale=0.4] 
        {\includegraphics[width=40mm]{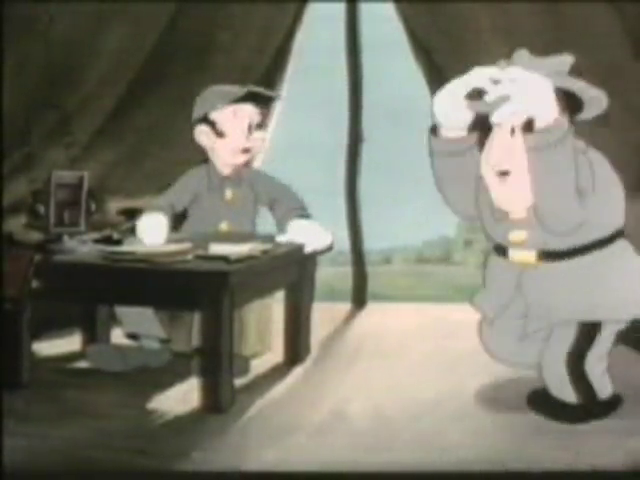}};
    \node (term) at (12,0) [yslant=0.5,xscale=0.4]     
        {\includegraphics[width=40mm]{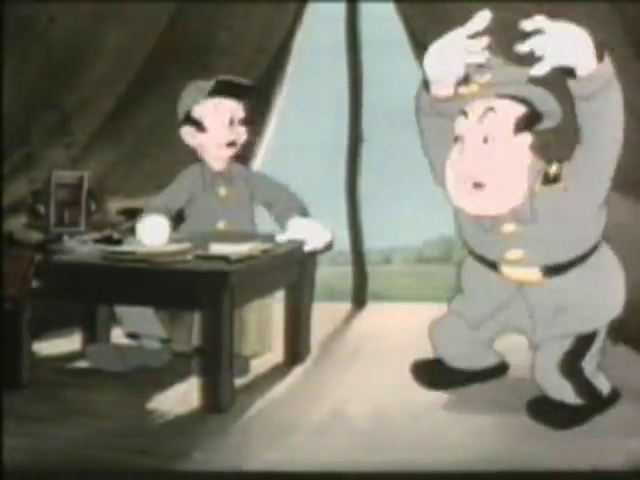}};
    \node (term) at (13,0) [yslant=0.5,xscale=0.4]      
        {\includegraphics[width=40mm]{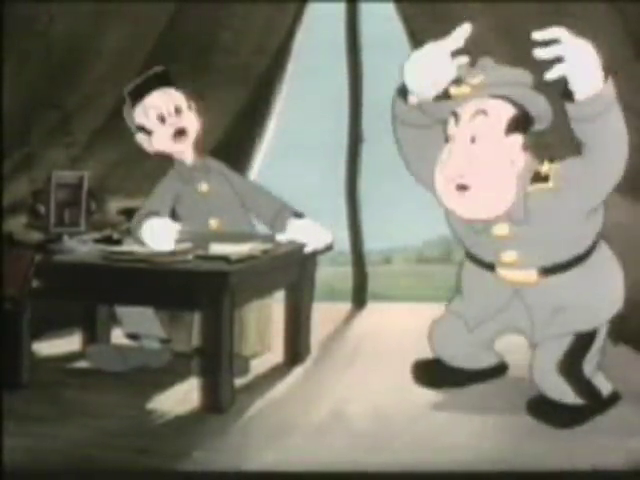}};
    \node (term) at (14,0) [yslant=0.5,xscale=0.4]      
        {\includegraphics[width=40mm]{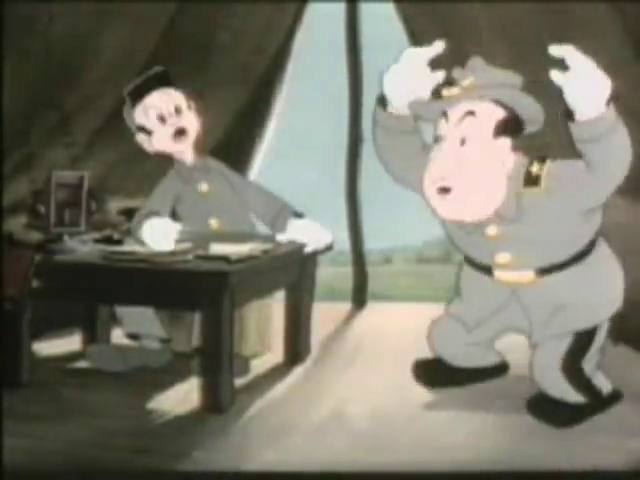}};
    \node (term) at (15,0) [yslant=0.5,xscale=0.4]      
        {\includegraphics[width=40mm]{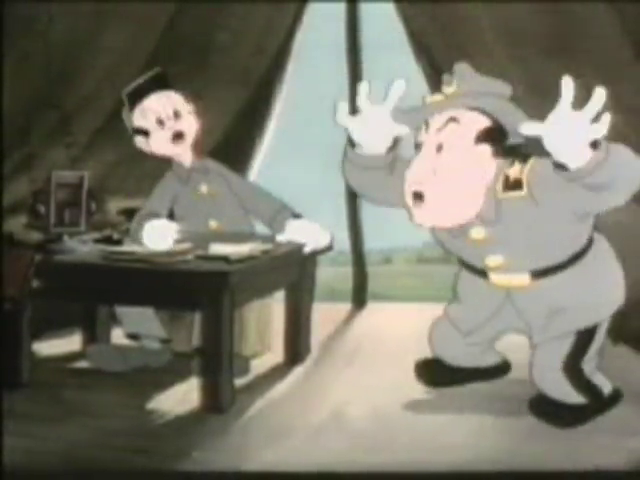}};
\end{tikzpicture}
\caption{Example frames from Video 1 (see Section X for more details). Each frame is a data sample and each scene in the video corresponds to a subspace.}
\label{Fig:vid_example}
\end{figure*}

\begin{figure*}[]
\centering
\fcolorbox{purple}{white}{
\begin{minipage}{.27\textwidth}
\includegraphics[width=0.3\textwidth]{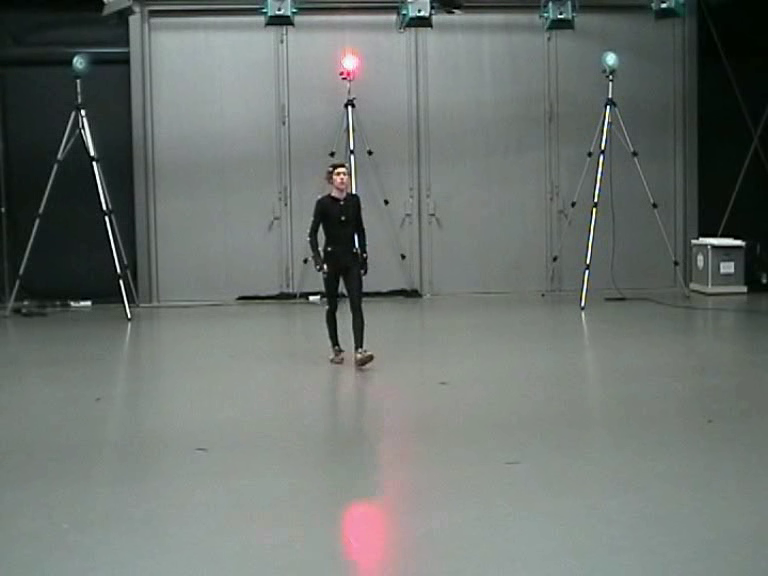}
\includegraphics[width=0.3\textwidth]{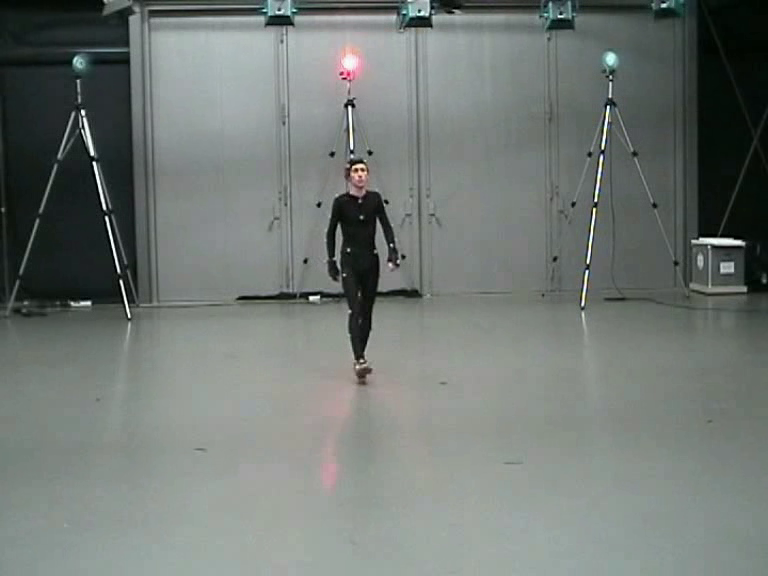}
\includegraphics[width=0.3\textwidth]{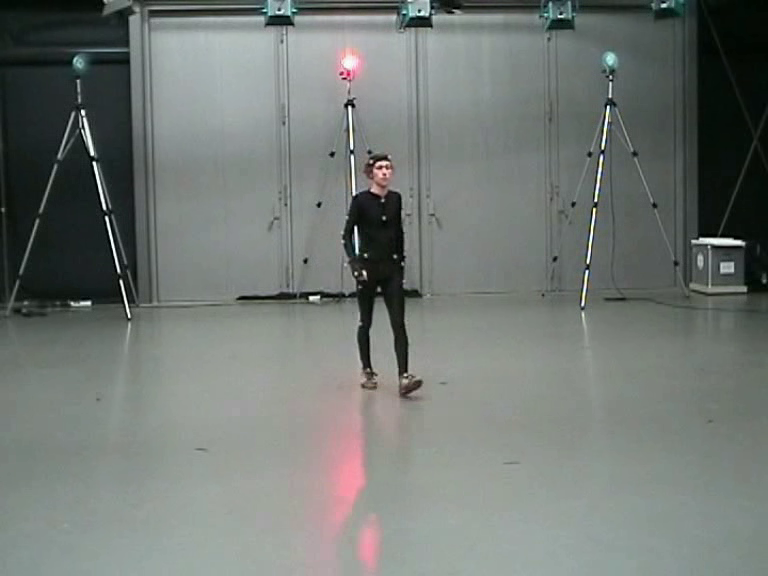}\\
\linebreak
\includegraphics[width=0.3\textwidth]{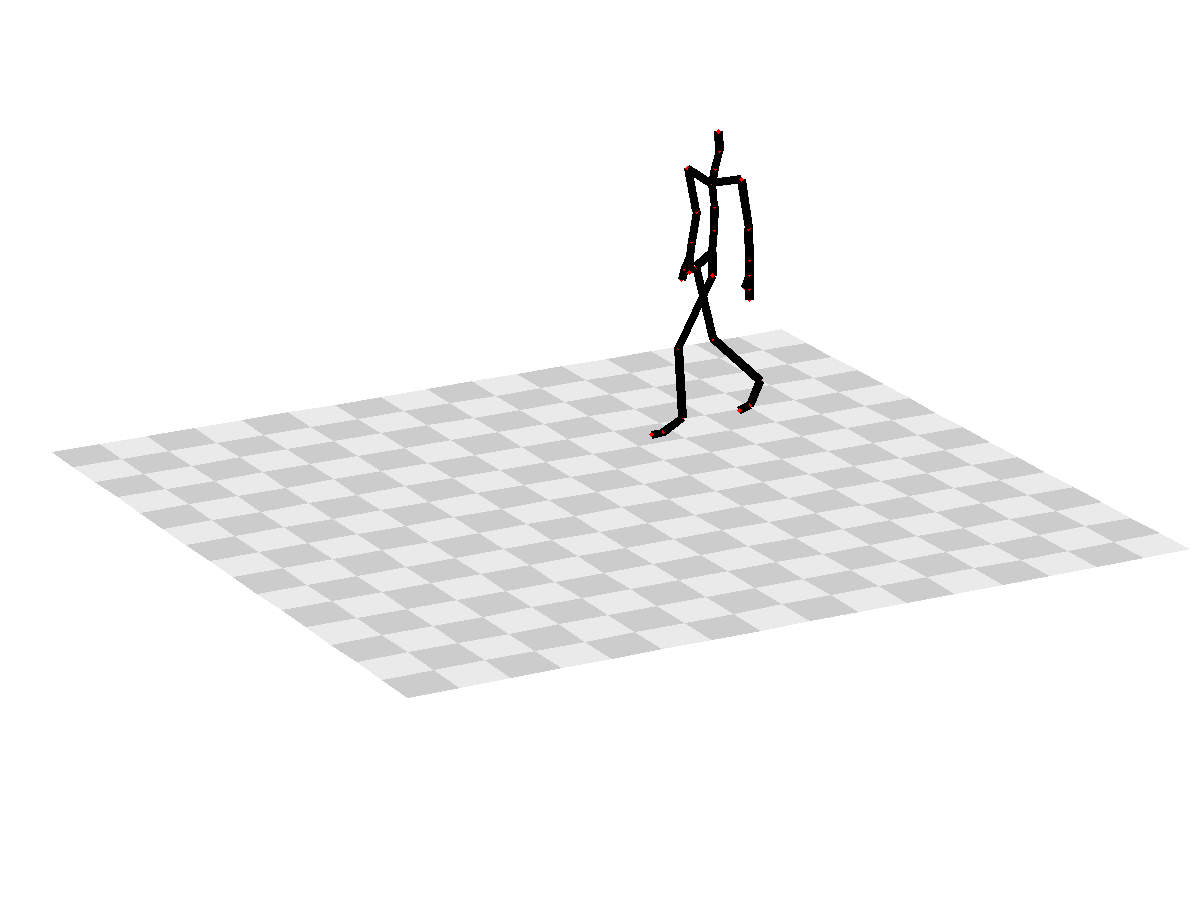}
\includegraphics[width=0.3\textwidth]{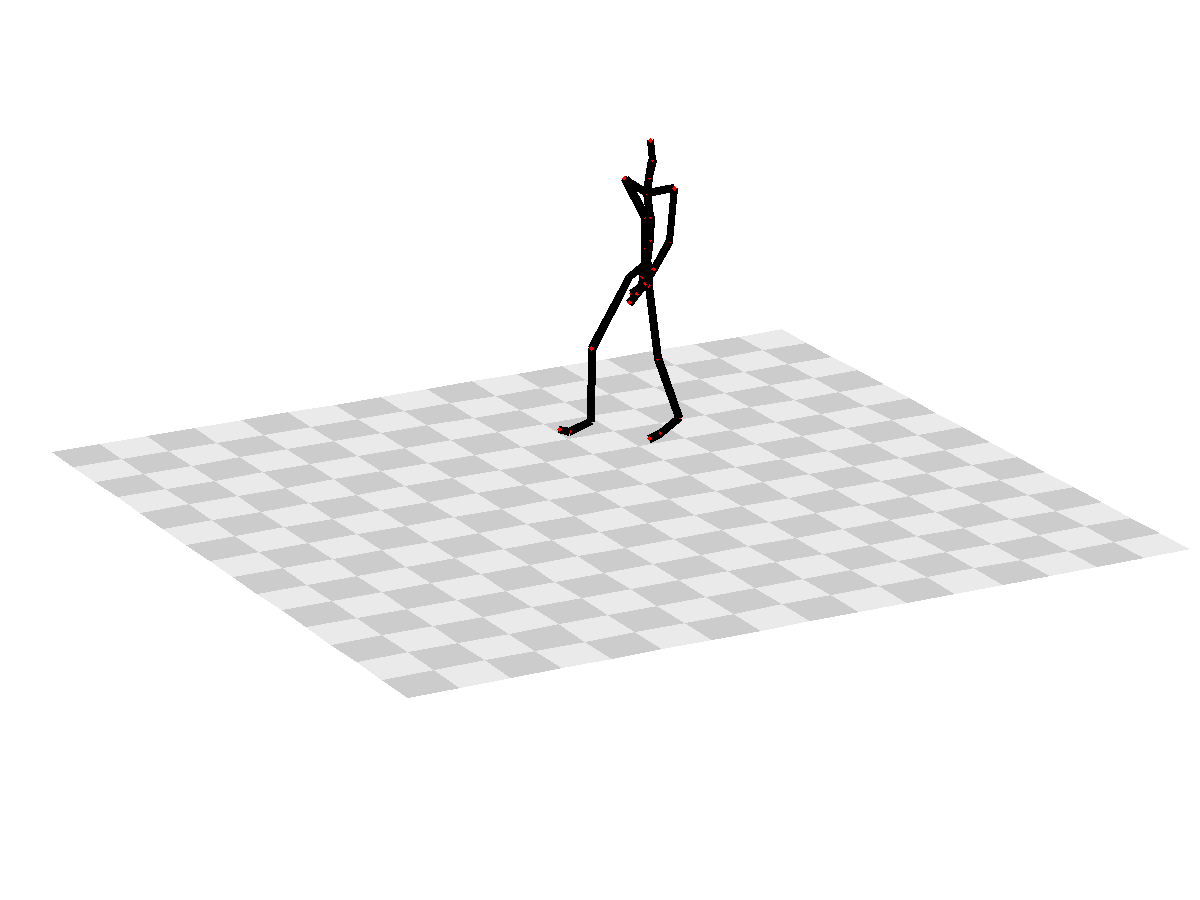}
\includegraphics[width=0.3\textwidth]{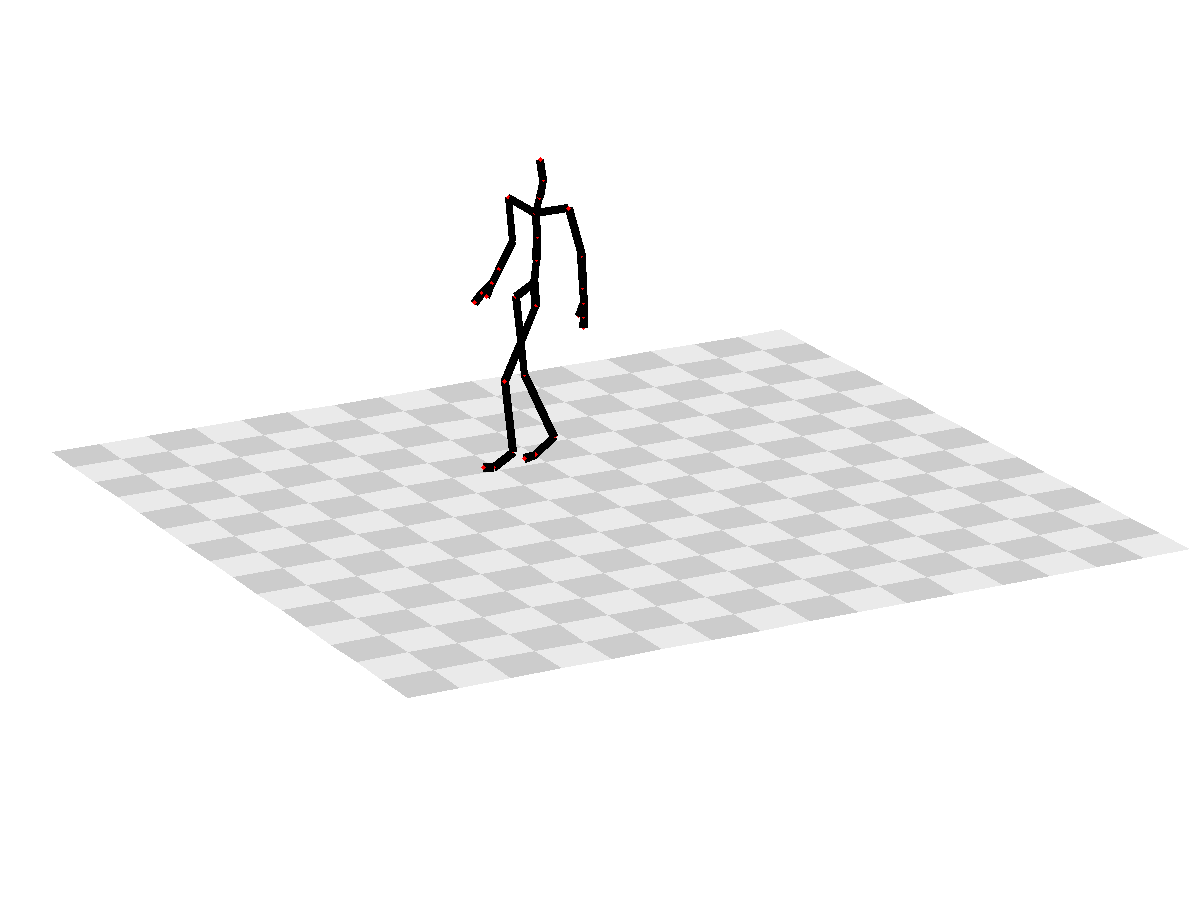}\end{minipage}}
\fcolorbox{blush}{white}{
\begin{minipage}{.27\textwidth}
\includegraphics[width=0.3\textwidth]{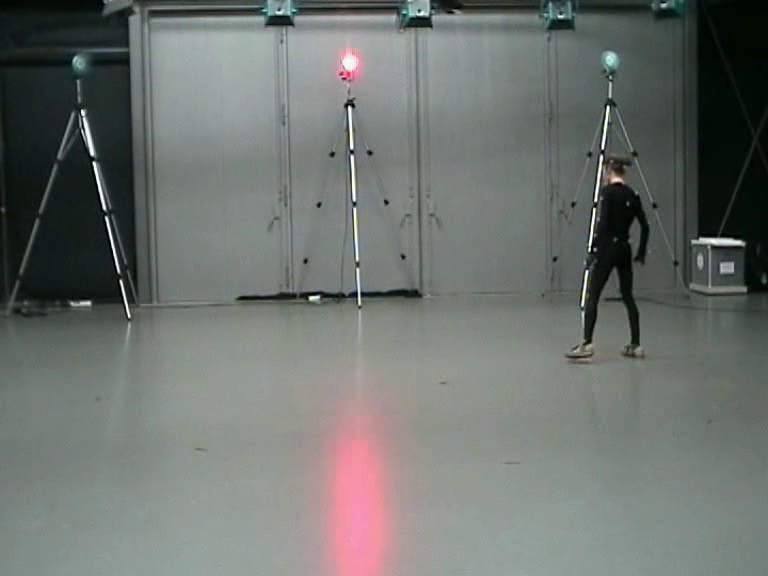}
\includegraphics[width=0.3\textwidth]{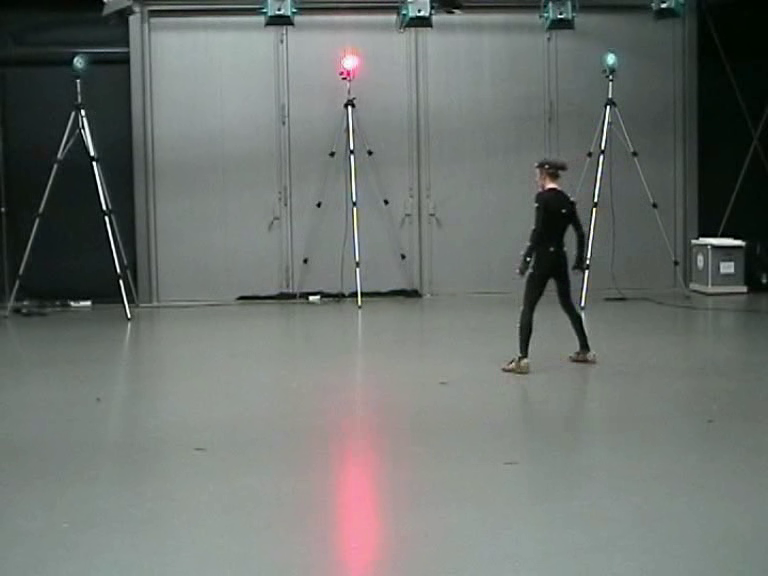}
\includegraphics[width=0.3\textwidth]{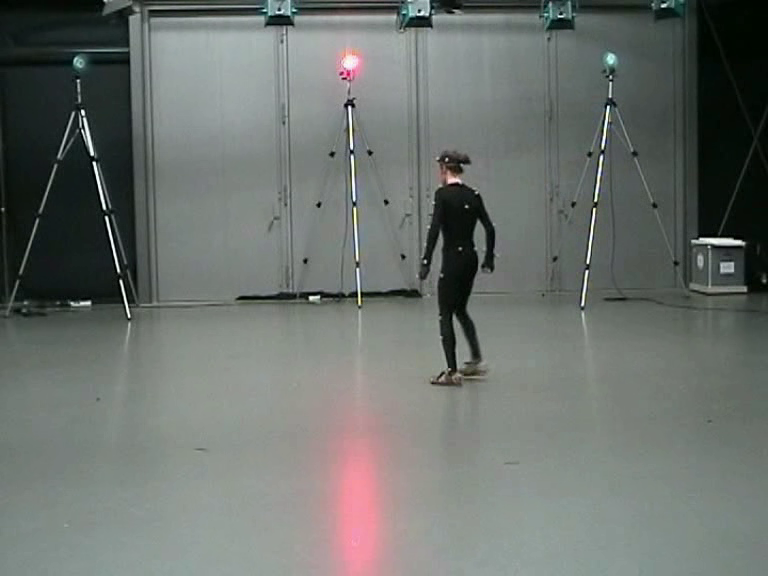}\\
\linebreak
\includegraphics[width=0.3\textwidth]{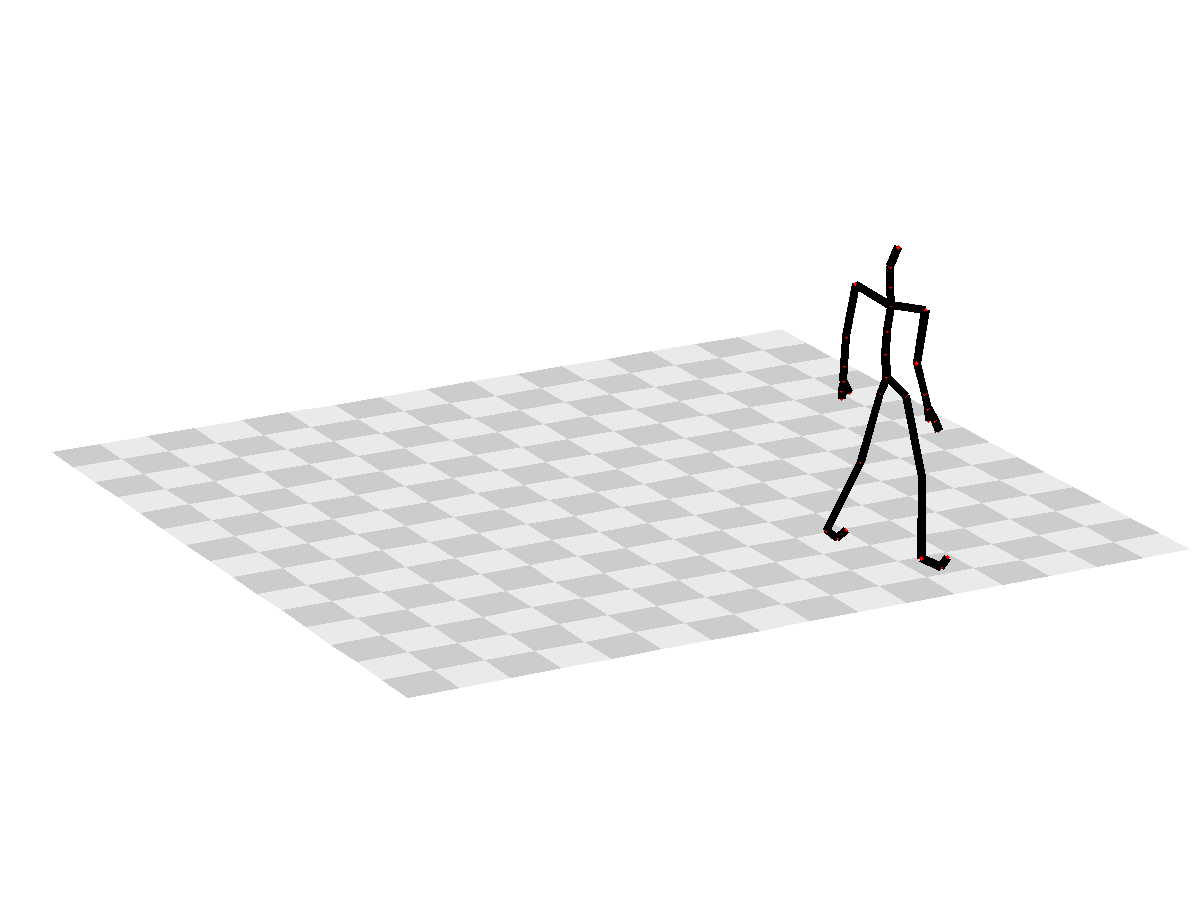}
\includegraphics[width=0.3\textwidth]{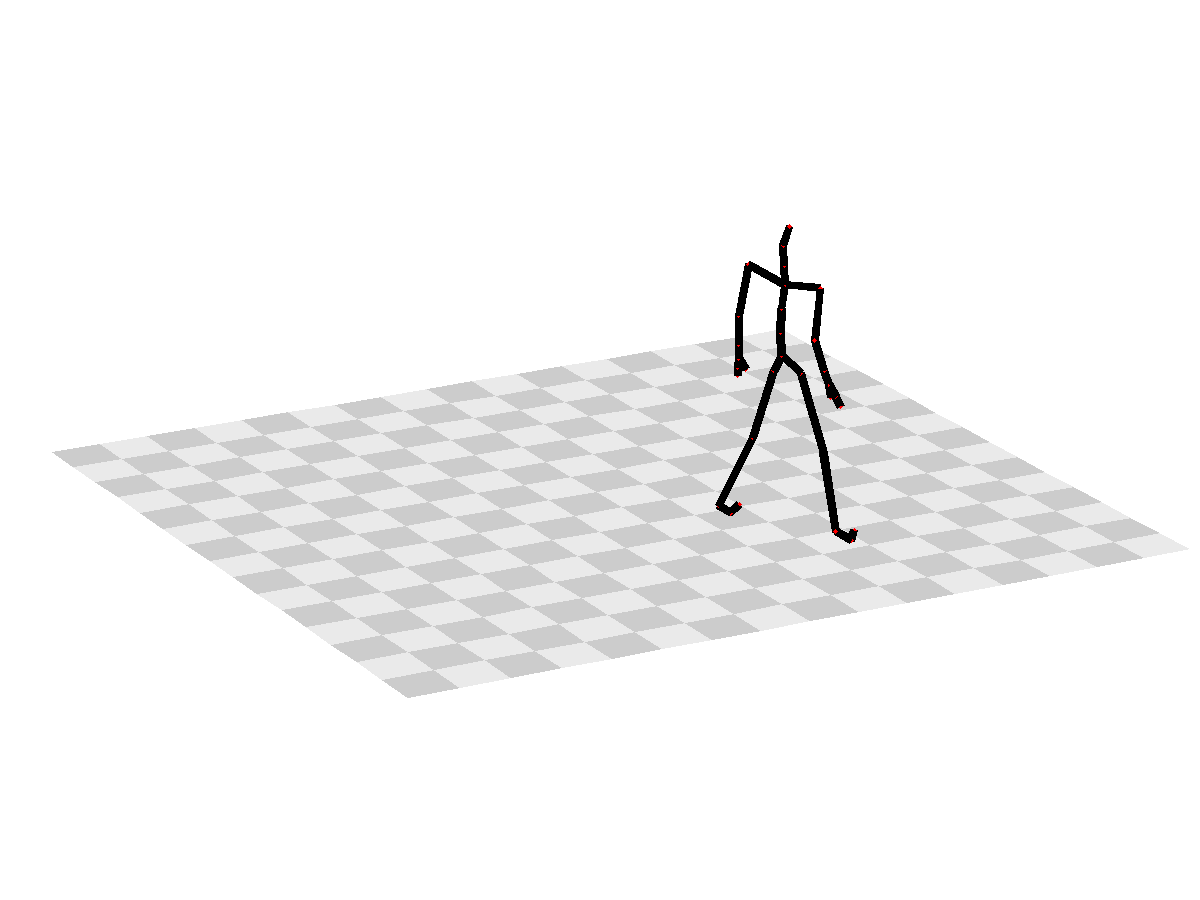}
\includegraphics[width=0.3\textwidth]{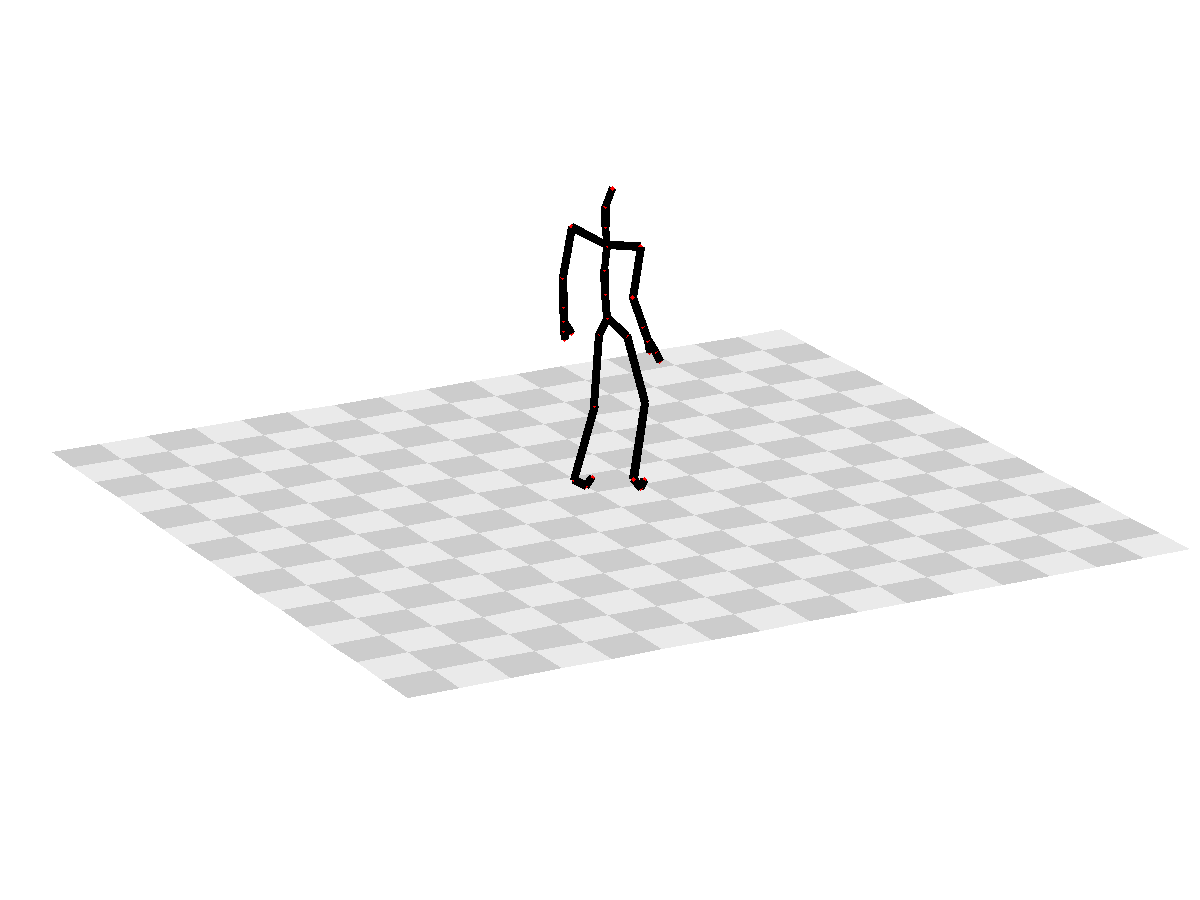}\end{minipage}}
\fcolorbox{celadon}{white}{
\begin{minipage}{.27\textwidth}
\includegraphics[width=0.3\textwidth]{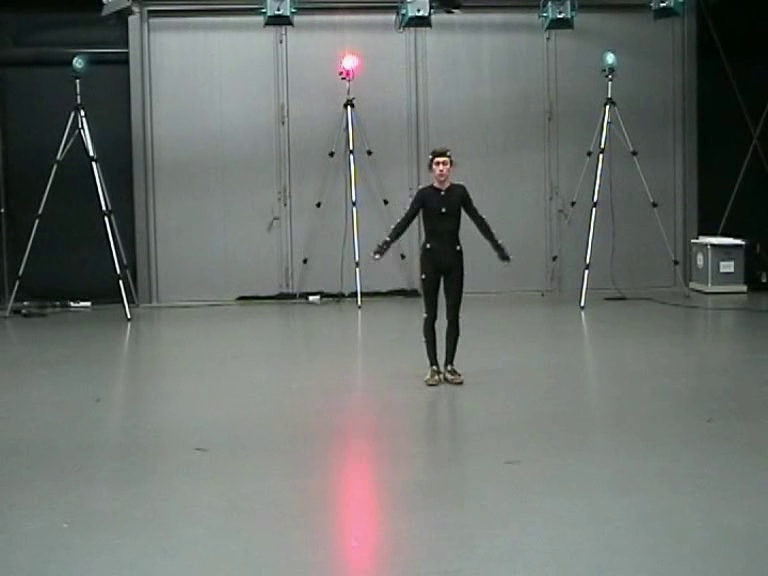}
\includegraphics[width=0.3\textwidth]{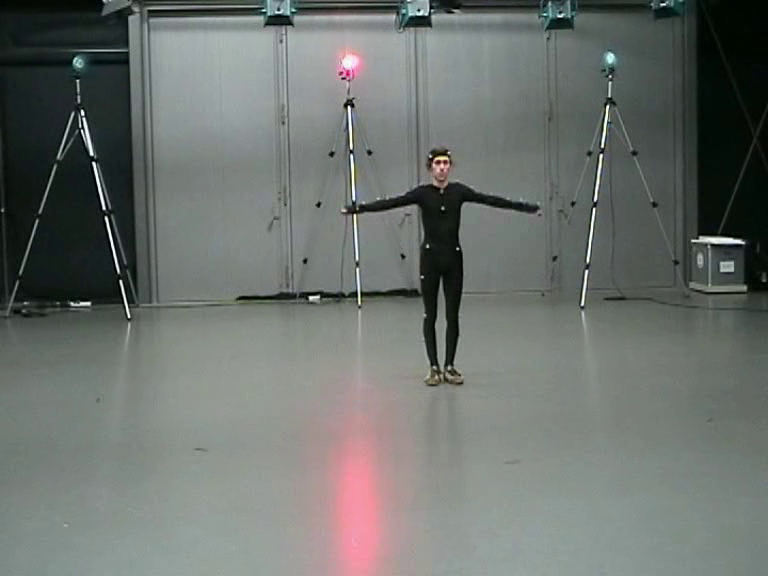}
\includegraphics[width=0.3\textwidth]{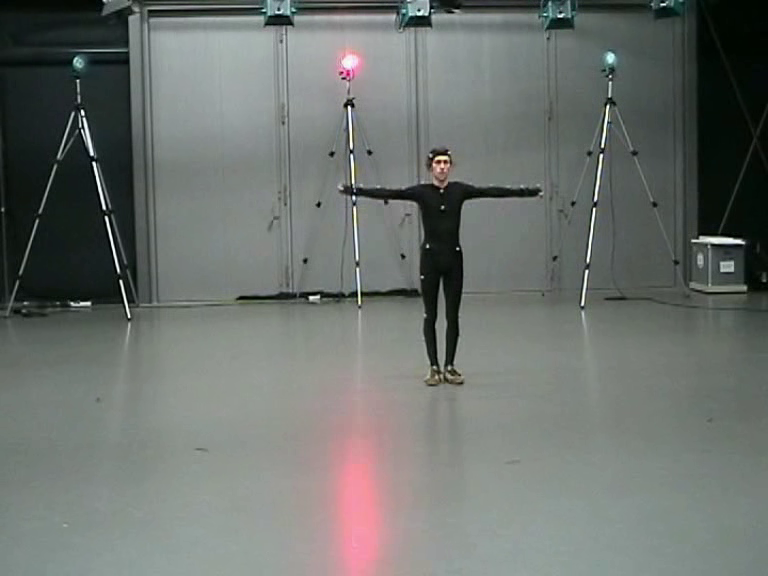}\\
\linebreak
\includegraphics[width=0.3\textwidth]{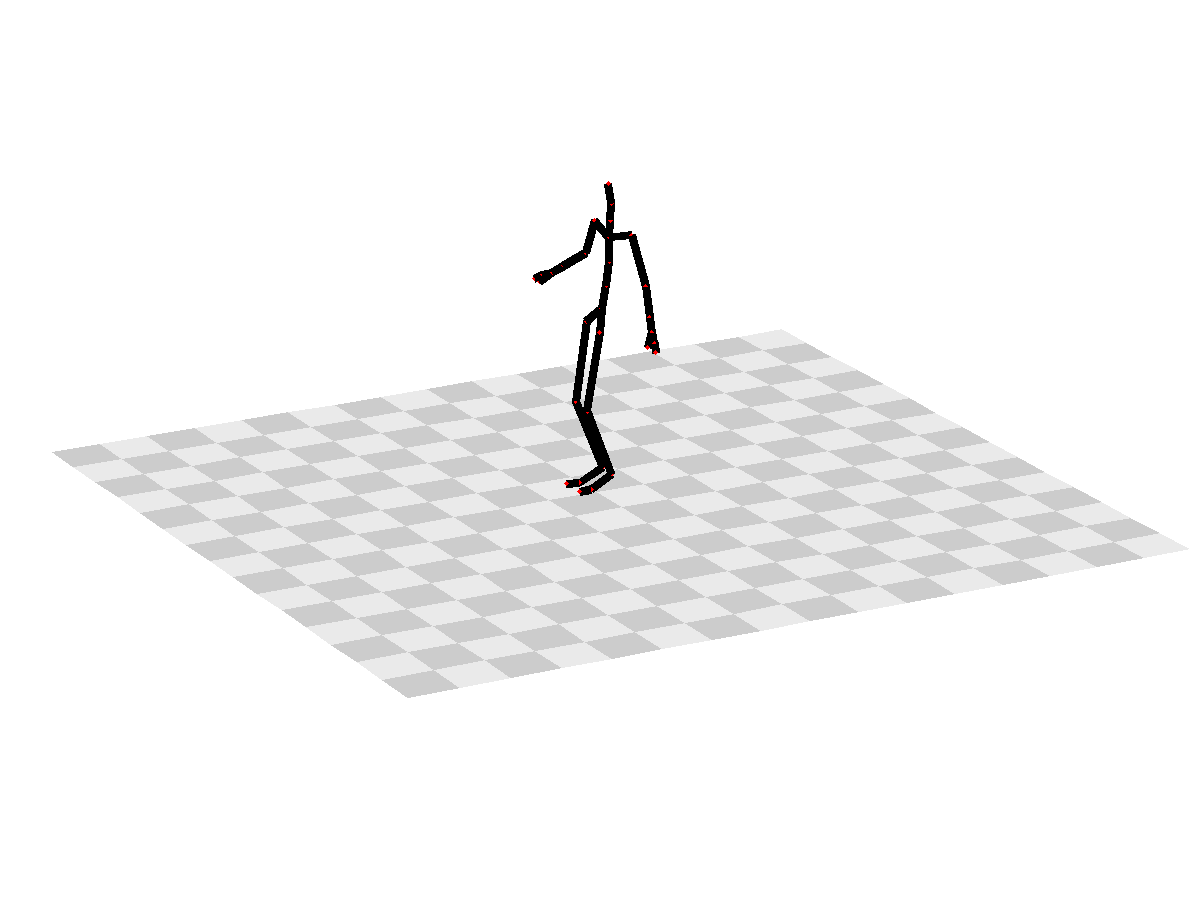}
\includegraphics[width=0.3\textwidth]{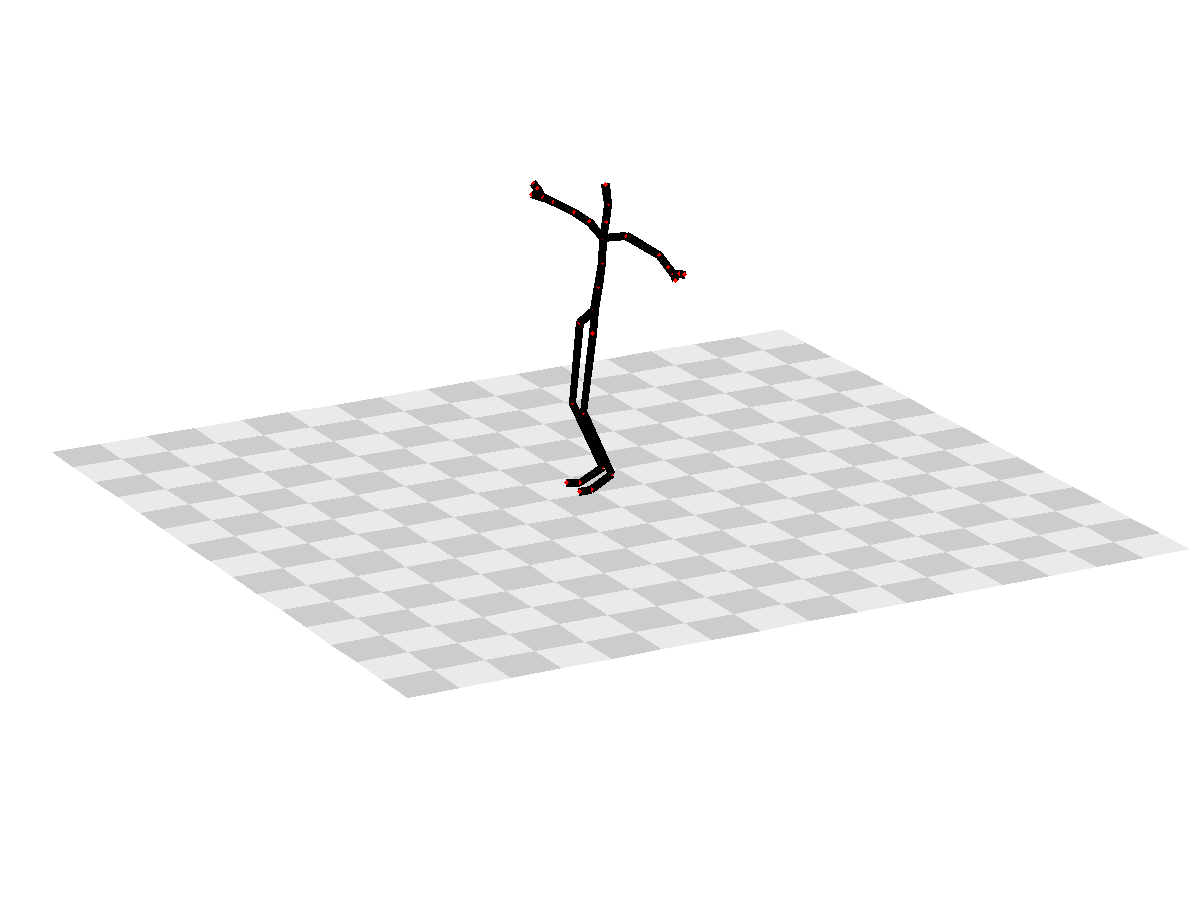}
\includegraphics[width=0.3\textwidth]{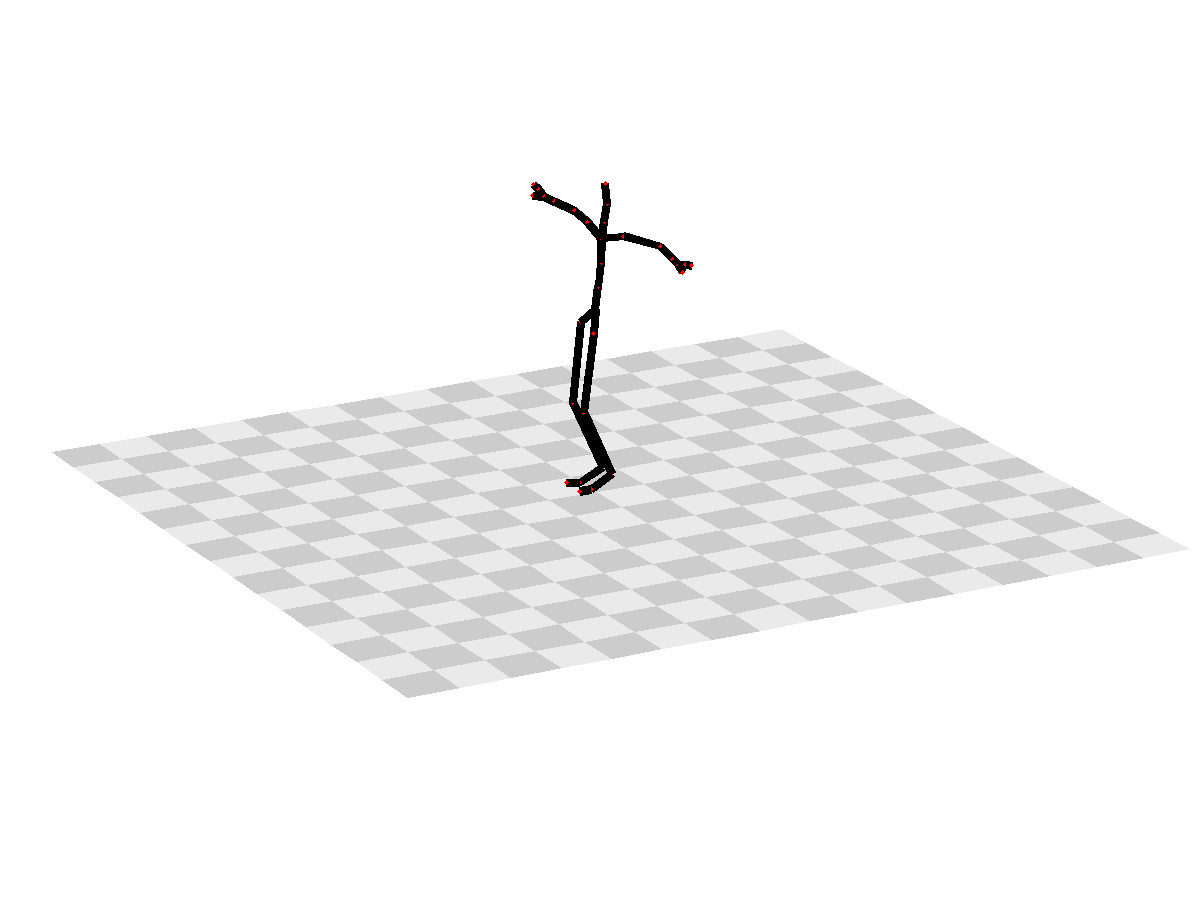}\end{minipage}}
\caption{Three examples of human activities (walking, side-stepping and balancing) from the HMD database. Each activity lies in it's own subspace. The top row demonstrates the actor wearing the reflective marker suit and the bottom row shows the captured skeletal structure.}
\label{Fig:mocap_example}
\end{figure*}

This papers main contribution is the proposal and discussion of the Ordered Subspace Clustering (OSC) method, which exploits the sequential structure of the data. Experimental evaluation demonstrates that OSC outperforms state-of-the-art subspace segmentation methods on both synthetic and real world datasets. A preliminary version of this paper was published in CVPR14 \cite{tierney2014subspace}. The optimisation scheme that was suggested in the preliminary version lacked a guarantee of convergence and suffered from huge computational cost. In this paper we provide two new optimisation schemes to solve the OSC objective, which have guaranteed convergence, much lower computational requirements and can be computed in parallel. Furthermore we perform experiments on new synthetic and real datasets.


\section{Prior and Related Work}

The state-of-the-art methods in subspace segmentation are the spectral subspace segmentation methods such as Sparse Subspace Clustering (SSC) and Low-Rank Representation (LRR). Spectral subspace segmentation methods consist of two steps:
\begin{enumerate}
\item Learn the subspace structure from the data
\item Interpret the structure as an affinity matrix and segment via spectral clustering
\end{enumerate}
The main difference between spectral methods is in their approaches to learning the subspace structure.

To learn the subspace structure of the data, spectral subspace segmentation methods exploit the the self expressive property \cite{elhamifar2012sparse}:
\begin{quote}
{\it{each data point in a union of subspaces can be efficiently reconstructed by a combination of other points in the data}}.
\end{quote}

In other words a point in a subspace can only be represented by a linear combination of points from within the same subspace. Unless the subspaces intersect or overlapping, which is assumed to be extremely unlikely in practice. This leads to the following model
\begin{align}
\mathbf{a_i = \mathbf A z_i}
\label{subspace}
\end{align}
where $\mathbf z_i \in \mathbb{R}^{N}$ is a vector of coefficients, which encode the subspace structure. Due to the self-expressive property the non-zero elements of $\mathbf z_i$ will correspond to samples in $\mathbf A$ that are in the same subspace as sample $i$. Therefore learning the coefficient vectors for each data sample can reveal some of the underlying subspace structure. The model can be expressed for all data points as
\[
\mathbf{A = A Z}
\]
where columns of $\mathbf Z = [\mathbf{ z_1,z_2,\dots,z_N}] \in \mathbb{R}^{N \times N}$.

After learning $\mathbf Z$ the next step is to assign each data point a subspace label. The first step in this process is to build a symmetric affinity matrix. The affinity matrix is usually defined as
\begin{align}
\mathbf W = |\mathbf Z|^T + |\mathbf Z|
\label{w_def}
\end{align}
where element $W_{ij}$ of $\mathbf W$ is interpreted as the affinity or similarity between data points $i$ and $j$. Next this affinity matrix is used by a spectral clustering method for final segmentation. Normalised Cuts (NCut) \cite{shi2000normalized} is the de facto spectral clustering method for this task \cite{elhamifar2012sparse, liu2010robust}.

So far it has been assumed that the original and clean data $\mathbf A$ is observed. Unfortunately this ideal situation is rare with real world data. Instead the data is usually corrupted by noise in the data capture process or during data transmission of the data. Therefore most subspace clustering methods assume the following data generation model
\[
\mathbf{X = A + N}
\]
where $\mathbf A$ is the original data where each point (column) lies on a subspace and $\mathbf N$ is noise. $\mathbf N$ follows some probability distribution. Two common assumptions for $\mathbf N$ are Gaussian distribution and Laplacian distribution.

Since it may be difficult to isolate the original data $\mathbf A$ from the noise $\mathbf N$, most subspace clustering methods actually address the issue of noise by allowing greater flexibility in the self-expressive model. The self-expressive model usually becomes
\[
\mathbf{X = X Z + E}
\]
where $\mathbf E$ is a fitting error and is {\em different from} $\mathbf N$.

\subsection{Sparse Subspace Clustering}

Sparse Subspace Clustering (SSC) was originally introduced by Elhamifar \& Vidal \cite{elhamifar2012sparse,elhamifar2009sparse}. SSC adopts concepts from the domain of sparse models, namely that
\begin{quote}
{\it{there exists a sparse solution, $\mathbf z_i$, whose nonzero entries correspond to data points from the same subspace as $\mathbf a_i$}}.
\end{quote}
In the case where the observed data is noiseless, i.e.\ we have $\mathbf A$, each data point lying in the $d_i$-dimensional subspace $S_i$ can be represented by $d_i$ points. This corresponds to the sparse representation of points, ideally a sparse solution should only select coefficients belonging to the same subspace as each point. Furthermore the number of non-zero coefficients should correspond to the dimension of the underlying subspace. The sparsity goals of SSC could be achieved through a solution to the following
\begin{align}
\min_{\mathbf{Z}} \; &  \|\mathbf Z\|_{0}, \quad\quad\text{s.t.} \quad   \mathbf{A = AZ}, \; \text{diag}(\mathbf Z) = \mathbf{0}, 
\end{align}
where $\| \cdot \|_0$ is called the $\ell_0$ norm and is defined the number of non-zero entries. The diagonal constraint is used to avoid the degenerate solution of expressing the point as a linear combination of itself.  However this problem is intractable, instead the convex relaxation $\ell_1$ norm is used
\begin{align}
\min_{\mathbf{Z}} \;  \|\mathbf Z\|_{1}, \quad\quad\text{s.t.} \quad \mathbf{A = AZ}, \; \text{diag}(\mathbf Z) = \mathbf{0}.  
\end{align}
The $\| \cdot \|_1$ is the $\ell_1$ norm and is defined as $\sum_{i = 1}^N \sum_{j = 1}^N | Z_{ij} |$ i.e.\ the sum of absolute values of the entries. We call this heuristic SSC.

To overcome the simultaneous presence of noise and outliers, Elhamifar \& Vidal \cite{elhamifar2009sparse} devised the following alternative
\begin{align}
\min_{\mathbf{E, S, Z}} \; \frac{\lambda_1}{2} \|\mathbf{E}\|^2_F + \lambda_2 \|\mathbf{S}\|_1 +  \|\mathbf Z\|_{1} \\
\text{s.t.} \quad \mathbf{X = XZ + E + S}, \text{diag}(\mathbf Z) = \mathbf{0} \nonumber
\end{align}
where $ \| \cdot \|_F$ is the Frobenius norm and $\mathbf S$ is high magnitude sparse fitting error. This model allows for flexibility in the fitting error since setting either $\lambda_1$ or $\lambda_2$ to $0$ eliminates $\mathbf E$ or $\mathbf S$ from the model. This only compensates for fitting errors however shows surprising robustness in practice.

Recent work by Soltanolkotabi, Elhamifar and Candes \cite{soltanolkotabi2014robust} showed that under rather broad  conditions using noisy data $\mathbf X$ the $\ell_1$ approach should produce accurate clustering results. These conditions include maximum signal-to-noise ratio, number of samples in each cluster and distance between subspaces and appropriate selection of parameters. They use the following relaxed objective
\begin{align}
\min_{\mathbf{z_i}} \; \frac{1}{2}\| \mathbf x_i - \mathbf X \mathbf z_i \|_F^2 + \lambda_i \|\mathbf z_i\|_{1},   
\quad \text{s.t.} \quad \text{diag}(\mathbf Z) = \mathbf{0}.  
\end{align}
with regularisation parameter $\lambda_i$ tuned for each data sample.

In practice SSC allows for efficient computation. Each column of $\mathbf Z$ can be computed independently and in parallel with the other columns. In contrast to LRR (discussed next) the computational requirements are lightweight. Furthermore $\mathbf Z$ is sparse, which can reduce memory requirements and decrease time computational requirements and time spent during the final spectral clustering step.

\subsection{Low-Rank Subspace Clustering}

Rather than compute the sparsest representation of each data point individually, Low-Rank Representation (LRR) by Liu, Lin and Yu \cite{liu2010robust} attempts to incorporate global structure of the data by computing the lowest-rank representation of the set of data points. Therefore the objective becomes
\begin{align}
\min_{\mathbf{Z}} \;  \textrm{rank}(\mathbf{Z}), \quad
\text{s.t.} \quad \mathbf{A = AZ}.  
\end{align}
This means that not only can the data points be decomposed as a linear combination of other points but the entire coefficient matrix should be low-rank. The aim of the rank penalty is to create a global grouping effect that reflects the underlying subspace structure of the data. In other words, data points belonging to the same subspace should have similar coefficient patterns.

Similar to SSC, the original objective for LRR is intractable. Instead the authors of LRR suggest a heuristic version which uses the closest convex envelope of the rank operator: the nuclear or trace norm. The objective then becomes
\begin{align}
\min_{\mathbf{Z}} \;  \| \mathbf{Z} \|_*, \quad 
\text{s.t.} \quad \mathbf{A = AZ}  
\end{align}
where $\| \cdot \|_*$ is the nuclear norm and is the sum of the singular values. The singular values can be computed through the Singular Value Decomposition (SVD).

LRR has achieved a lot of attention in the subspace segmentation community, which had led to some interesting discoveries. The most surprising of which is that there is a closed form solution to the heuristic noiseless LRR objective. The closed form solution is given by the Shape Interaction Matrix (SIM) and is defined as
\[
\mathbf{Z = V_A V_A^T}
\]
where $\mathbf{V_A}$ are the right singular vectors given by the SVD of $\mathbf A$.

In the case where noise is present, the authors of LRR suggested a similar model to that used in SSC. However they assume that their fitting error will be only be present in a small number of columns. This results in the following objective
\begin{align}
\min_{\mathbf {E, Z}} \; \lambda \|\mathbf{E}\|_{1,2} + \|\mathbf Z\|_{*}, \quad
\text{s.t.} \quad \mathbf{X = XZ + E},
\end{align}
where $\|\mathbf{E}\|_{1,2} = \sum^n_{i=1} \| \mathbf{e_i} \|_2$ is the $\ell_{1,2}$ norm.

Even though LRR has shown impressive accuracy performance in many subspace segmentation tasks it has two drawbacks:
\begin{itemize}
\item high computational cost,
\item large memory requirements.
\end{itemize}
LRR's high computational cost comes from the required computation of the SVD of $\mathbf Z$ at every iteration. Depending on the convergence tolerance LRR may iterate hundreds or even thousands of times. However some improvements have been made by computing partial or skinny SVD approximations. Similarly the large memory requirements of LRR stem from the computation of the SVD of $\mathbf Z$. Since the number of elements in $\mathbf Z$  scales quadratically with the number of number of data samples it may not be possible to apply LRR even for modest datasets. Work has been done in fast approximations of SVD \cite{liberty2007randomized, woolfe2008fast} but it has not yet been applied to LRR at the time of writing.

\subsection{Regularised Variants}

Laplacian Regularised LRR \cite{Yin:rz} and LRR with Local Constraint \cite{Zheng2013398} incorporate Laplacian regularisation to ensure that data points close in the ambient space share similar coefficient structure. The objectives for both approaches can be generalised to
\begin{align}
\min_{\mathbf{Z}} \; &  f(\mathbf Z) + \lambda \sum_i^N W_{ij} \| \mathbf z_i - \mathbf z_j \|_2, \ \ \ \text{s.t.} \quad \mathbf{X = XZ + E},  
\end{align}
where $f(\mathbf Z)$ is a placeholder for a fitting term and other regularisation such as nuclear norm or $\ell_1$ on $\mathbf Z$ and $W_{ij}$ is a weight based on distance between sample $i$ and $j$.

Spatial Subspace Clustering (SpatSC) \cite{Guo.Y;Gao.J;Li.F-2013} extended SSC by incorporating a sequential $\ell_1$ neighbour penalty
\begin{align}
\label{YiGuo1}
\min_{\mathbf Z, \mathbf E} \frac12\|\mathbf E|^2_F +\lambda_1\|\mathbf Z\|_{1}+\lambda_2\|\mathbf Z\mathbf R\|_{1}, \\
\text{s.t.} \quad \mathbf{X = XZ + E}\text{, diag}(\mathbf Z) = \mathbf 0, \nonumber
\end{align}
where $\mathbf R$ is a lower triangular matrix with $-1$ on the diagonal and $1$ on the second lower diagonal:
\begin{align}
\label{Rdef}
\mathbf R \in \mathbb{Z}^{N \times N-1} = \left[ \begin{matrix}
-1\\
1 & -1\\
& 1 & -1\\
& & \ddots & \ddots\\
& & & 1 & -1
\end{matrix} \right].
\end{align}
Therefore $\mathbf Z\mathbf R = [\mathbf z_2 - \mathbf z_1, \mathbf z_3-\mathbf z_2, ..., \mathbf z_N - \mathbf z_{N-1}]$. The aim of this formulation is to force consecutive columns of $\mathbf Z$ to be similar.

\section{Ordered Subspace Clustering}

The assumption for Ordered Subspace Clustering (OSC) \cite{tierney2014subspace} is that the data is sequentially structured. Since physically neighbouring data samples are extremely likely to lie in the same subspace they should have similar coefficient patterns. Consider a a video sequence from a television show or movie. The frames are sequentially ordered and each scene lies on a subspace. Since the scene changes are relatively rare compared to the high frame rate it is extremely likely that consecutive frames are from the same subspace. In other words the columns of $\mathbf Z$ should follow the rule $\mathbf z_i \approx \mathbf z_{i+1}$.

Similar to SpatSC, OSC extends SSC with an additional regularisation penalty. The objective is as follows:
\begin{align}
\label{objective}
\min_{\mathbf Z, \mathbf E} & \frac12\|\mathbf E \|^2_F +\lambda_1\|\mathbf Z\|_{1}+\lambda_2\|\mathbf Z\mathbf R\|_{1,2}, \text{s.t.} \quad \mathbf{X = XZ + E}.  
\end{align}
where $\mathbf R$ is defined as in \eqref{Rdef}.

Instead of the $\ell_1$ norm over $\mathbf{ZR}$ as used in SpatSC, OSC uses the $\ell_{1,2}$ norm to enforce column similarity of $\mathbf Z$. In contrast to SpatSC this objective much more strictly enforces column similarity in $\mathbf Z$. $\|\mathbf Z\mathbf R\|_{1}$ only imposes sparsity at the element level in the column differences $\mathbf z_i - \mathbf z_{i-1}$ and does not directly penalise whole column similarity. Therefore it allows the support (non-zero entries) of each consecutive column to vary. In effect this allows some values in consecutive columns to be vastly different. This does not meet the stated objective of $\mathbf z_i \approx \mathbf z_{i+1}$.

Thus in \eqref{objective}, the weak penalty $\|\mathbf Z \mathbf R\|_{1}$ from SpatSC has been replaced with the stronger penalty $\|\mathbf Z\mathbf R\|_{1,2}$ to strictly enforce column similarity. We remove the diagonal constraint as it is no longer required in most cases and can interfere with column similarity. However we discuss in the following section how to include the constraint if it is required.

\section{Solving the Objective Function}

In the preliminary version of this paper \cite{tierney2014subspace} a procedure to solve the relaxed version of the objective \eqref{objective} was discussed. However the former procedure lacked a guarantee of convergence. Furthermore the procedure suffered from huge computational complexity due to the expensive Sylvester equation \cite{penzl1998numerical, golub1979hessenberg} required.

In this paper two new procedures are discussed for relaxed and exact variants, both of which have guaranteed convergence and reduced computational complexity. For a demonstration of the speed improvements please see Section \ref{section:run_time} and Figure \ref{Plot:running_time}. There improvements have been achieved through adoption of the LADMAP (Linearized Alternating Direction Method with Adaptive Penalty) \cite{boyd2011distributed}, \cite{LinLiuSu2011} and LADMPSAP (Linearized Alternating Direction Method with Parallel Spliting and Adaptive Penalty) \cite{DBLP:conf/acml/LiuLS13} frameworks. Overviews of the procedures can be found in Algorithms \ref{alg_relaxed} and \ref{alg_exact}.

\subsection{Relaxed Constraints}


The relaxed variant of \eqref{objective} can be written as:
\begin{align}
\label{objective_relaxed}
\min_{\mathbf Z, \mathbf J} &\frac12\|\mathbf X - \mathbf X\mathbf Z\|^2_F +\lambda_1\|\mathbf Z\|_{1}+\lambda_2\|\mathbf J \|_{1,2}, \ \ \ \  \text{ s.t.} \quad \mathbf{J = ZR}.  
\end{align}

Then the Augmented Lagrangian for the introduced auxiliary variable constraint is,
\begin{align}
\mathcal{L}(\mathbf Z, \mathbf J, \mathbf Y|\mu) = & \frac12\|\mathbf X - \mathbf X\mathbf Z\|^2_F + \lambda_1 \|\mathbf Z\|_{1} + \lambda_2\|\mathbf J\|_{1,2} \notag\\
& + \langle \mathbf Y, \mathbf J - \mathbf Z\mathbf R\rangle +\frac{\mu}2\|\mathbf J - \mathbf Z\mathbf R\|^2_F.
\label{objectiveADMM}
\end{align}

Objective \eqref{objectiveADMM} will be solved for $\mathbf Z$ and $\mathbf J$ in a sequential and alternative manner when fixing the other, respectively.  Given the solution state $\mathbf Z^k, \mathbf J^k, \mathbf Y^k$ and adaptive constant $\mu^k$, the procedure for $k=1, 2, ...$ is as follows:

\begin{enumerate}

\item 
Update $\mathbf Z^{k+1}$ by solving the following subproblem
\begin{align}
\mathbf Z^{k+1}  = \argmin_{\mathbf Z}\mathcal{L}(\mathbf Z, \mathbf J^k, \mathbf Y^k|\mu^k). \label{Eq:24Sept2014-1}
\end{align}

which is equivalent to
\begin{align} 
\mathbf Z^{k+1}  = & \argmin_{\mathbf Z} \lambda_1 \|\mathbf Z\|_{1} + \frac12\|\mathbf X - \mathbf X\mathbf Z\|^2_F \label{Eq:24Sept2014-2} \\
& +\langle \mathbf Y^k, \mathbf J^k - \mathbf Z\mathbf R\rangle +\frac{\mu^k}2\|\mathbf J^k - \mathbf Z\mathbf R\|^2_F.  \notag
\end{align}
There is no closed form solution to the above problem because of the coefficient matrices $\mathbf X$ and $\mathbf R$ on $\mathbf Z$. Thus linearisation over the last three terms is used. Denote
$g(\mathbf Z) = \frac12\|\mathbf X - \mathbf X\mathbf Z\|^2_F$ and $h(\mathbf Z) = \langle \mathbf Y^k, \mathbf J^k - \mathbf Z\mathbf R\rangle +\frac{\mu^k}2\|\mathbf J^k - \mathbf Z\mathbf R\|^2_F$ which is from the augmented Langrangian. The linear approximation at $\mathbf Z^k$ \cite{BeckTeboulle2009} for $g(\mathbf Z)$ and $h(\mathbf Z)$ respectively, is
\begin{align}
g(\mathbf Z) \approx \langle \nabla g(\mathbf Z^k), \mathbf Z - \mathbf Z^k\rangle + \frac{L_z}2\|\mathbf Z - \mathbf Z^k\|^2_F \label{Eq:25Sept2014-3}
\end{align}
and
\begin{align}
h(\mathbf Z) \approx \langle \nabla h(\mathbf Z^k), \mathbf Z - \mathbf Z^k\rangle + \frac{\sigma^k_z}2\|\mathbf Z - \mathbf Z^k\|^2_F. \label{Eq:25Sept2014-4}
\end{align}
where $\sigma^k_z = \mu^k\eta_z$ and
\begin{align*}
&\nabla g(\mathbf Z^k) = - \mathbf X^T(\mathbf{X} - \mathbf X\mathbf Z^k),\\
&\nabla h(\mathbf Z^k) = -(\mathbf Y^k + \mu^k (\mathbf J^k - \mathbf Z^k\mathbf R))\mathbf R^T.
\end{align*}

Denote
\[
\mathbf V^k = \mathbf Z^k + \frac1{\sigma^k_z+L_z}[\mathbf X^T(\mathbf{X} - \mathbf X\mathbf Z^k) + \widetilde{Y}^k \mathbf R^T]
\]
where
\[
\widetilde{Y}^k = Y^k+ \mu^k (\mathbf J^k - \mathbf Z^k\mathbf R),
\]
then problem \eqref{Eq:24Sept2014-1} can be approximated by the following problem
\begin{align}
& \mathbf Z^{k+1} = \argmin_{\mathbf Z}\lambda_1\|\mathbf Z\|_1 + \frac{\sigma^k_z+L_z}{2}
\left\|\mathbf Z - \mathbf V^k \right\|^2_F \label{ForZ}
\end{align}

Problem \eqref{ForZ} is separable at element level and each has a closed-form solution defined by the soft thresholding operator, see \cite{bach2011convex, liu2010efficient}, as follows
\begin{align}
\label{SolutionZ_relaxed}
\mathbf Z^{k+1} = \textrm{sign}\left(\mathbf V^k \right) \max\left(\left|\mathbf V^k \right| - \frac{\lambda_1}{\sigma^k_z+L_z}\right).
\end{align}

\item Given the new value $\mathbf Z^{k+1}$ from last step, $\mathbf J^{k+1}$ is updated by solving
\begin{align*}
\mathbf J^{k+1} = \argmin_{\mathbf J}&\mathcal{L}(\mathbf Z^{k+1}, \mathbf J, \mathbf Y^k|\mu^k)\\ 
= \argmin_{\mathbf J}& \lambda_2\|\mathbf J\|_{1,2} + \langle \mathbf Y^k, \mathbf J - \mathbf Z^{k+1}\mathbf R\rangle+\frac{\mu^k}2\|\mathbf J - \mathbf Z^{k+1}\mathbf R\|^2_F.
\end{align*}

The linear term is easily absorbed into the quadratic term such that a solvable problem can be achieved as follows,
\begin{align}
\min_{\mathbf J} \lambda_2\|\mathbf J\|_{1,2} + \frac{\sigma^k_J}2\|\mathbf J - \mathbf Z^{k+1}\mathbf R + \frac{1}{\sigma^k_J} \mathbf Y^k \|^2_F \label{Eq:25Sept2014-5}
\end{align}
where $\sigma^k_J = \mu^k \eta_J$ with a constant $\eta_J > 1$.\footnote{Ideally $\eta_J=1$. For the purposes of convergence analysis in Section 5.4 $\eta_J$ is set to larger than 1.} Denote by $\mathbf U^k = \mathbf Z^{k+1}\mathbf R - \frac{1}{\sigma^k_J} \mathbf Y^k$, then the above problem has a closed-form solution defined as follows,
\begin{align}
\label{SolutionJ_relaxed}
\mathbf J_i = \begin{cases} {\displaystyle \frac{\|\mathbf u_i\| - \frac{\lambda_2}{\sigma^k_J}}{\|\mathbf u_i\|}}\mathbf u_i & \textrm{if } \|\mathbf u_i\| > \frac{\lambda_2}{\sigma^k_J} \\
0 & \textrm{otherwise}
\end{cases}
\end{align}
where $\mathbf J_j$ and $\mathbf u_i$ are the $i$-th columns of $\mathbf J^{k+1}$ and $\mathbf U^k$, respectively. Please refer to \cite{liu2010robust}.

\item Update $\mathbf Y^{k+1}$ by
\begin{align}
\mathbf Y^{k+1} =& \; \mathbf Y^k + \mu^k (\mathbf J^{k+1} - \mathbf Z^{k+1} \mathbf R)  \label{UpdateYY}
\end{align}

\item Update adaptive constant $\mu^{k+1}$ by
\begin{align*}
 \mu^{k+1} = \textrm{min}( \mu_{\text{max}_1}, \gamma \mu^k) 
\end{align*}
\end{enumerate}

The entire procedure for solving the relaxed OSC objective is summarized in Algorithm 1. This set of updating rules is a generalisation of those in LADMAP \cite{LinLiuSu2011}, as such it will be referred to as v-LADMAP for short. Note however, that in the original LADMAP the linearisation is performed {\em only} on the augmented Lagrange term, i.e. on $h(\mathbf Z)$, based on which the convergence analysis is carried out. Whereas in the v-LADMAP case, both $h(\mathbf Z)$ and $g(\mathbf Z)$ are linearised in order to obtain a closed-form solution to $\mathbf Z$. This difference means that the convergence analysis in LADMAP is no longer applicable here. As such, detailed analysis on the convergence of v-LADMAP is provided in Section \ref{sec:convergence}.

\subsection{Exact Constraints}

Similar to the relaxed version, auxiliary constraint variables are introduced
\begin{align}
\label{objective_exact}
\min_{\mathbf Z, \mathbf E, \mathbf J} \frac12\|\mathbf E\|^2_F +\lambda_1\|\mathbf Z\|_{1}+\lambda_2\| \mathbf J \|_{1,2}\\
\text{s.t.} \quad \mathbf{X = XZ + E}, \mathbf{J = ZR} \nonumber
\end{align}
Then the Augmented Lagrangian form is used to incorporate the constraints
\begin{align}
& \mathcal{L}(\mathbf E, \mathbf Z, \mathbf J, \mathbf Y_1, \mathbf Y_2 | \mu) \notag\\
= & \frac12\|\mathbf E\|^2_F + \lambda_1 \|\mathbf Z\|_{1} + \lambda_2\|\mathbf J\|_{1,2} \notag\\
& + \langle \mathbf Y_1, \mathbf{XZ} - \mathbf X + \mathbf E \rangle +\frac{\mu}2\|\mathbf{XZ} - \mathbf X + \mathbf E\|^2_F \notag\\
& + \langle \mathbf Y_2, \mathbf J - \mathbf Z\mathbf R\rangle +\frac{\mu}2\|\mathbf J - \mathbf Z\mathbf R\|^2_F \label{Eq:24Sept2014-4}
\end{align}

In problem \eqref{Eq:24Sept2014-4}, there are three primary variables $\mathbf Z$, $\mathbf E$ and $\mathbf J$, so a simple linearised ADM as used in the previous subsection may diverge in the multi-variable case as demonstrated in \cite{DBLP:conf/acml/LiuLS13}. To overcome this the so-called Linearized Alternating Direction Method with Parallel Splitting and Adaptive Penalty method (LADMPSAP) is adopted, which for problem \eqref{Eq:24Sept2014-4}, consists of the following steps, see \cite{DBLP:conf/acml/LiuLS13}:

\begin{enumerate}




\item  Update $\mathbf Z^{k+1}$
\begin{align}
\min_{\mathbf Z} &\lambda_1 \| \mathbf Z \|_1 + \langle \mathbf Y^{k}_1, \mathbf{XZ} - \mathbf X + \mathbf E^k \rangle + \frac{\mu^k}2\|\mathbf{XZ} - \mathbf X + \mathbf E^k\|^2_F \notag \\
&+ \langle \mathbf Y^{k}_2, \mathbf J^k - \mathbf Z\mathbf R\rangle +\frac{\mu^k}2\|\mathbf J^k - \mathbf Z\mathbf R\|^2_F \label{Eq:24Sept2014-5}
\end{align}

Define
\begin{align*}
F(\mathbf Z) = &\langle \mathbf Y^{k}_1, \mathbf{XZ} - \mathbf X + \mathbf E^k \rangle + \frac{\mu^k}2\|\mathbf{XZ} - \mathbf X + \mathbf E^k\|^2_F\\
&+ \langle \mathbf Y^{k}_2, \mathbf J^k - \mathbf Z\mathbf R\rangle + \frac{\mu^k}2\|\mathbf J^k - \mathbf Z\mathbf R\|^2_F.
\end{align*}

By linearizing $F$, \eqref{Eq:24Sept2014-5} can be approximated with the following proximal problem
\begin{align*}
\min_{\mathbf Z} \lambda_1 \| \mathbf Z \|_1 + \frac{\sigma^k_z}{2} \| \mathbf Z - (\mathbf Z^{k} - \frac{1}{\sigma^k_z} \nabla F(\mathbf Z^k)) \|_F^2
\end{align*}
where $\sigma^k_z = \mu^k\eta_z$ ($\eta_z$ is an appropriate constant) and
\begin{align*}
\nabla F(\mathbf Z^k) = & \mathbf X^T (\mathbf Y^{k}_1 + \mu^k (\mathbf{XZ}_{k} - \mathbf X + \mathbf E^k)) - (\mathbf Y^{k}_2 + \mu^k(\mathbf J^k - \mathbf Z^k\mathbf R))\mathbf R^T.
\end{align*}

As discussed before the solution is given by the soft thresholding operator defined in \eqref{SolutionZ_relaxed} with $\mathbf V = (\mathbf Z^k - \frac{1}{\sigma^k_z} \nabla F(\mathbf Z^k))$.

\item Update $\mathbf E^{k+1}$ by
\begin{align*}
\min_{\mathbf E} \frac12\|\mathbf E\|^2_F + \langle \mathbf Y^{k}_1, \mathbf{X}\mathbf Z^k - \mathbf X + \mathbf E \rangle  + \frac{\mu^k}2\|\mathbf{X}\mathbf Z^k - \mathbf X + \mathbf E\|^2_F
\end{align*}
This is a least square problem whose solution can be given by
\begin{align}
\label{SolutionE_exact}
\mathbf E^{k+1} = \frac{\mathbf{X}\mathbf Z^k - \mathbf X + \frac{1}{\mu^k} \mathbf Y^{k}_1}{\frac{1}{\mu^k} + 1}
\end{align}

\item Update $\mathbf J^{k+1}$ by
\begin{align*}
\min_{\mathbf J} \lambda_2\|\mathbf J\|_{1,2} + \langle \mathbf Y^{k}_2, \mathbf J - \mathbf Z^k\mathbf R\rangle +\frac{\mu^k}2\|\mathbf J - \mathbf Z^k\mathbf R\|^2_F.
\end{align*}

The solution can be obtained by using \eqref{SolutionJ_relaxed} with $\mathbf Z^{k+1}$ replaced by $\mathbf Z^k$.




\item Update multipliers with the new values of primary variables by
\begin{align*}
\mathbf Y^{k+1}_1 =& \mathbf Y^k_1 + \mu^k (\mathbf{XZ}^{k+1} - \mathbf X + \mathbf E^{k+1})\\
\mathbf Y^{k+1}_2 =& \mathbf Y^k_2 + \mu^k (\mathbf J^{k+1} - \mathbf Z^{k+1} \mathbf R)
\end{align*}

\item Update $\mu^{k+1}$
\begin{align*}
 \mu^{k+1} = \textrm{min}( \mu_{\text{max}_1}, \gamma \mu^k)
\end{align*}

\end{enumerate}

The entire procedure is summarised in Algorithm \ref{alg_exact}.

\begin{algorithm}
\caption{{\bf Solving \eqref{objective_relaxed} by v-LADMAP}}
\label{alg_relaxed}
\begin{algorithmic}[1]

\REQUIRE $\mathbf X^{D \times N}$ - observed data, $\lambda_1$, $\lambda_2$ -  regularisation parameters, $\mu$, $\mu^{\text{max}} \gg \mu$ - rate of descent parameters and $\epsilon_1, \epsilon_2 > 0$.

\STATE Initialise $\mathbf J^k = \mathbf 0^{N \times N-1}$, $\mathbf Y^k = \mathbf 1^{N \times N-1}$, $\mathbf Z^k = \mathbf 0^{N \times N}$

\WHILE{not converged}

\STATE Find $\mathbf Z^{k+1}$ by using \eqref{SolutionZ_relaxed}

\STATE Find $\mathbf J^{k+1}$ by using \eqref{SolutionJ_relaxed}

\STATE Check stopping criteria
\[
\|\mathbf J^{k+1} - \mathbf Z^{k+1} \mathbf R \|_F < \epsilon_1
\]
\[
\mu^k \textrm{max} ( \| \mathbf Z^{k+1} - \mathbf Z^{k}  \|_F  , \|  \mathbf J^{k+1} - \mathbf J^{k} \|_F) < \epsilon_2
\]

\STATE $\mathbf Y^{k+1} = \mathbf Y^k + \mu^k (\mathbf J^{k+1} - \mathbf Z^{k+1} \mathbf R)$

\STATE Update $\gamma$
\[
\gamma =
\begin{cases}
\gamma^0 & \text{if} \;\; \mu^k \textrm{max} \{ \| \mathbf Z^{k+1} - \mathbf Z^{k}  \|_F , \\
& \phantom{\text{if} \;\; \mu^k \textrm{max} \{}\|  \mathbf J^{k+1} - \mathbf J^{k} \|_F\} < \epsilon_2 \\
1 & \text{otherwise,}
\end{cases}
\]

\STATE $ \mu^{k+1} = \textrm{min}( \mu_{\text{max}_1}, \gamma \mu^k)$

\ENDWHILE

\RETURN $\mathbf Z$
\end{algorithmic}
\end{algorithm}

\begin{algorithm}
\caption{{\bf Solving \eqref{objective_exact} by LADMPSAP}}
\label{alg_exact}
\begin{algorithmic}[1]

\REQUIRE $\mathbf X^{D \times N}$ - observed data, $\lambda_1$, $\lambda_2$ -  regularisation parameters, $\mu$, $\mu^{\text{max}} >> \mu$, $\rho > \| \mathbf X \|^2$ - rate of descent parameters and $\epsilon_1, \epsilon_2 > 0$.

\STATE Initialise $\mathbf S = \mathbf 0^{N \times N}$, $\mathbf U = \mathbf S \mathbf R$, $\mathbf Y_1 = \mathbf 1^{N \times N}$, $\mathbf Y_2 = \mathbf 1^{N \times N-1}$, $\mathbf Z = \mathbf 0^{N \times N}$

\WHILE{not converged}

\STATE Find $\mathbf Z^{k+1}$ by using \eqref{SolutionZ_exact}

\STATE Find $\mathbf E^{k+1}$ by using \eqref{SolutionE_exact}

\STATE Find $\mathbf J^{k+1}$ by using \eqref{SolutionJ_relaxed} with $\mathbf{Z}^{k+1}$ replaced by $\mathbf{Z}^k$

\STATE Check stopping criteria
\begin{align*}
&\frac{\|\mathbf{XZ^{k+1} - X + E^{k+1}} \|_F}{\| \mathbf X \|_F} < \epsilon_1;
 \\
&\frac{\|\mathbf{J^{k+1} - Z^{k+1}R}\|_F}{\| \mathbf X \|_F} < \epsilon_1;
\\
\frac{\mu^k \sqrt{\rho}}{\| \mathbf X \|_F}&\textrm{max} \big\{ \| \mathbf Z^{k+1} - \mathbf Z^{k}  \|_F  , \|  \mathbf E^{k+1} - \mathbf E^{k} \|,\\
&\| \mathbf J^{k+1} -  \mathbf J^{k}  \|_F,
\|  \mathbf Z^{k+1} \mathbf R - \mathbf Z^{k} \mathbf R \|_F)\big\}   < \epsilon_2
\end{align*}

\STATE $\mathbf Y^{k+1}_1 = \mathbf Y_1^k + \mu^k_1 ( \mathbf{X Z^{k+1}} - \mathbf X + \mathbf E^{k+1})$

\STATE $\mathbf Y^{k+1}_2 = \mathbf Y_2^k + \mu^k_2 (\mathbf J^{k+1} - \mathbf S^{k+1} \mathbf R)$

\STATE Update $\gamma$
\[
\gamma_1 =
\begin{cases}
\gamma^0 & \text{if} \;\; \frac{\mu^k \sqrt{\rho}}{\| \mathbf X \|_F}   \textrm{max} \{ \| \mathbf Z^{k+1} - \mathbf Z^{k}  \|_F, \\
&\phantom{\text{if}\;\;}\|  \mathbf E^{k+1} - \mathbf E^{k} \|, \| \mathbf J^{k+1} - \mathbf J^{k}  \|_F,  \\
&\phantom{\text{if}\;\;}\|  \mathbf Z^{k+1} \mathbf R - \mathbf Z^{k} \mathbf R \|_F\}   < \epsilon_2 \\
1 & \text{otherwise,}
\end{cases}
\]

\STATE $ \mu^{k+1} = \textrm{min}( \mu_{\text{max}_1}, \gamma \mu^k)$

\ENDWHILE

\RETURN $\mathbf Z$
\end{algorithmic}
\end{algorithm}

%

\subsection{Diagonal Constraint}

In some cases, it may be desirable to enforce the constraint $\textrm{diag}(\mathbf Z) = \mathbf 0$ i.e.\ we should not allow each data point to be represented by itself. The objective becomes
\begin{align}
\min_{\mathbf Z, \mathbf E} \frac12\|\mathbf E \|^2_F +\lambda_1\|\mathbf Z\|_{1}+\lambda_2\|\mathbf Z\mathbf R\|_{1,2} \\
\text{s.t.} \quad \mathbf{X = XZ + E}, \textrm{diag}(\mathbf Z) = \mathbf 0 \nonumber
\end{align}
To enforce such a constraint it is not necessary to significantly alter the aforementioned optimisation schemes. This constraint only affects the step involving $\mathbf Z$. Since this step is the soft shrinkage operator and is separable at the element level one can simply set the diagonal entries to $0$ afterwards. In other words  the $\mathbf Z$ update solution \eqref{SolutionZ_relaxed}  becomes 
\begin{align}
\label{SolutionZ_exact}
Z_{ij} =
\begin{cases}
0 & \textrm{if} \; i = j \\
\textrm{sign}\left(V_{ij} \right) \max\left(\left|V_{ij} \right| - \frac{\lambda_1}{\rho}\right) & \text{otherwise,}
\end{cases}
\end{align}

\subsection{Convergence Analysis for Algorithms}\label{sec:convergence}
LADMPSAP adopts a special strategy that updates all primary variables in parallel using their values from the last iteration. See Step 2 to Step 11 and the equations they refer to. The LADMPSAP algorithm for problem \eqref{objective_exact} is guaranteed to converge. For the convergence analysis, please refer to \cite{DBLP:conf/acml/LiuLS13}. The convergence theorem is repeated here with some modifications reflecting the settings in our problem.
\begin{thm}\label{Theorem1}If $\mu^k$ is non-decreasing and upper bounded, $\eta_z > \|\mathbf X\|^2+\|\mathbf R\|^2$, then the sequence $\{(\mathbf Z^k, \mathbf E^k, \mathbf J^k, \mathbf Y^k_1, \mathbf Y^k_2)\}$ generated by Algorithm 2 converges to a KKT point of problem \eqref{objective_exact}.
\end{thm}

Differently in v-LADMAP, which is used to solve the relaxed objective, updating the primary variables is performed in sequence. Meaning that one updated primary variable is used immediately to update another primary variable so that the optimisation is carried out by alternating directions sequentially. In Step 2 in Algorithm 1, the updated value of $\mathbf Z^{k+1}$ is used to obtain $\mathbf J^{k+1}$. The proof of convergence for LADMAP does not completely extend to v-LADMAP and since variables are updated sequentially the convergence from LADMPSAP does not apply either. As such the convergence theorem for v-LADMAP is presented in the remainder of this section.

Consider the original relaxed constrained version \eqref{objective_relaxed}. The KKT conditions of problem  \eqref{objective_relaxed} lead to the following: there exists a triplet $(\mathbf Z^*, \mathbf J^*, \mathbf Y^*)$ such that
\begin{align}
\mathbf J^* = \mathbf Z^*\mathbf R;  \ \ \ -\mathbf Y^*\in\lambda_2\partial \|\mathbf J^*\|_{1,2} \\
\mathbf X^T(\mathbf X - \mathbf X \mathbf Z^*) + \mathbf Y^*\mathbf R^T \in \lambda_1\partial \|\mathbf Z^*\|_1.
\end{align}
where $\partial$ denotes the subdifferential.

\begin{lemma}\label{Lemma1}The following relations hold
\begin{align}
 \mathbf T^k_Z \triangleq & -(\sigma^k_z+L_z)(\mathbf Z^{k+1} - \mathbf Z^k) + \widetilde{\mathbf Y}^k\mathbf R^T \\ &+\mathbf X^T\mathbf X(\mathbf Z^{k+1} - \mathbf Z^k) \in  \nabla g(\mathbf Z^{k+1}) + \lambda_1 \partial \|\mathbf Z^{k+1}\|_1 \notag\\
\mathbf T^k_J \triangleq &-\sigma^k_J (\mathbf J^{k+1} - \mathbf J^k) - \widehat{\mathbf Y}^k \in \lambda_2\partial \|J^{k+1}\|_{1,2},
\end{align}
where
\begin{align}
&\widetilde{\mathbf Y}^k = \mathbf Y^k +  \mu^k(\mathbf J^k - \mathbf Z^k\mathbf R) \label{Eq:25Sept2014-1}\\
&\widehat{\mathbf Y}^k = \mathbf Y^k +  \mu^k(\mathbf J^k - \mathbf Z^{k+1}\mathbf R)\label{Eq:25Sept2014-2}
\end{align}
\end{lemma}
\begin{proof}
Checking the optimality conditions of two subproblems \eqref{ForZ} and \eqref{Eq:25Sept2014-5} for $\mathbf{Z}^{k+1}$ and $\mathbf{J}^{k+1}$ leads to the above claims.
\end{proof}


\begin{lemma}\label{Lemma2}For the sequence generated by Algorithm \ref{alg_relaxed} the following identity holds
\begin{align}
&(\eta_z+L_z(\mu^k)^{-1})\|\mathbf Z^{k+1}-\mathbf Z^*\|^2_F - \|(\mathbf Z^{k+1}-\mathbf Z^*)\mathbf R\|^2_F \notag\\
& \phantom{\eta_z}+ \eta_J\|\mathbf J^{k+1}-\mathbf J^*\|^2_F + (\mu^k)^{-2}\|\mathbf Y^{k+1}-\mathbf Y^*\|^2_F \notag\\
=&(\eta_z+L_z(\mu^k)^{-1})\|\mathbf Z^{k}-\mathbf Z^*\|^2_F - \|(\mathbf Z^{k}-\mathbf Z^*)\mathbf R\|^2_F \notag\\
& \phantom{\eta_z}+ \eta_J\|\mathbf J^{k}-\mathbf J^*\|^2_F + (\mu^k)^{-2}\|\mathbf Y^{k}-\mathbf Y^*\|^2_F\\
&-\{ (\mu^k)^{-2} \|\mathbf Y^{k+1} - \mathbf Y^k\|^2_F + \eta_J  \|\mathbf J^{k+1} - \mathbf J^k\|^2_F \notag\\
& \phantom{\eta_z\eta_z} -2 (\mu^k)^{-1}\langle \mathbf Y^{k+1} - \mathbf Y^k, \mathbf J^{k+1} - \mathbf J^k\rangle\} \label{Eq:39}
\\
&-((\eta_z+L_z(\mu^k)^{-1}) \|\mathbf Z^{k+1} - \mathbf Z^k\|^2_F - \|(\mathbf Z^{k+1} - \mathbf Z^k)\mathbf R\|^2_F ) \label{Eq:40}  \\
\phantom{=}&-2(\mu^k)^{-1}\big\langle \mathbf Z^{k+1} - \mathbf Z^*, \mathbf T^k_Z  - \mathbf Y^* \mathbf R^T\big\rangle  \label{Eq:41}\\
&-2(\mu^k)^{-1}\big\langle \mathbf J^{k+1} - \mathbf J^*,  \mathbf T^k_J  + \mathbf Y^*\big\rangle  \label{Eq:42}\\
&+2(\mu^k)^{-1}\big\langle \mathbf Z^{k+1} - \mathbf Z^*, \mathbf X^T\mathbf X(\mathbf Z^{k+1} - \mathbf Z^k)\big\rangle  \label{Eq:43}
\end{align}
where $L_z$, $\eta_z$ and $\eta_J$ are the constants used in linearisation \eqref{Eq:25Sept2014-3},  \eqref{Eq:25Sept2014-4} and \eqref{Eq:25Sept2014-5}, respectively.
\end{lemma}


\begin{proof} This identity can be checked by using the definition of $\widetilde{\mathbf Y}^k$ and $\widehat{\mathbf Y}^k$, see \eqref{Eq:25Sept2014-1} and \eqref{Eq:25Sept2014-2}, and using the following identities
\begin{align*}
&\mathbf J^* = \mathbf Z^*\mathbf R;\\
&2\langle \mathbf a - \mathbf b, \mathbf a - \mathbf c\rangle  = \|\mathbf a - \mathbf b\|^2 - \|\mathbf b - \mathbf c\|^2+\|\mathbf a - \mathbf c\|^2,
\end{align*}
as well as the updating rule for $\mathbf Y^{k+1}$, see \eqref{UpdateYY}. Since the full proof is lengthy and tedious it is omitted here.
\end{proof}

Before the most important lemma it is necessary to introduce the following inequalities.
\begin{lemma}\label{Lemma2.1} The following inequalities hold with $\eta_J>1$
\begin{align}
&D_Y^{k}\triangleq(\mu^k)^{-2} \|\mathbf Y^{k+1} - \mathbf Y^k\|^2_F + \eta_J  \|\mathbf J^{k+1} - \mathbf J^k\|^2_F \notag\\
& \phantom{\eta_z\eta_z} -2 (\mu^k)^{-1}\langle \mathbf Y^{k+1} - \mathbf Y^k, \mathbf J^{k+1} - \mathbf J^k\rangle \ge 0 \label{Eq:39inq}\\
&\big\langle \mathbf Z^{k+1} - \mathbf Z^*,  \mathbf T^k_Z  - \mathbf Y^* \mathbf R^T\big\rangle \ge 0 \label{Eq:41inq}\\
&\big\langle \mathbf J^{k+1} - \mathbf J^*, \mathbf T^k_J  + \mathbf Y^*\big\rangle \ge 0. \label{Eq:42inq}
\end{align}
\end{lemma}
\begin{proof}
\eqref{Eq:39inq} is due to Cauchy inequality. \eqref{Eq:41inq} and \eqref{Eq:42inq} are the results of combining the convexity of the objective functions \eqref{ForZ} and \eqref{Eq:25Sept2014-5}, which are used to update $\mathbf Z^{k+1}$ and $\mathbf J^{k+1}$ respectively, with the following inequality
\[
\big\langle x-y, p_x - p_y\big\rangle \ge 0, \ \ \forall p_x\in\partial f(x) \text{ and } p_y\in\partial f(y) 
\] where $f(x)$ is any convex function.
\end{proof}

Next the most important lemma is presented.
\begin{lemma}\label{Lemma3}If $\mu^k$ is increasing, $\eta_z >\|\mathbf R\|^2$, $\eta_J>1$, $\mu^{k+1} - \mu^k \geq L_z/(\eta_z - \|\mathbf R\|^2)$ and $(\mathbf Z^*, \mathbf J^*, \mathbf Y^*)$ is any KKT point of problem \eqref{objective_relaxed}, then the sequence generated by Algorithm 1 satisfies
\begin{enumerate}
\item $s^k\triangleq(\eta_z+L_z(\mu^k)^{-1})\|\mathbf Z^{k}-\mathbf Z^*\|^2_F - \|(\mathbf Z^{k}-\mathbf Z^*)\mathbf R\|^2_F + \eta_J\|\mathbf J^{k}-\mathbf J^*\|^2_F + (\mu^k)^{-2}\|\mathbf Y^{k}-\mathbf Y^*\|^2_F$ is nonnegative and nonincreasing;
\item $\|\mathbf Z^{k+1} - \mathbf Z^k\|_F \rightarrow 0$, $\|\mathbf J^{k+1} - \mathbf J^k\|_F \rightarrow 0$, and $\|\mathbf Y^{k+1} - \mathbf Y^k\|_F \rightarrow 0$.
\end{enumerate}
\end{lemma}

\begin{proof} For claim 1) note that
\begin{align*}
& 2\langle \mathbf Z^{k+1} - \mathbf Z^*, \mathbf X^T\mathbf X(\mathbf Z^{k+1} - \mathbf Z^k)\big\rangle \\
\leq & 2\|\mathbf X^T\mathbf X\| \|\mathbf Z^{k+1} - \mathbf Z^k\|\|\mathbf Z^{k+1} - \mathbf Z^*\|\\
\leq & L_z \left(\frac{\mu^{k+1}}{\mu^{k+1}-\mu^k}\|\mathbf Z^{k+1} - \mathbf Z^k\|^2\right.
+ \left.\frac{\mu^{k+1}-\mu^k}{\mu^{k+1}} \|\mathbf Z^{k+1} - \mathbf Z^*\|^2\right)
\end{align*}
where we have chosen $L_z = \|\mathbf X\|^2$.

Then from Lemma \ref{Lemma2} and Lemma \ref{Lemma2.1} and noting $\mu^k$ is increasing, we have the following
\begin{align}
s^{k+1}\leq & s^{k} - D^k_Y + \|\mathbf Z^{k+1} - \mathbf Z^k\|^2_F\|\mathbf R\|^2_F 
-(\eta_z-\frac{L_z}{\mu^{k+1} - \mu^k}) \|\mathbf Z^{k+1} - \mathbf Z^k\|^2_F \label{Eq:44} 
\end{align}

Now it is easy to check that when
\begin{align}
\mu^{k+1} - \mu^k \geq \frac{L_z}{\eta_z - \|\mathbf R\|^2_F} \label{Eq:49}
\end{align}
the sum of last two terms in \eqref{Eq:44} is nonnegative. Hence the claim 1) has been proved.

Regarding claim 2), $s^k$ is non-increasing and nonnegative, thus it must have a limit, denoted by $s^{\infty}$. If we take the sum over \eqref{Eq:44} for all the iterations $k$, we have
\begin{align*}
&\sum^{\infty}_{k=1} D_Y^k+\sum^{\infty}_{k=1}  (\eta_z- \frac{L_z}{\mu^{k+1}-\mu^k} - \|\mathbf R\|^2_F) \|\mathbf Z^{k+1} - \mathbf Z^k\|^2_F \leq s^1-s^{\infty}
\end{align*}

Hence 
$\sum^{\infty}_{k=1}D_Y^k$ and $\sum^{+\infty}_{k=1}\|\mathbf Z^{k+1}-\mathbf Z^k\|^2_F$ are bounded, under the condition \eqref{Eq:49}.  This gives
\[
\|\mathbf Z^{k+1}-\mathbf Z^k\|_F \rightarrow 0 \text{ and } D_Y^k \rightarrow 0.
\]

It is easy to check that
\[
D_Y^k \geq (\eta_J - 1)\|\mathbf J^{k+1} - \mathbf J^k\|^2_F,
\]
which means $\|\mathbf J^{k+1}-\mathbf J^k\|_F \rightarrow 0.$

By using the similar strategy, we have
\[
D_Y^k \geq ((\mu^k)^{-1}\|\mathbf Y^{k+1} - \mathbf Y^k\| - \sqrt{\eta_J}\|\mathbf J^{k+1} - \mathbf J^k\|)^2.
\]
Hence we have $\|\mathbf Y^{k+1}-\mathbf Y^k\|_F \rightarrow 0.$
This completes the proof for claim 2).
\end{proof}

\begin{thm}\label{Theorem4}Under the conditions of Lemma \ref{Lemma3}, the sequence $\{(\mathbf Z^k, \mathbf J^k, \mathbf Y^k)\}$ generated by Algorithm 1 converges to a KKT point of problem \eqref{objective_relaxed}.
\end{thm}

\begin{proof} By claim 2) of Lemma \ref{Lemma3}, we know that the sequence $\{(\mathbf Z^k, \mathbf J^k, \mathbf Y^k)\}$ is bounded, hence it has an accumulation point, denoted by
\[
(\mathbf Z^{k_n}, \mathbf J^{k_n}, \mathbf Y^{k_n}) \rightarrow (\mathbf Z^{\infty}, \mathbf J^{\infty}, \mathbf Y^{\infty}).
\]

First we prove that $(\mathbf Z^{\infty}, \mathbf J^{\infty}, \mathbf Y^{\infty})$ is a KKT point of problem \eqref{objective_relaxed}. According to update rule \eqref{UpdateYY}, we have
\[
\mathbf J^{k+1} - \mathbf Z^{k+1}\mathbf R = (\mu^k)^{-1}(\mathbf Y^{k+1} - \mathbf Y^k) \rightarrow 0.
\]
This shows $\mathbf J^{\infty} = \mathbf Z^{\infty}\mathbf R$, i.e., any accumulation point is a feasible solution.

Taking $k = k_n-1$ in Lemma \ref{Lemma1} and using the definition of subdifferential, we have
\begin{align*}
&\frac12\|\mathbf X-\mathbf X\mathbf Z^{k_n}\|^2_F + \lambda_1\|\mathbf Z^{k_n}\|_1 + \lambda_2 \|\mathbf J^{k_n}\|_{1,2}\\
\leq & \frac12\|\mathbf X-\mathbf X\mathbf Z^*\|^2_F + \lambda_1\|\mathbf Z^*\|_1 + \lambda_2 \|\mathbf J^*\|_{1,2}\\
&+\langle \mathbf Z^{k_n} - \mathbf Z^*, -(\sigma^{k_n-1}_z+L_z)(\mathbf Z^{k_n} - \mathbf Z^{k_n-1}) \widetilde{\mathbf Y}^{k_n-1}\mathbf R^T
+ \mathbf X^T\mathbf X(\mathbf Z^{k_n} - \mathbf Z^{k_n-1})\rangle \\
& + \langle \mathbf J^{k_n} - \mathbf J^*, - \sigma^{k_n-1}_J(\mathbf J^{k_n} - \mathbf J^{k_n-1}) - \widehat{\mathbf Y}^{k_n-1} \rangle.
\end{align*}

Let $n\rightarrow +\infty$, by using the fact in Lemma \ref{Lemma3}, we can see
\begin{align*}
&\frac12\|\mathbf X-\mathbf X\mathbf Z^{\infty}\|^2_F + \lambda_1\|\mathbf Z^{\infty}\|_1 + \lambda_2 \|\mathbf J^{\infty}\|_{1,2}\\
\leq & \frac12\|\mathbf X-\mathbf X\mathbf Z^*\|^2_F + \lambda_1\|\mathbf Z^*\|_1 + \lambda_2 \|\mathbf J^*\|_{1,2}\\
&\langle \mathbf Z^{\infty} - \mathbf Z^*, \mathbf Y^{\infty}\mathbf R^T\rangle + \langle \mathbf J^{\infty} - \mathbf J^*, - {\mathbf Y}^{\infty} \rangle\\
= & \frac12\|\mathbf X-\mathbf X\mathbf Z^*\|^2_F + \lambda_1\|\mathbf Z^*\|_1 + \lambda_2 \|\mathbf J^*\|_{1,2}.
\end{align*}
because $(\mathbf Z^{\infty}, \mathbf J^{\infty})$ is a feasible solution. So we conclude that $(\mathbf Z^{\infty}, \mathbf J^{\infty})$ is actually an optimal solution to  \eqref{objective_relaxed}.

 In a similar way, from Lemma \ref{Lemma1}, we can also prove that $-\mathbf Y^{\infty}\in \lambda_2\partial \|J^{\infty}\|_{1,2}$ and $\mathbf X^T(\mathbf X - \mathbf X\mathbf Z^{\infty}) + \mathbf Y^{\infty}\mathbf R^T \in\lambda_1\partial \|\mathbf Z^{\infty}\|_1$. Therefore,  $(\mathbf Z^{\infty}, \mathbf J^{\infty}, \mathbf Y^{\infty})$ is a KKT point of problem \eqref{objective_relaxed}.

Taking $(\mathbf Z^{*}, \mathbf J^{*}, \mathbf Y^{*}) = (\mathbf Z^{\infty}, \mathbf J^{\infty}, \mathbf Y^{\infty})$ in Lemma \ref{Lemma3}, we have $(\eta_z+L_z(\mu^{k_n})^{-1})\|\mathbf Z^{k_n}-\mathbf Z^{\infty}\|^2_F - \|(\mathbf Z^{k_n}-\mathbf Z^{\infty})\mathbf R\|^2_F + \eta_J\|\mathbf J^{k_n}-\mathbf J^{\infty}\|^2_F + (\mu^{k_n})^{-2}\|\mathbf Y^{k_n}-\mathbf Y^{\infty}\|^2_F \rightarrow 0.$  With claim 1) of Lemma \ref{Lemma3}, we have $(\eta_z+L_z(\mu^k)^{-1})\|\mathbf Z^{k}-\mathbf Z^{\infty}\|^2_F - \|(\mathbf Z^{k}-\mathbf Z^{\infty})\mathbf R\|^2_F + \eta_J\|\mathbf J^{k}-\mathbf J^{\infty}\|^2_F + (\mu^k)^{-2}\|\mathbf Y^{k}-\mathbf Y^{\infty}\|^2_F \rightarrow 0$ for all $k\rightarrow +\infty$. Hence $(\mathbf Z^{k}, \mathbf J^{k}, \mathbf Y^{k}) = (\mathbf Z^{\infty}, \mathbf J^{\infty}, \mathbf Y^{\infty})$.

As $(\mathbf Z^{\infty}, \mathbf J^{\infty}, \mathbf Y^{\infty})$ can be any accumulation point of $\{(\mathbf Z^{k}, \mathbf J^{k}, \mathbf Y^{k})\}$, we  conclude that  $\{(\mathbf Z^{k}, \mathbf J^{k}, \mathbf Y^{k})\}$ converges to a KKT point of problem \eqref{objective_relaxed}.  This completes the proof of the theorem.
\end{proof}

\section{Segmentation}

Once a solution to \eqref{objective} has been found, the next step is to use the information encoded in $\mathbf Z$ to produce subspace labels for each data point. In the simple case that $\mathbf Z$ is strictly block-diagonal, one can use the non-zero columns of matrix $\mathbf{ZR}$ to identify the change from one block (subspace) to another since $\mathbf Z\mathbf R = [\mathbf z_2 - \mathbf z_1, \mathbf z_3-\mathbf z_2, ..., \mathbf z_N - \mathbf z_{N-1}]$. By strictly block-diagonal matrix we mean a matrix with columns already ordered so that its appearance is block-diagonal. We distinguish strictly block-diagonal with general block-diagonal matrices which can be strictly block-diagonal once reordered. In this field general block-diagonal matrices refers to matrices where the non-zero entries link each data point to every other data point in the same subspace \cite{feng2013robust}.  However obtaining a strictly block-diagonal matrix is rarely seen in practice due to noisy data and it assumes that the subspace will only occur once in the sequence.

The case of strictly block-diagonal $\mathbf W$ (as defined in \eqref{w_def}) is a special case of the more general unique and connected subspace assumption. Under this assumption we know that once a subspace stops occurring in the sequence of data it will never occur again. In practice it is unlikely that $\mathbf W$ will be exactly block-diagonal but we often assume that subspaces will be unique and connected. When $\mathbf Z$ is not block-diagonal there will be a large number of non-zero columns of $\mathbf{ZR}$, therefore we cannot use the earlier described method. Instead one can apply some minor post processing to $\mathbf{ZR}$ to find the boundaries of subspaces. First let $\mathbf B = |\mathbf{ZR}|$ be the absolute value matrix of $\mathbf{ZR}$. Then let $\bar{\mathbf b}$ be the vector of column-wise means of $\mathbf B$. Then we employ a peak finding algorithm over $\bar{\mathbf b}$, where the peaks are likely to correspond to the boundaries of subspaces.

A more robust approach is to use use spectral clustering. The matrix $\mathbf Z$ is used to build an affinity matrix of an undirected graph. The affinity matrix or similarity graph is defined as $\mathbf W = |\mathbf Z| + |\mathbf Z|^T$. Element $W_{ij}$ corresponds to the edge weight or affinity between vertices (data points) $i$ and $j$. Then we use the spectral clustering technique, Normalised Cuts (NCUT) \cite{shi2000normalized}, to obtain final segmentation. NCUT has been shown to be robust in subspace segmentation tasks and is considered state of the art \cite{elhamifar2012sparse, liu2010robust}. In cases where $\mathbf Z$ is not block-diagonal or contains significant noise, NCUT will provide better segmentation than other spectral clustering methods. Algorithm \ref{alg_final} summarises the entire proposed algorithm.

\begin{algorithm}
\caption{Ordered Subspace Clustering Procedure}
\label{alg_final}
\begin{algorithmic}[1]

\REQUIRE $\mathbf X^{D \times N}$ - observed data

\STATE Obtain the sparse coefficients $\mathbf Z$ by solving the relaxed or exact objective

\STATE Form the similarity graph $\mathbf W = |\mathbf Z| + |\mathbf Z|^T$

\STATE Estimate the number of subspaces $k$ from $\mathbf W$

\STATE Apply spectral clustering to $\mathbf W$ to partition the data into $k$ subspaces

\RETURN Subspaces $\{S_i\}^k_{i=1}$

\end{algorithmic}
\end{algorithm}

\subsection{Estimating the number of subspaces}

Spectral clustering techniques require the number of clusters to be declared before hand. In many cases the number of subspaces in the data is unknown. Fortunately the number of clusters can be estimated from the affinity matrix $\mathbf W$.
Here we suggest some possible estimation techniques.

For general unordered block-diagonal $\mathbf Z$ or where the subspaces may be reoccurring in the sequence the number of clusters can be obtained from the singular values of $\mathbf W$ \cite{6180173}. Specifically the number of non-zero singular values of $\mathbf Z$ or the number of zero singular values from the Laplacian matrix $\mathbf L$ corresponds to the number of blocks or subspaces. The Laplacian matrix is defined as $\mathbf{L = D - W}$ and $\mathbf D$ is a diagonal matrix where $D_{ii} = \sum_j W_{ij}$. Prior work has suggested using the normalised Laplacian $\mathbf L_n = \mathbf{I - D}^{-\frac{1}{2}} \mathbf{W D}^{-\frac{1}{2}}$ however we noted no structural difference between $\mathbf L$ and $\mathbf L_n$.

Singular values can be used even when the matrix $\mathbf W$ is not block diagonal due to noise. In this setting noise refers to non-zero entries of $\mathbf W$ that link data points belonging to different subspaces. However the raw values require some processing since there will be a large number of non-zero singular values. One can threshold these values as suggested in \cite{liberty2007randomized}. In other words any singular value less than a given value should be ignored. We can express this as
\[
k = \sum_{i=1}^N ( 1 | \sigma_i > \tau )
\]
where $\tau$ is the threshold value.

Singular value thresholding can produce acceptable results but it requires user selection of the threshold value. A more automatic approach is to use either the Eigen-gap \cite{vidal2011subspace} or the closely related SVD-gap \cite{6180173} heuristic. The Eigen-gap heuristic uses the eigenvalues of $\mathbf W$ or $\mathbf L$ to find the number subspaces by finding the largest gap between the ordered eigenvalues. Let $\{\delta_i\}_{i=1}^N$ be the descending sorted eigenvalues of $\mathbf W$ such that $\delta_1 \geq \delta_2 \geq \dots \geq \delta_N$. Then $k$ can be estimated by
\[
k = \argmax_{i = 1, \dots, N-1} (\delta_i - \delta_{i+1})
\]
The SVD-gap heuristic is the same procedure with eigenvalues of $\mathbf W$ replaced with singular values \cite{6180173}.

\section{Experimental Evaluation}

%

In this section the performance of OSC is compared against SSC, LRR and SpatSC methods with a variety of data sources.
Parameters were fixed for each experiment. In order to evaluate performance consistently NCUT was used for final segmentation for every method in every experiment. MATLAB implementations were used from the open source SubKit\footnote{\url{https://github.com/sjtrny/SubKit}} package. Implementations for relaxed and exact variants of OSC have been included in this library. The relaxed v-LADMAP variant of OSC was used in all experiments.

We use the subspace clustering error metric from \cite{elhamifar2012sparse} to evaluate clustering accuracy. The subspace clustering error (SCE) is as follows
\begin{align}
\text{SCE} = \frac{\text{num. misclassified points}}{\text{total num. of points}}
\end{align}
Furthermore additional noise is injected into the data to test robustness. We report the level of noise using Peak Signal-to-Noise Ratio (PSNR) which is defined as
\begin{align}
\text{PSNR} = 10 \log_{10} \left( \frac{s^2}{\frac{1}{m n} \sum_i^{m} \sum_j^{n} ( A_{ij} - X_{ij})^2} \right)
\label{PSNR}
\end{align}
where $\mathbf{X = A + N}$ is the noisy data and $s$ is the maximum possible value of an element of $\mathbf A$. Decreasing values of PSNR indicate increasing amounts of noise.

In contrast to other works \cite{elhamifar2012sparse, liu2010robust} the minimum, maximum, median and mean statistics on clustering error are provided for the comparative experiments. It is important to consider these ranges holistically when evaluating these methods. In all experiments Gaussian noise was used with zero mean and unit variance. When parameters are fixed we report them in Table \ref{Table:Parameters}.

\begin{table}
\centering

\begin{tabular}{c c | c c c c}
Experiment & & $\lambda_1$ & $\lambda_2$ & $\mu$ & $\textrm{diag}(\mathbf Z) = \mathbf 0$ \\
\hline
\multirow{4}{*}{Synthetic and Semi-Synthetic}
		& OSC		& 0.1	& 1		& 1		& 0 \\
		& SpatSC	& 0.1	& 0.01	& 1		& 1 \\
		& LRR		& 0.4 	&		&		& 0 \\
		& SSC		& 0.2 	&		&		& 1 \\
\hline
\multirow{4}{*}{Video and Activity Segmentation}
		& OSC		& 0.1	& 1		& 1		& 0 \\
		& SpatSC	& 0.1	& 0.01	& 0.1	& 1 \\
		& LRR		& 0.4	&		&		& 0 \\
		& SSC		& 0.1	&		&		& 1 \\

\end{tabular}

\caption{Overview of parameters used for each experiment.}
\label{Table:Parameters}
\end{table}

\section{Synthetic Subspace Segmentation}

In this section evaluation is performed using randomly generated subspace structured data. Similar to \cite{liu2010robust} $5$ subspaces $\{S_i\}^5_{i=1}$ are constructed whose bases $\{\mathbf U_i\}^5_{i=1}$ are computed by $\mathbf U_{i+1} = \mathbf T \mathbf U_i, 1 \leq i \leq 4$, where $\mathbf T$ is a random rotation matrix and $\mathbf U_1$ is a random orthonormal basis of dimension $100 \times 4$. In other words each basis is a random rotation away from the previous basis and the dimension of each subspace is $4$. $20$ data points are sampled from each subspace by $\mathbf X_i = \mathbf U_i \mathbf Q_i$ where $\mathbf Q_i \in \mathbb{R}^{4 \times 20}$ is a random gaussian multi-variate matrix with row-wise variance of $0.001$ and $0.0005$ between neighbouring columns i.e. the following covariance matrix:
\begin{align*}
\mathbf C \in \mathbb{R}^{20 \times 20} = \left[ \begin{matrix}
0.001 & 0.0005\\
0.0005 & 0.001 & 0.0005\\
& 0.0005 & 0.001 & 0.0005\\
& & \ddots & \ddots\\
& &  & & 0.0005\\
& & & 0.0005 & 0.001
\end{matrix} \right].
\end{align*}
This mimics the assumption that consecutive data points within the same subspace are similar to each other. Finally the data is concatenated $\mathbb X = [\mathbf X_1, \mathbf X_2,\dots, \mathbf X_{5}]$.

We repeated the experiment $50$ times with new random bases and coefficient matrices each time. Furthermore we repeated the experiment with various levels of noise to determine robustness. Results are reported in Figure \ref{Plot:SyntheticStats}. OSC (ours) demonstrated significantly better clustering accuracy than SpatSC, LRR and SSC in all metrics. Even in cases of extremely noisy data (low PSNR) OSC still demonstrates excellent accuracy.

\begin{figure*}[!t]
\centering
\subfloat[Mean SCE]{\includegraphics[width=0.25\textwidth]{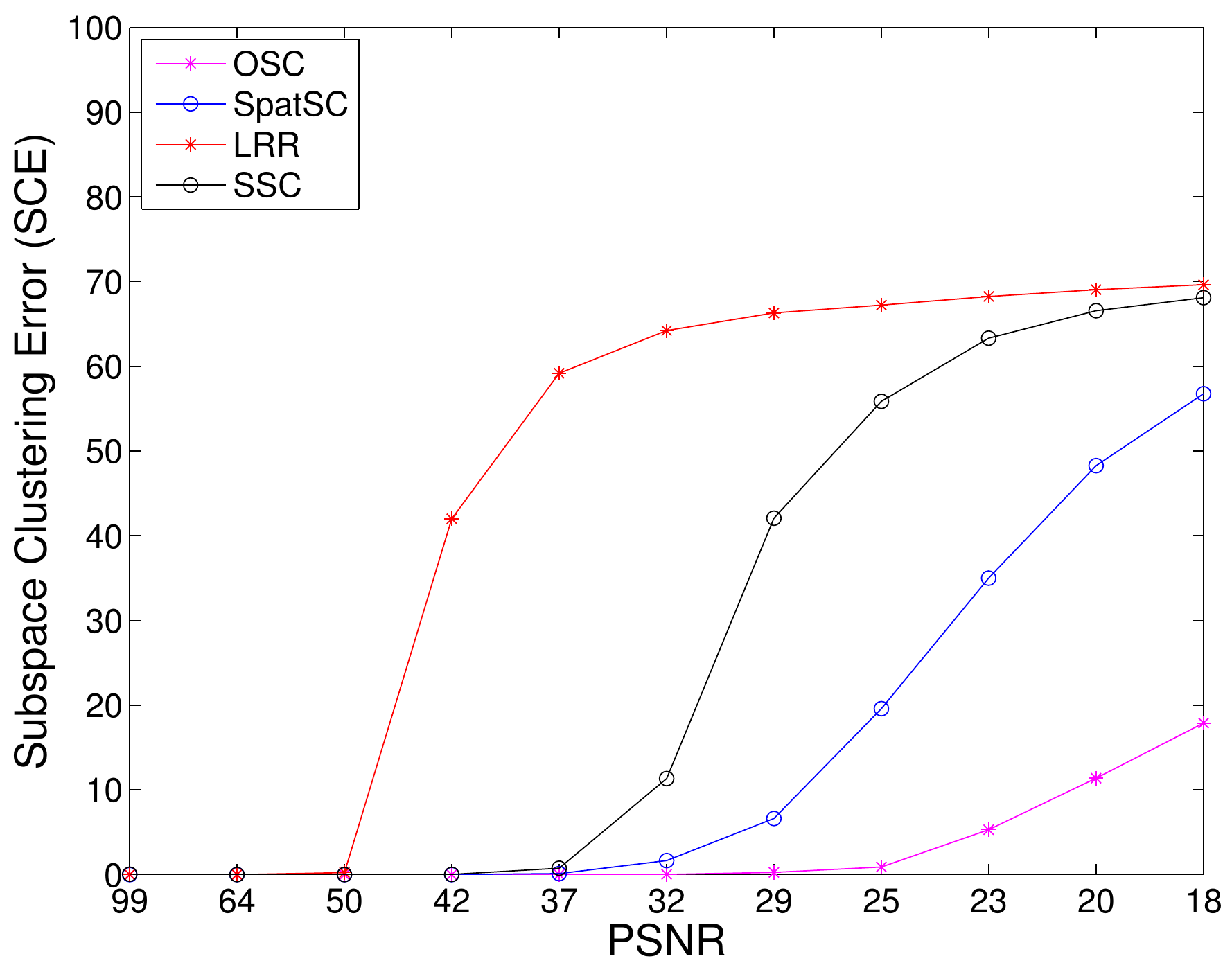}}
\subfloat[Median SCE]{\includegraphics[width=0.25\textwidth]{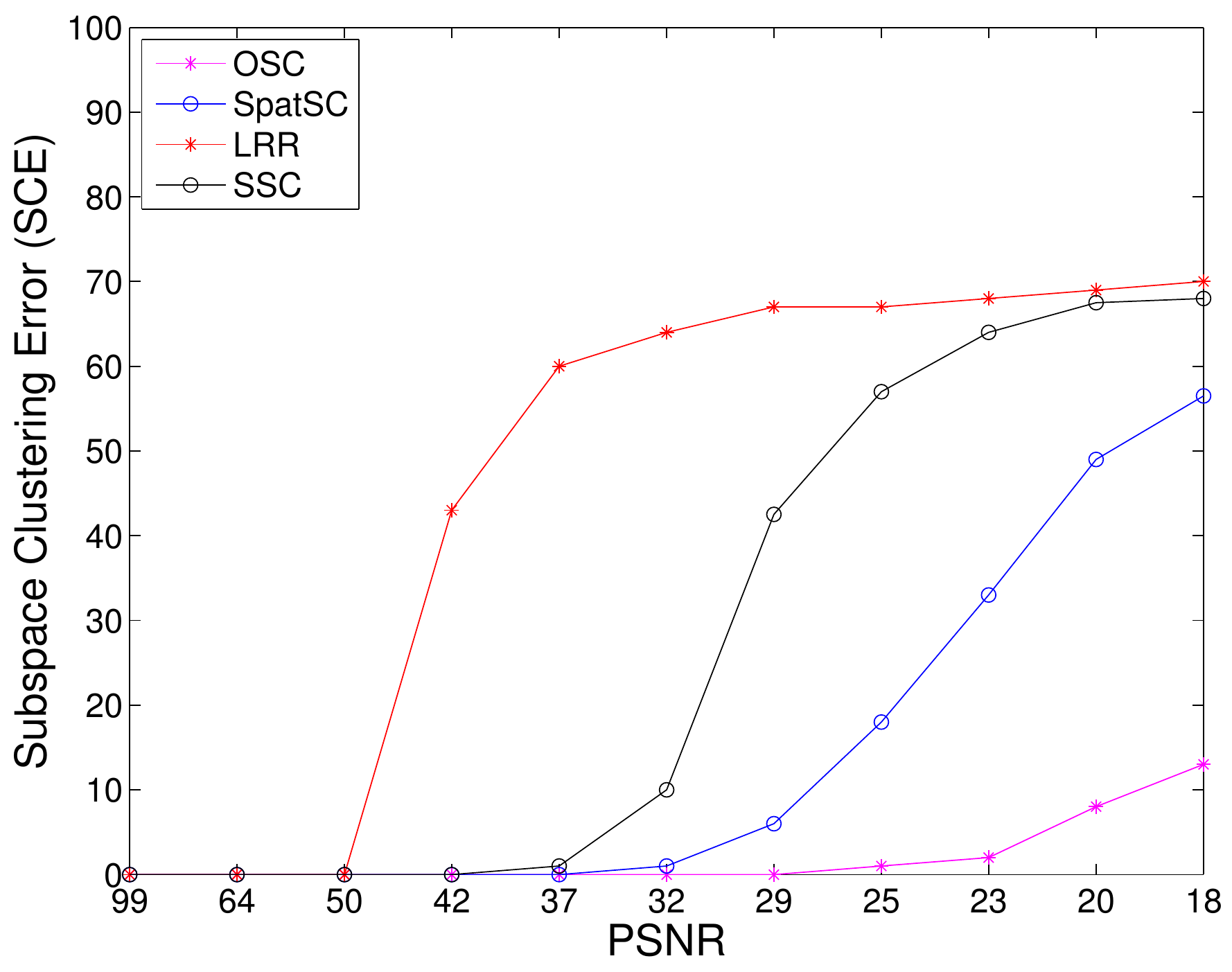}}
\subfloat[Minimum SCE]{\includegraphics[width=0.25\textwidth]{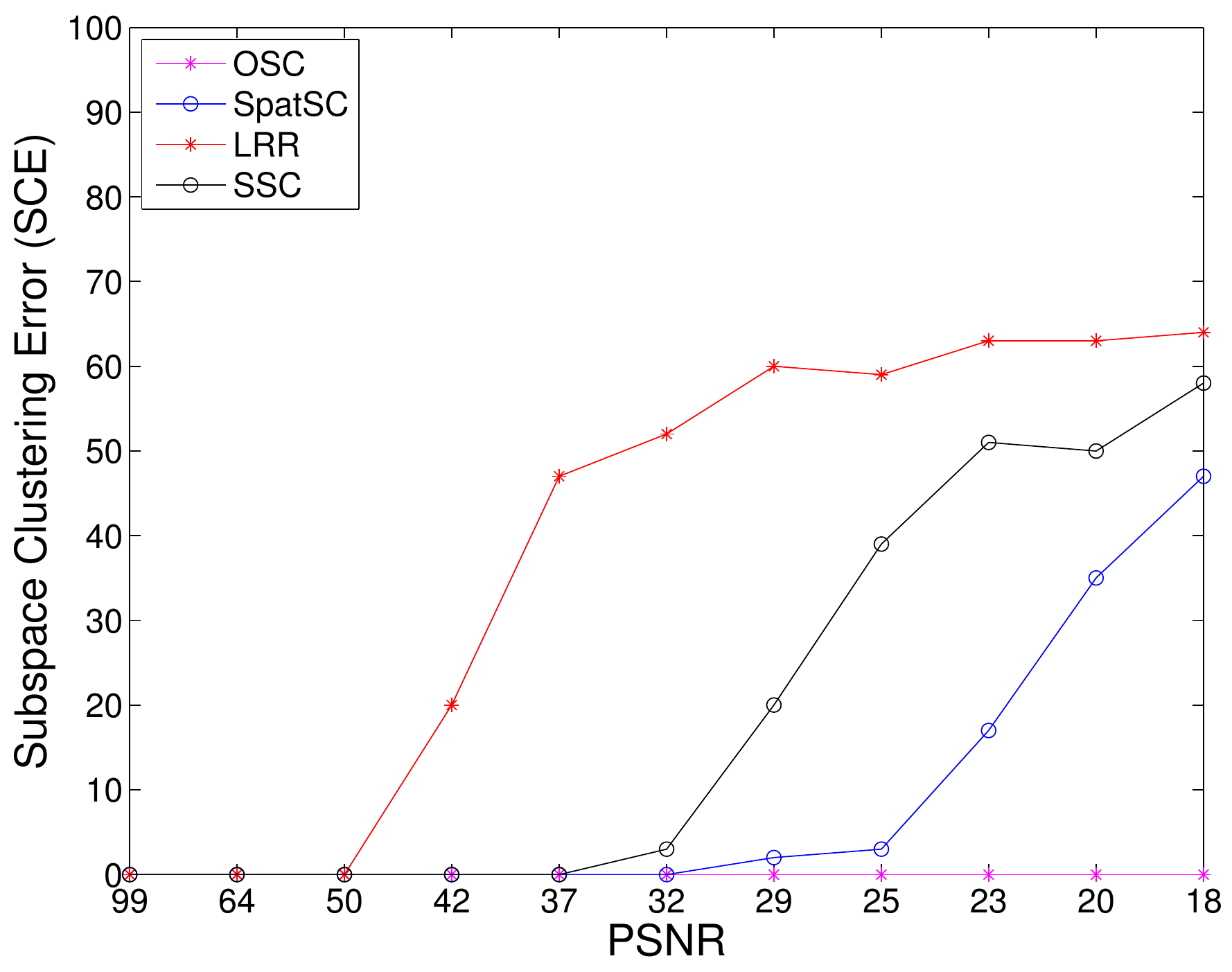}}
\subfloat[Maximum SCE]{\includegraphics[width=0.25\textwidth]{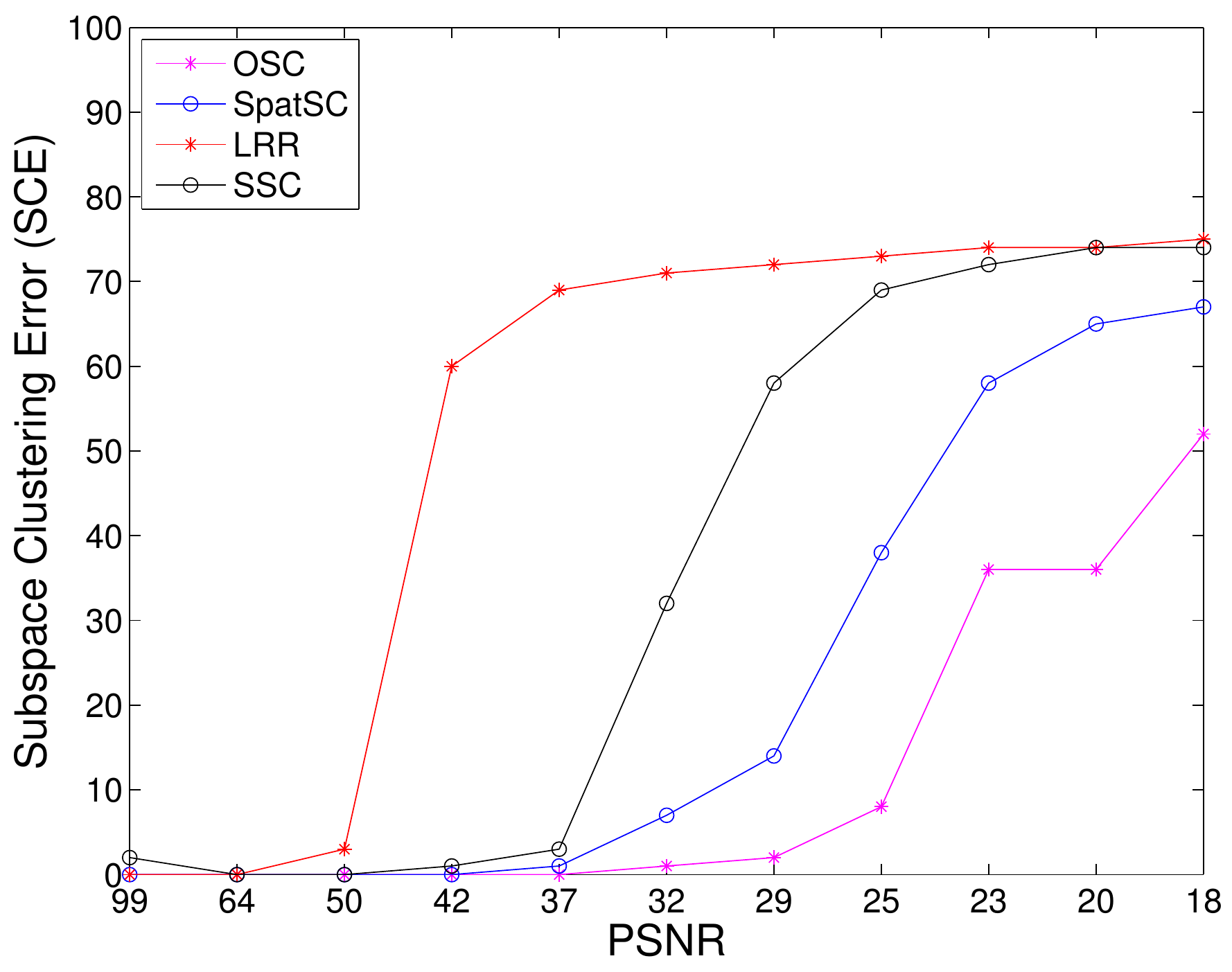}}
\caption{Results for the synthetic data segmentation experiment with various magnitudes of Gaussian noise. OSC (ours) outperforms SpatSC, LRR and SSC in the majority of cases.}
\label{Plot:SyntheticStats}
\end{figure*} 

\section{Running Time}

\label{section:run_time}

\begin{figure*}[!t]
\centering
\includegraphics[width=0.7\textwidth]{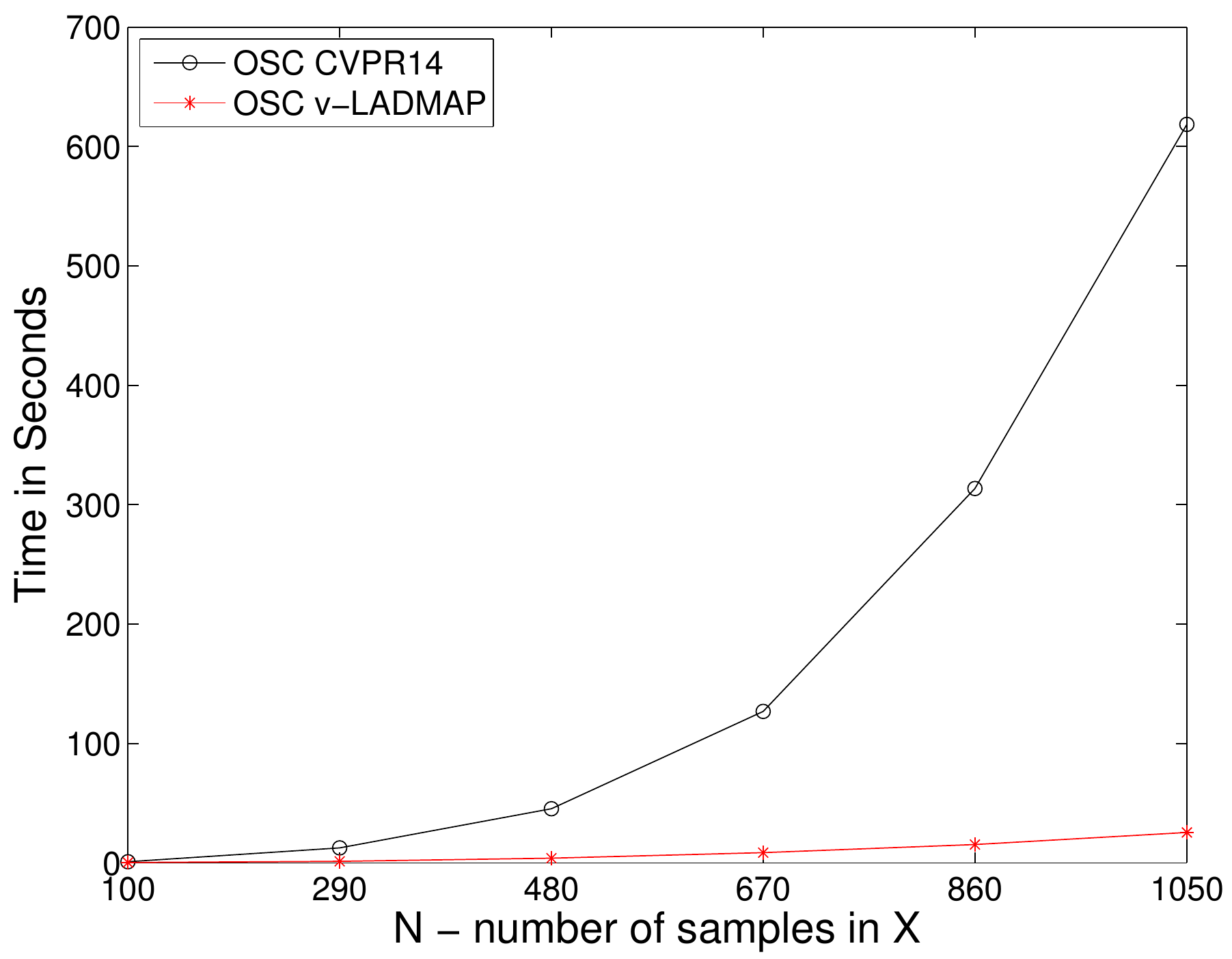}
\caption{Average running time of OSC implementations for increasing amounts of data. The v-LADMAP procedure is a significant improvement over the procedure (CVPR14) suggested in the preliminary version of this paper.}
\label{Plot:running_time}
\end{figure*} 

To demonstrate the improvements in running time we compare the the optimisation scheme suggested in the preliminary version \cite{tierney2014subspace} of this paper and the v-LADMAP procedure suggested in this paper. Synthetic data was generated as in the previous section. However this time the number of samples in each of the $5$ clusters is progressively increased. At each increase in the number of samples we repeated both OSC procedures $10$ times to obtain an average running time. The experiment was performed on a machine with a $3.4$ Ghz i7 CPU and $32$GB RAM. Results are reported in Figure \ref{Plot:running_time}. From the Figure we can clearly see that v-LADMAP provides a significant improvement over the CVPR14 procedure. The v-LADMAP procedure completes the experiments in a matter of seconds where the CVPR14 procedure takes over $10$ minutes. The CVPR14 procedure results exhibits a quadratic run time in the number of samples while v-LADMAP is linear. 

\section{Semi-Synthetic Experiment}

Semi-Synthetic data is assembled from a library of pure infrared hyper spectral mineral data as in \cite{Guo.Y;Gao.J;Li.F-2013}. Similar to the synthetic experiment $5$ subspaces are created with $20$ data samples in each. For each subspace $5$ spectra samples are randomly chosen as the bases such that $\mathbf U_i \in \mathbb{R}^{321 \times 5}$. The $20$ data samples are then sampled from each subspace by $\mathbf X_i = \mathbf U_i \mathbf Q_i$ where $\mathbf Q_i \in \mathbb{R}^{5 \times 20}$ is a random gaussian multi-variate matrix as defined in the previous section. The data is concatenated $\mathbb X = [\mathbf X_1, \mathbf X_2,\dots, \mathbf X_{5}]$.

Similar to the previous experiment we repeated the experiment $50$ times with new random bases and coefficient matrices each time and we repeated the experiment with various levels of noise to determine robustness. Results are reported in Figure \ref{Plot:TirStats}. Again the experiment reveals that OSC outperforms all other methods in a majority of cases.

\begin{figure*}[!t]
\centering
\subfloat[Mean SCE]{\includegraphics[width=0.25\textwidth]{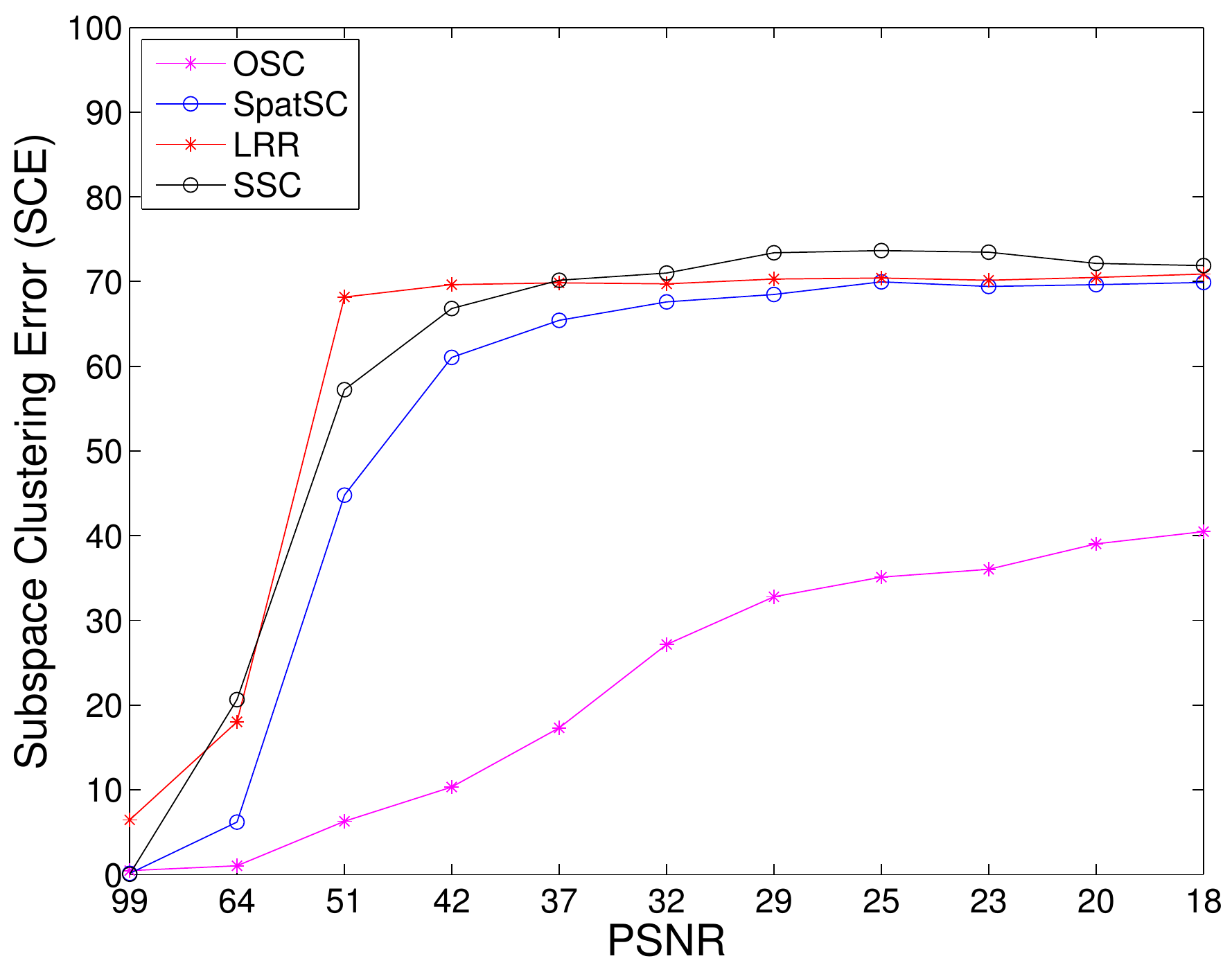}}
\subfloat[Median SCE]{\includegraphics[width=0.25\textwidth]{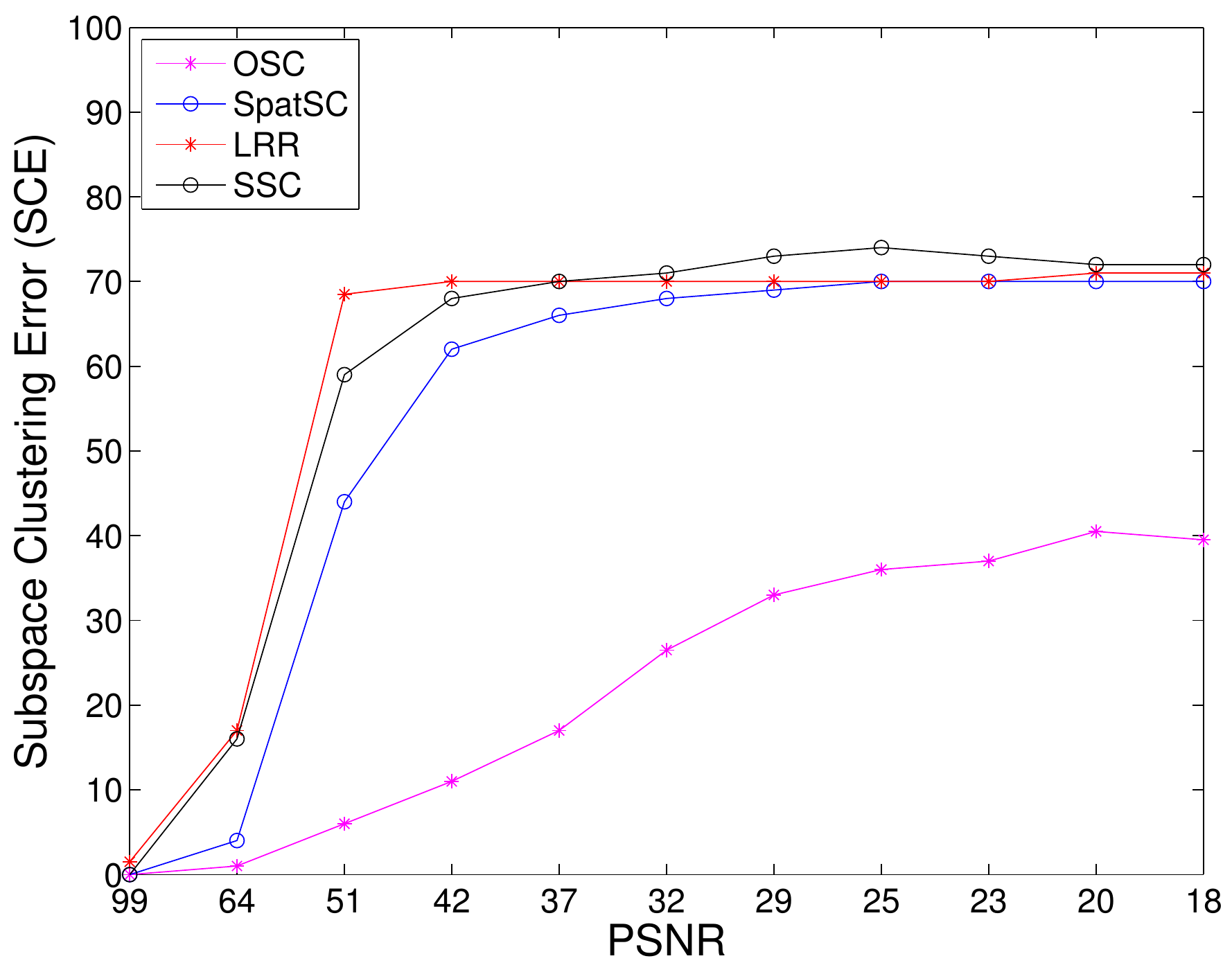}}
\subfloat[Minimum SCE]{\includegraphics[width=0.25\textwidth]{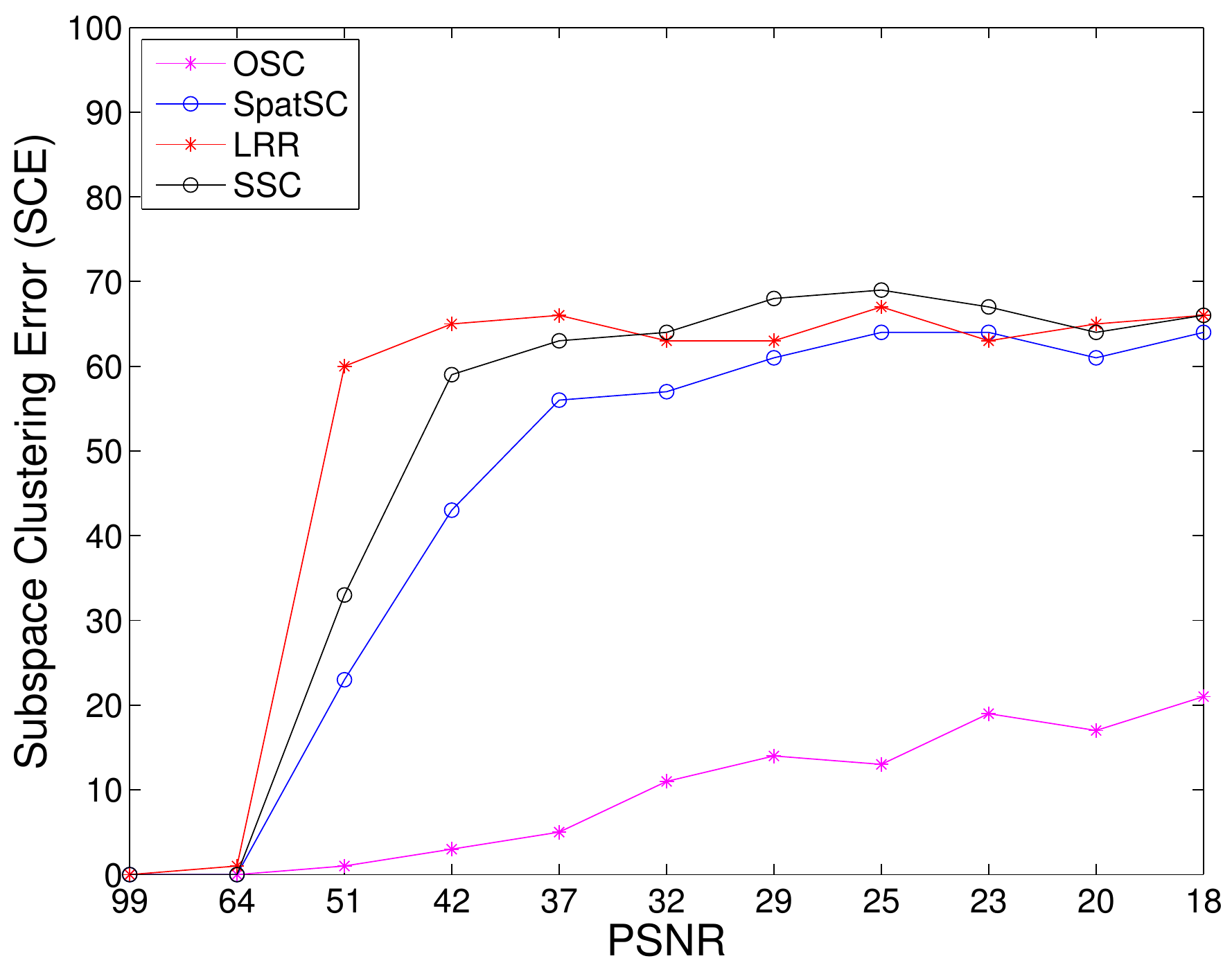}}
\subfloat[Maximum SCE]{\includegraphics[width=0.25\textwidth]{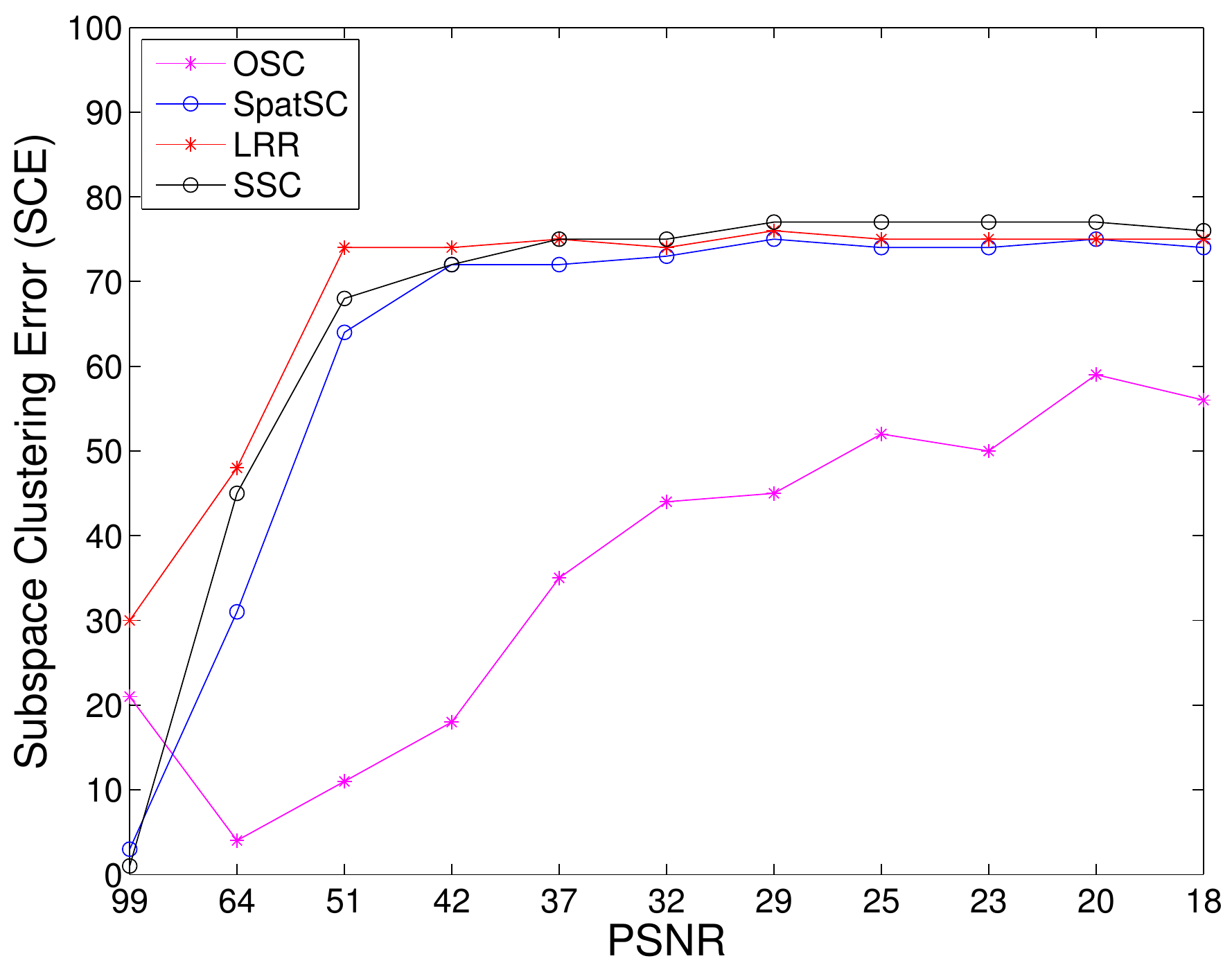}}
\caption{Results for the semi-synthetic data segmentation experiment with various magnitudes of Gaussian noise. OSC (ours) outperforms SpatSC, LRR and SSC in the majority of cases.}
\label{Plot:TirStats}
\end{figure*} 


\section{Video Scene Segmentation}

The aim of this experiment is to segment individual scenes, which correspond to subspaces, from a video sequence. The video sequences are drawn from two short animations freely available from the Internet Archive\footnote{\url{http://archive.org/}}. See Figure \ref{Fig:vid_example} for an example of a sequence to be segmented. The sequences are around 10 seconds in length (approximately 300 frames) containing three scenes each. There are 19 and 24 sequences from videos 1 and 2 respectively. The scenes to be segmented can contain significant translation and morphing of objects within the scene and sometimes camera or perspective changes, which presents considerable difficulty. Scene changes were collected manually to form ground truth data.

The pre-processing of each sequence consisted of converting colour video to grayscale and down sampling to a resolution of $129 \times 96$. Each frame in the sequence was vectorised to $\mathbf x_i \in \mathbb{R}^{12384} $ and concatenated with consecutive frames to form $\mathbf X \in \mathbb{R}^{12384 \times 300}$. Each sequence was further corrupted with various magnitudes of gaussian noise, with the experiment being repeated $50$ times at every magnitude.

Results can be found in Figures \ref{Plot:Vid_1_Stats} and \ref{Plot:Vid_2_Stats}. Generally OSC outperforms other methods and the error rates are consistently low when compared to other methods which greatly increase as the magnitude of the noise is increased.

\begin{figure*}[!t]
\centering
\subfloat[Mean SCE]{\includegraphics[width=0.25\textwidth]{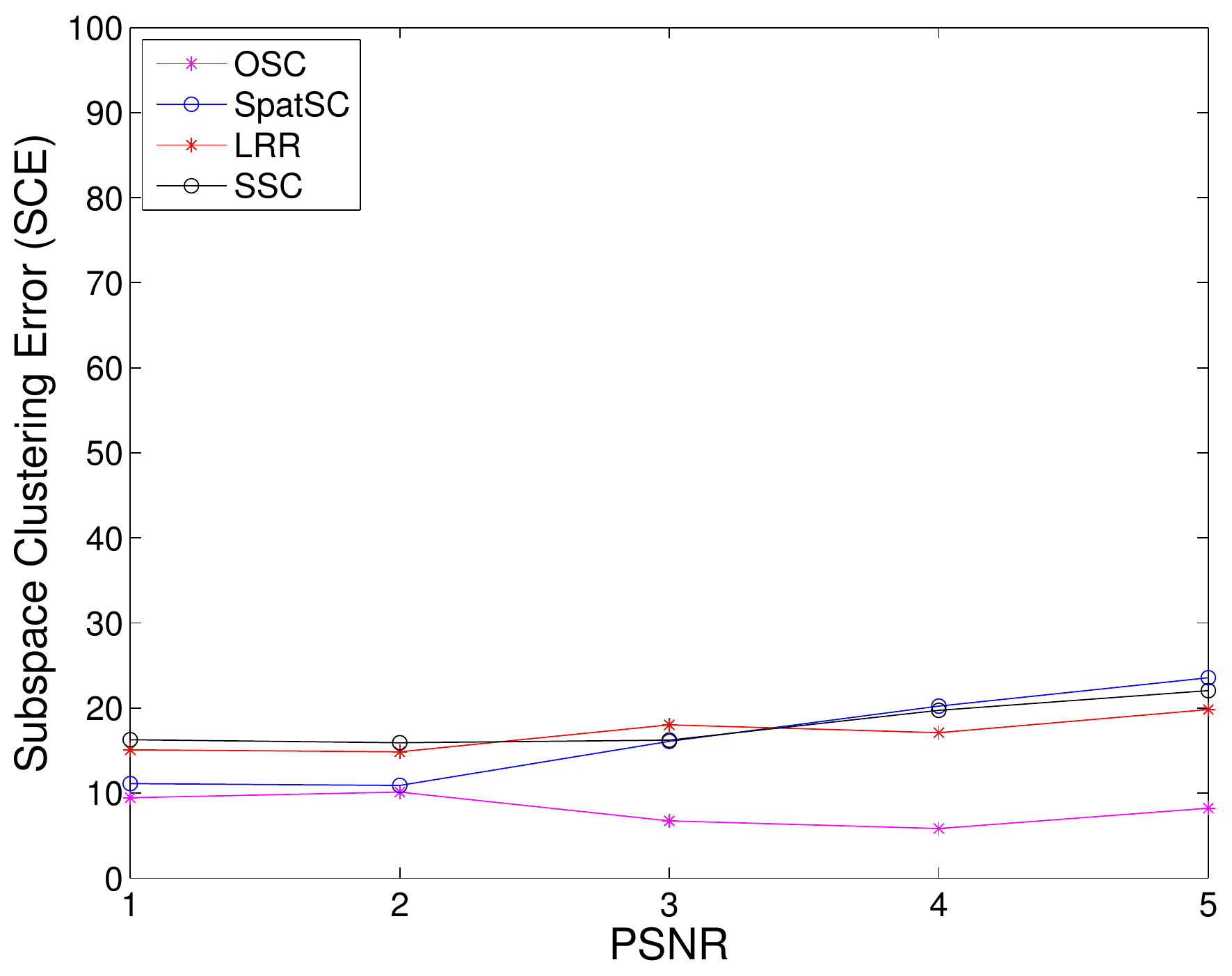}}
\subfloat[Median SCE]{\includegraphics[width=0.25\textwidth]{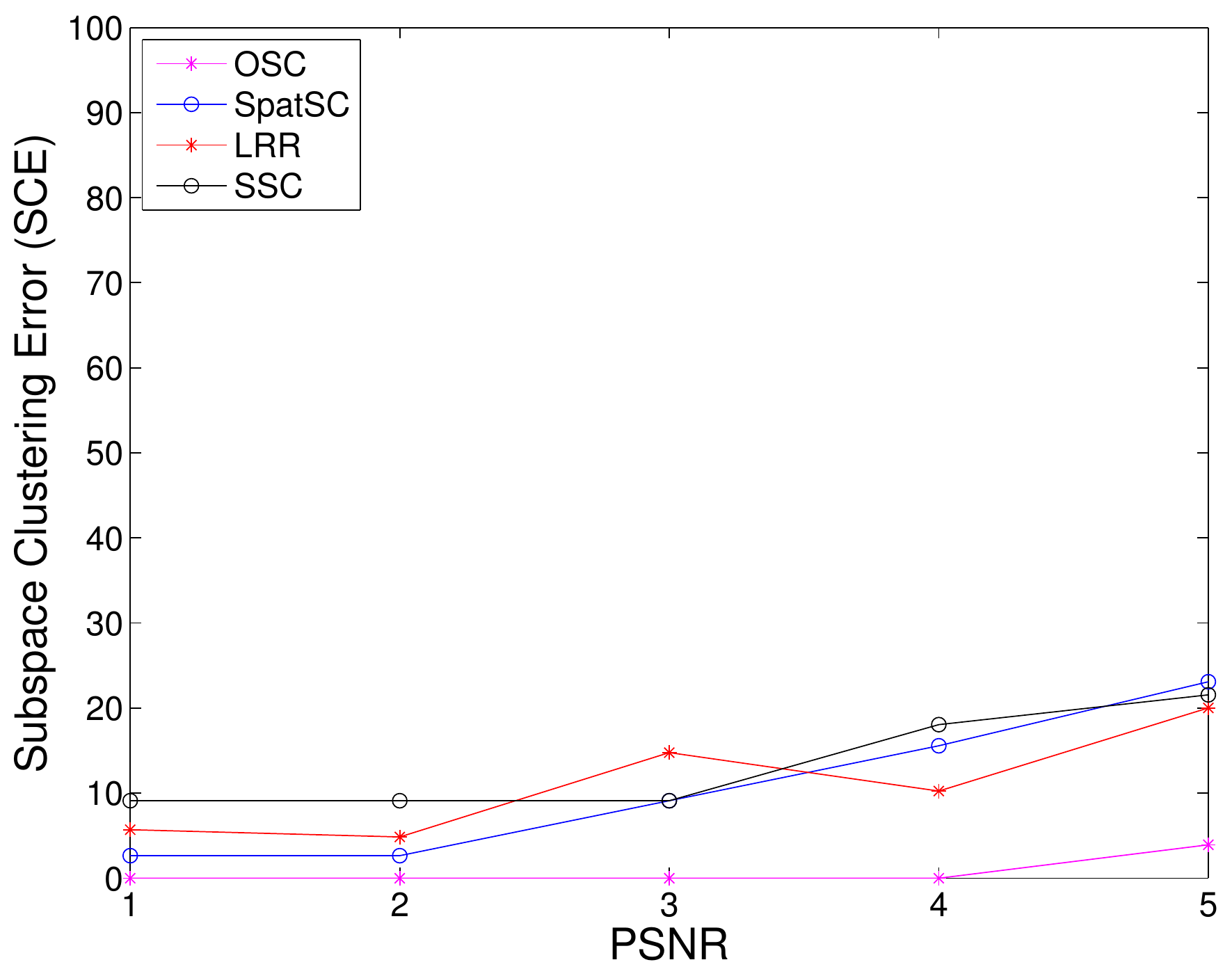}}
\subfloat[Minimum SCE]{\includegraphics[width=0.25\textwidth]{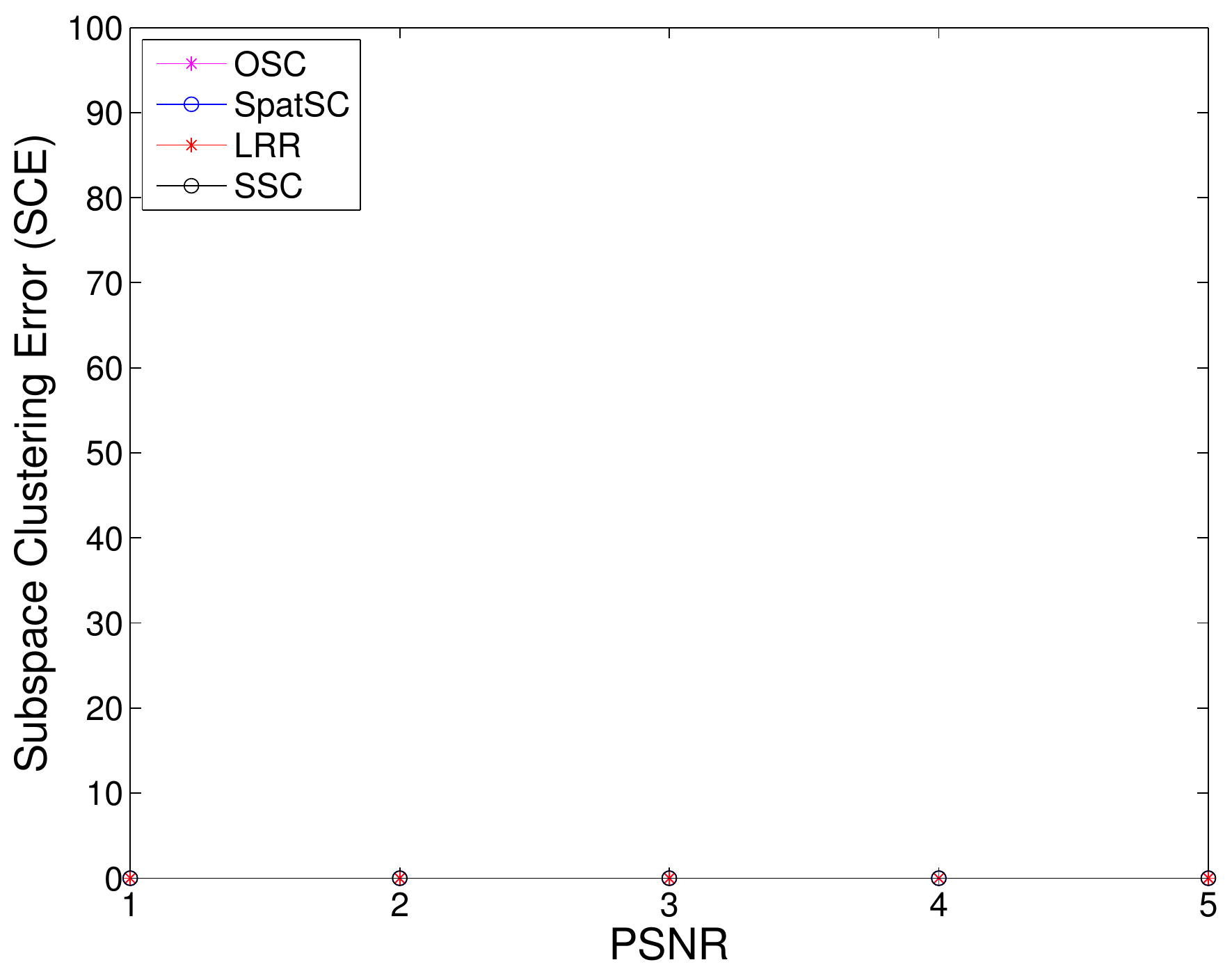}}
\subfloat[Maximum SCE]{\includegraphics[width=0.25\textwidth]{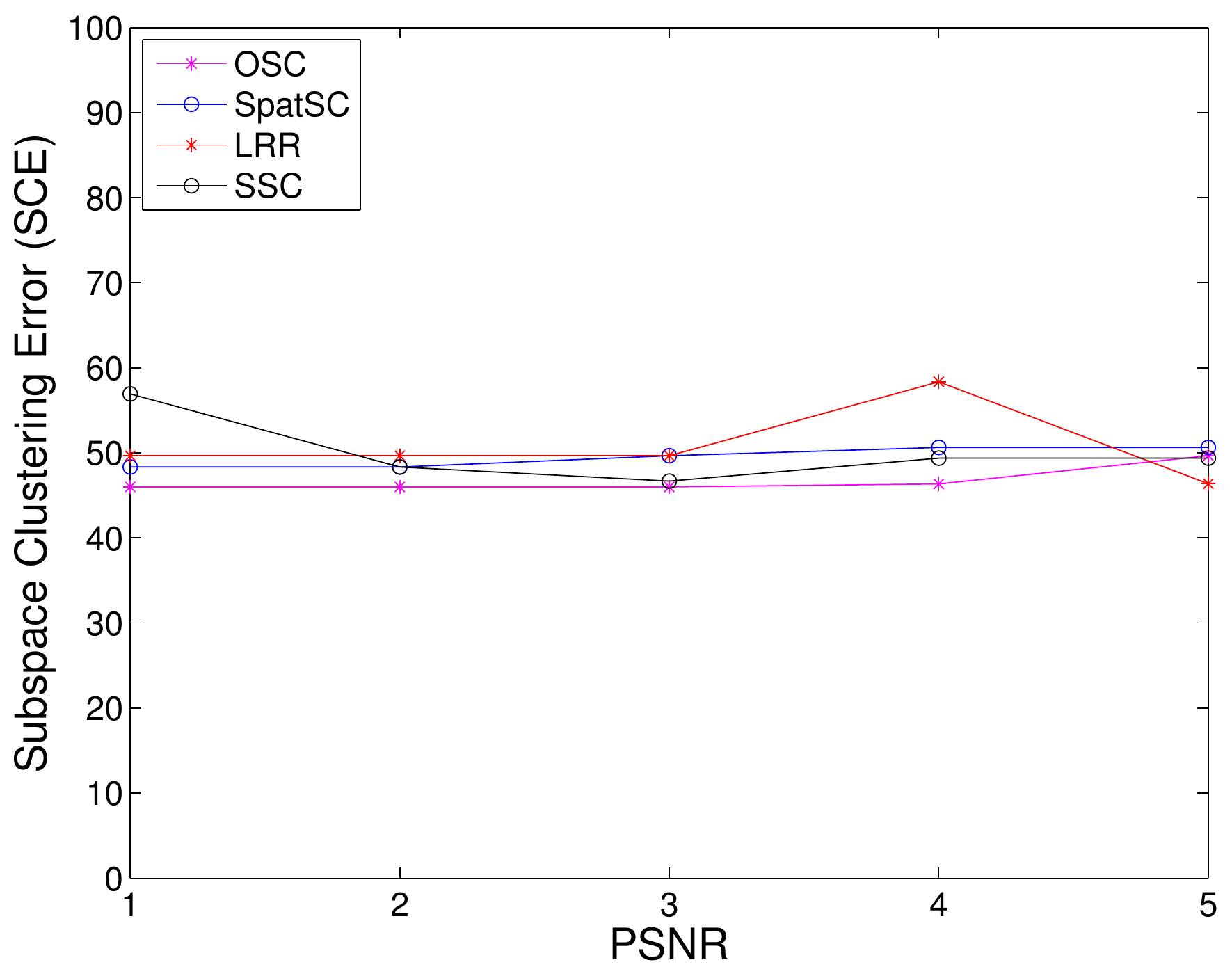}}
\caption{Results for the video scene segmentation experiment (Video 1) with various magnitudes of Gaussian noise. OSC (ours) outperforms SpatSC, LRR and SSC in the majority of cases.}
\label{Plot:Vid_1_Stats}
\end{figure*} 

\begin{figure*}[!t]
\centering
\subfloat[Mean SCE]{\includegraphics[width=0.25\textwidth]{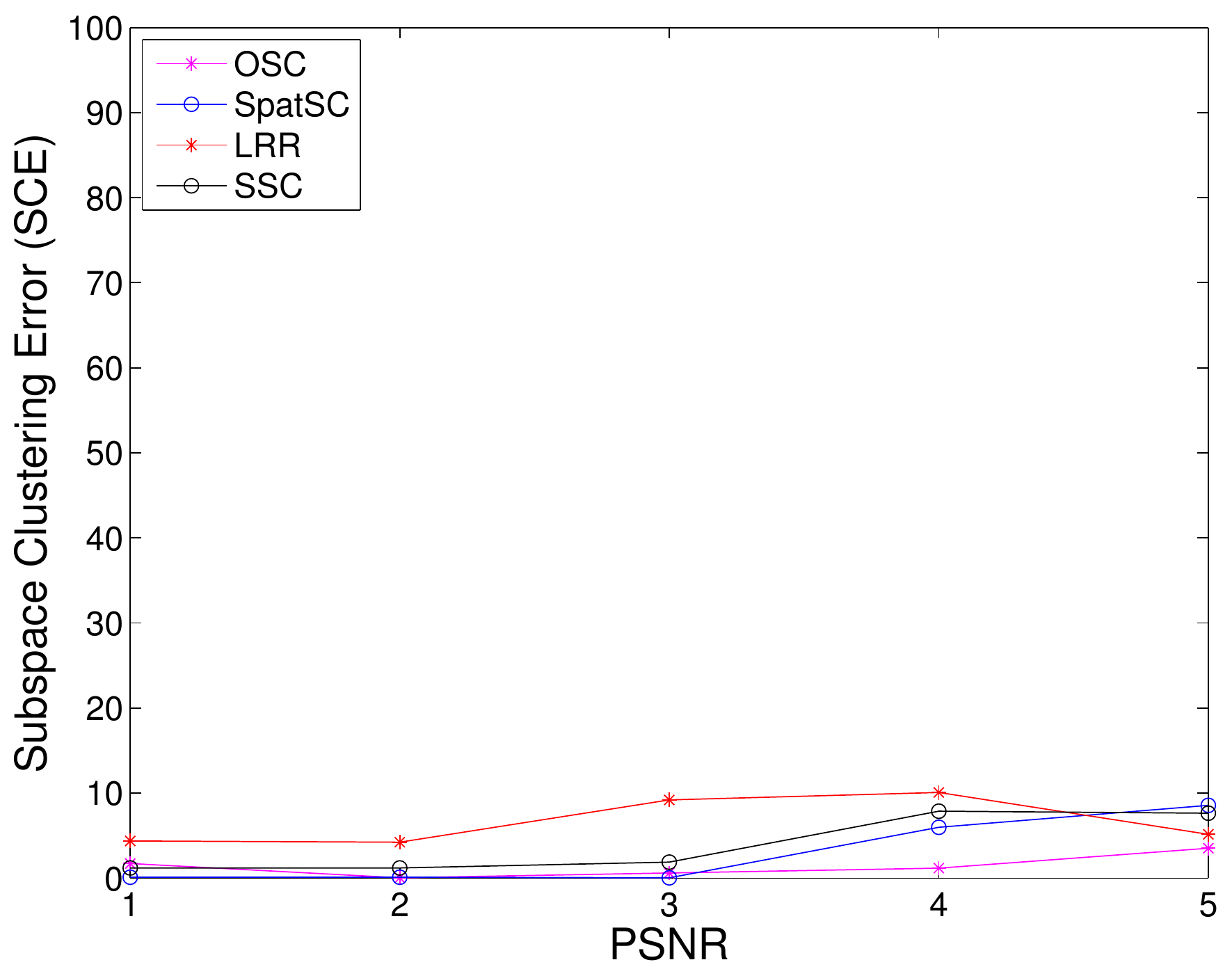}}
\subfloat[Median SCE]{\includegraphics[width=0.25\textwidth]{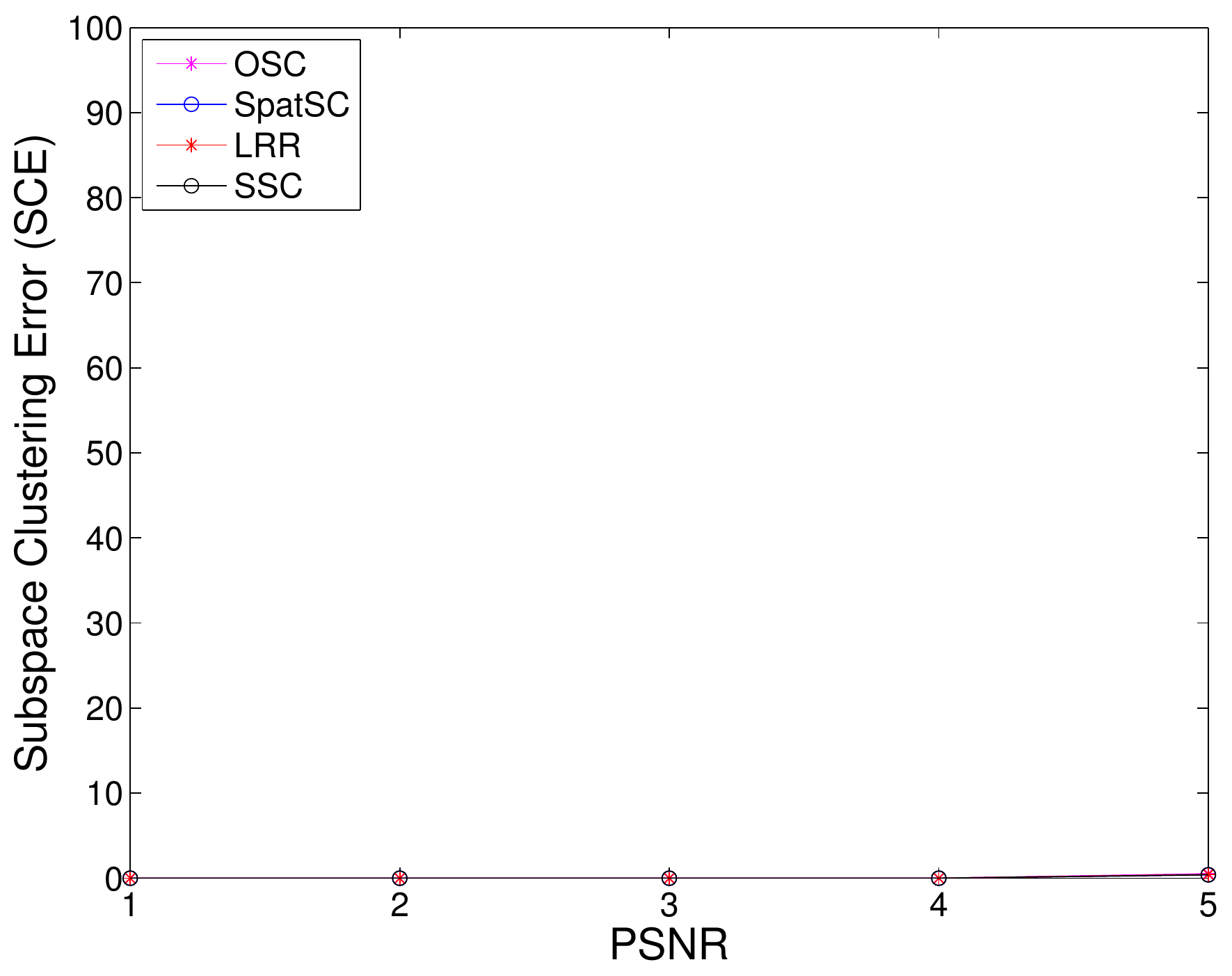}}
\subfloat[Minimum SCE]{\includegraphics[width=0.25\textwidth]{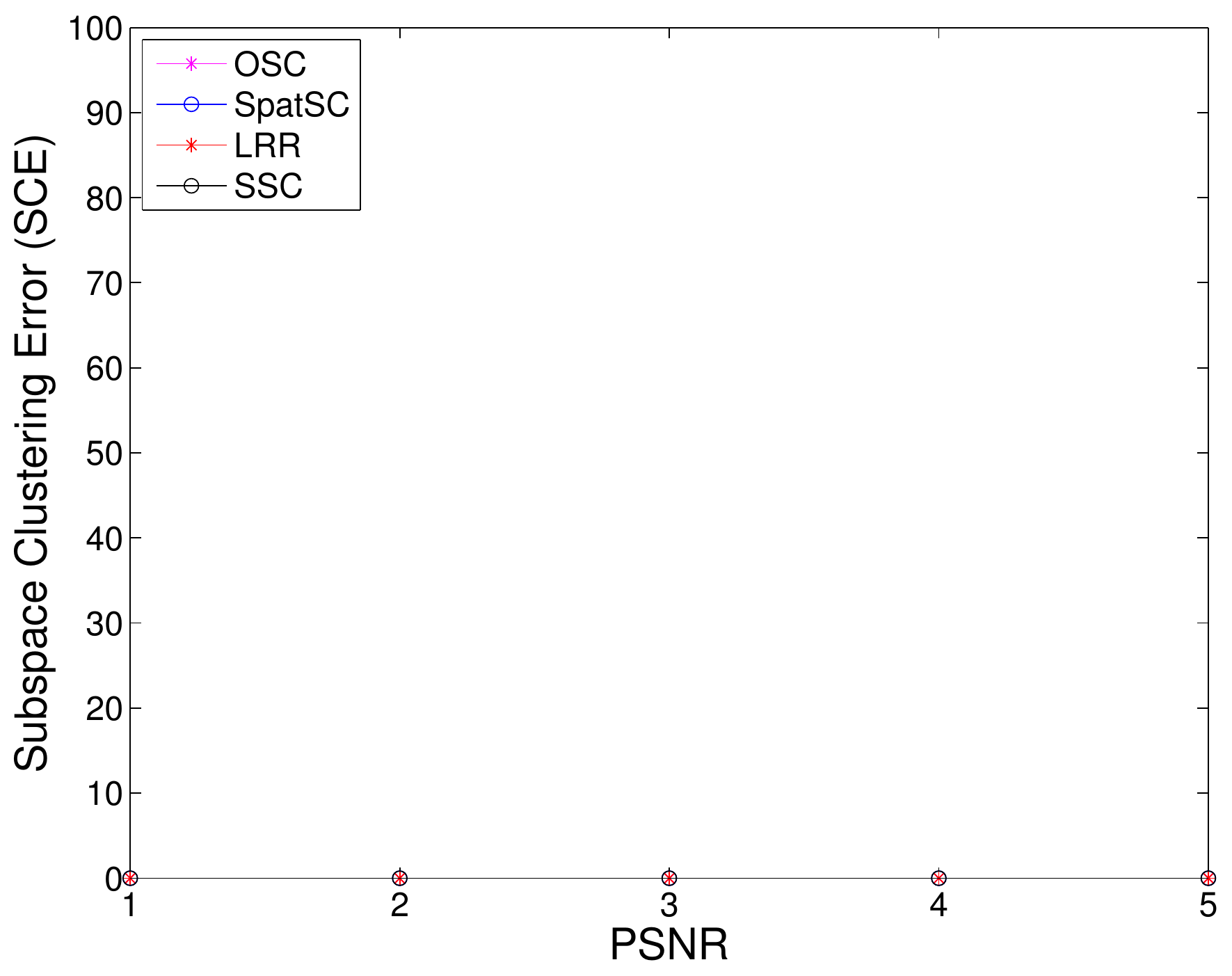}}
\subfloat[Maximum SCE]{\includegraphics[width=0.25\textwidth]{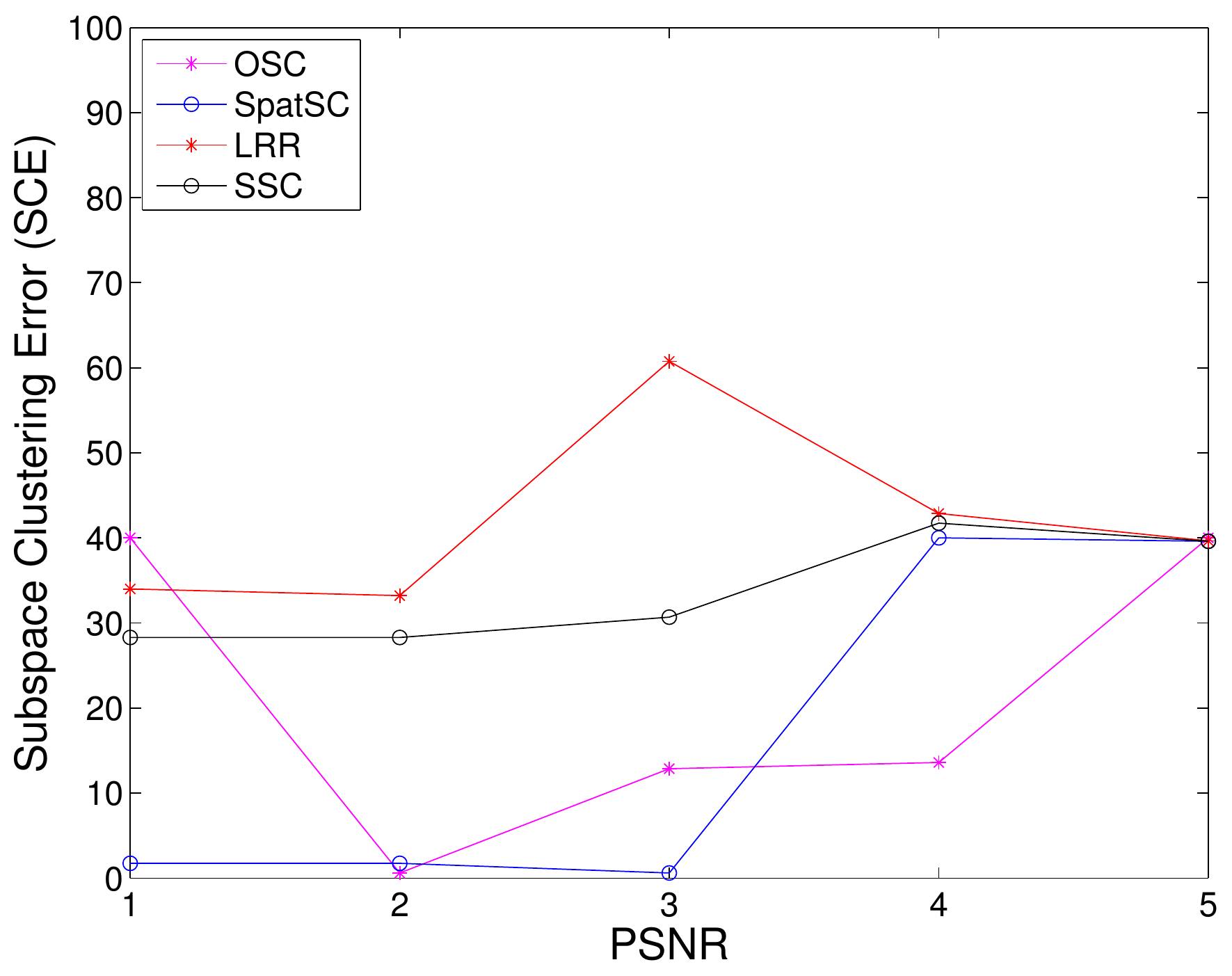}}
\caption{Results for the video scene segmentation experiment (Video 2) with various magnitudes of Gaussian noise. OSC (ours) outperforms SpatSC, LRR and SSC in the majority of cases.}
\label{Plot:Vid_2_Stats}
\end{figure*} 

\section{Human Activity Segmentation} 

The aim of this experiment is to segment activities in a sequence from the HDM05 Motion Capture Database \cite{cg-2007-2}. This dataset consists of a sequence of around $60$ joint/joint angle positions for each frame, which has been captured at $120$ Hz. These positions were determined by optically tracking a number of reflective markers on an actor. From the 2D frames containing the marker positions software is used to locate these points in 3D space. Then these points are transferred into joint and joint angle positions since this format requires less storage space. For an example of the capture environment and captured marker positions and skeletal structure please see Figure \ref{Fig:mocap_example}.

Unfortunately there is no provided frame by frame ground truth for this dataset. Therefore our ground truth has been assembled by watching the replay of the activities and hand labelling the activities using the activity list provided by \cite{cg-2007-2}. For this experiment we chose scene 1-1, which consists of mostly walking activities in various poses but also includes other actions such as double stepping or shuffling sideways and multiple turns. This scene contains $9842$ frames therefore to ease computational burden we divided the scene into four sections of between $2000-3000$ frames each.

We report subspace clustering error for this experiment in Figure \ref{Plot:MocapStats}. Similar to the previous experiments we add increasing amounts of Gaussian noise to determine the robustness of each method and repeat the experiment $50$ times for each magnitude of noise.

\begin{figure*}[!t]
\centering
\subfloat[Mean SCE]{\includegraphics[width=0.25\textwidth]{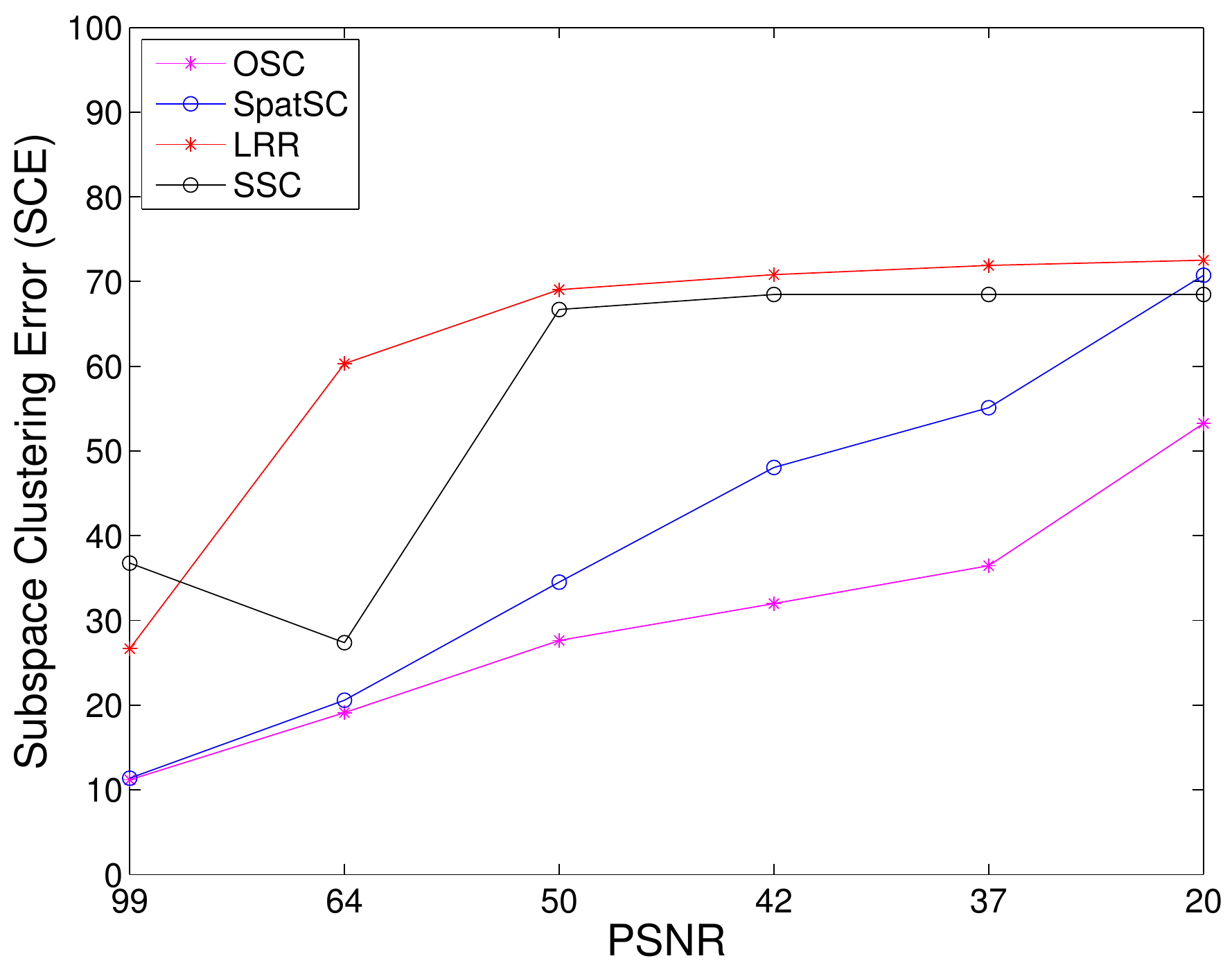}}
\subfloat[Median SCE]{\includegraphics[width=0.25\textwidth]{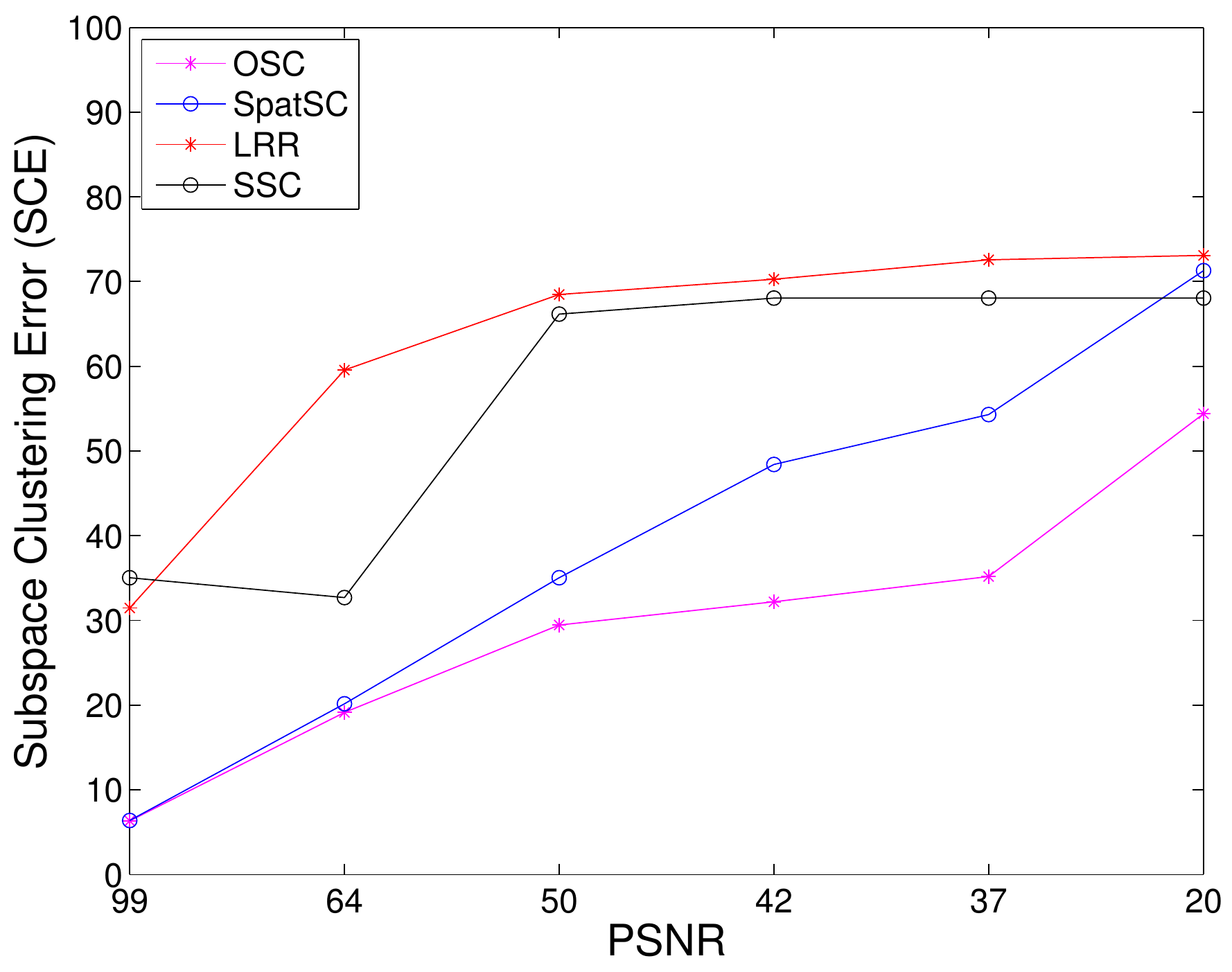}}
\subfloat[Minimum SCE]{\includegraphics[width=0.25\textwidth]{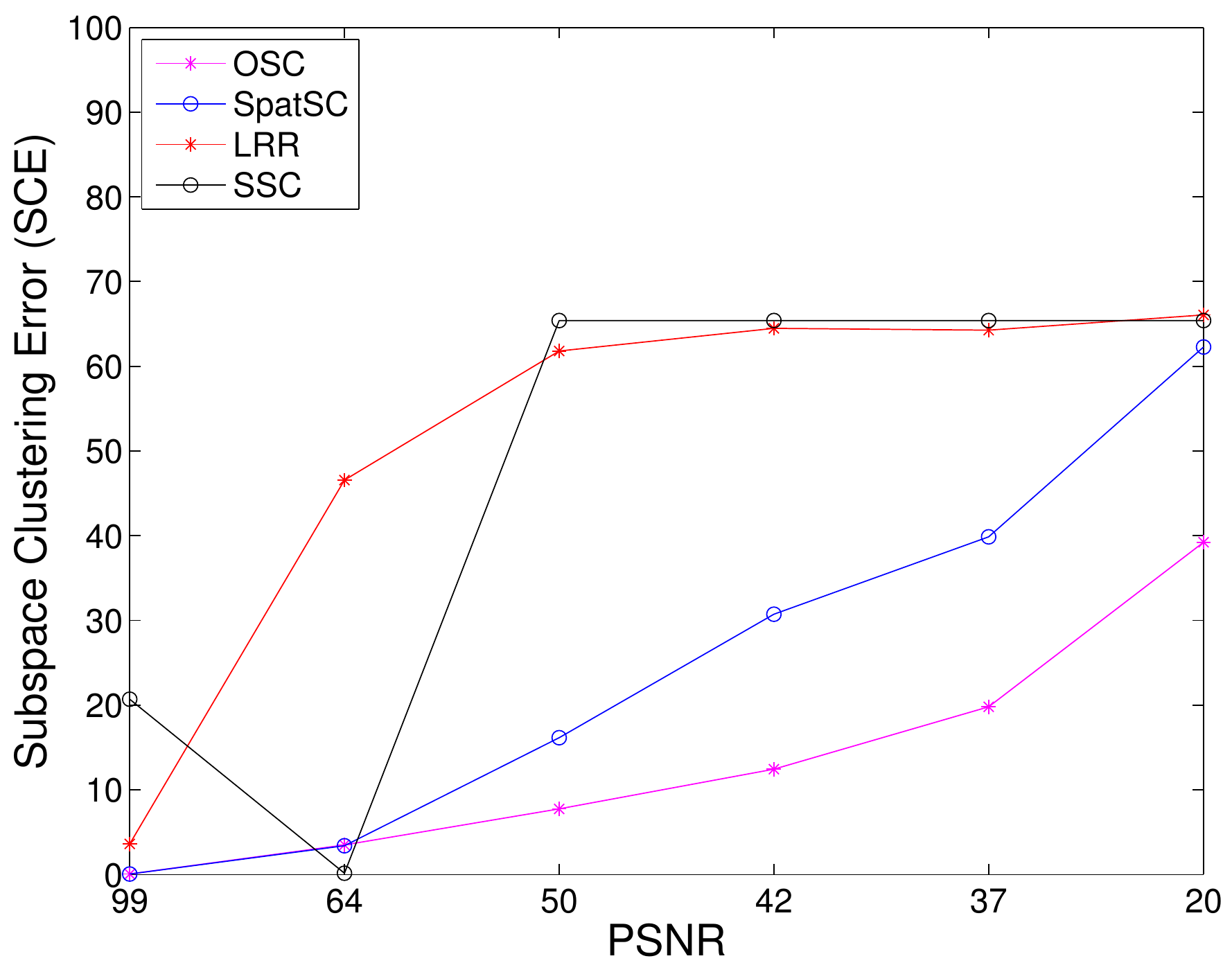}}
\subfloat[Maximum SCE]{\includegraphics[width=0.25\textwidth]{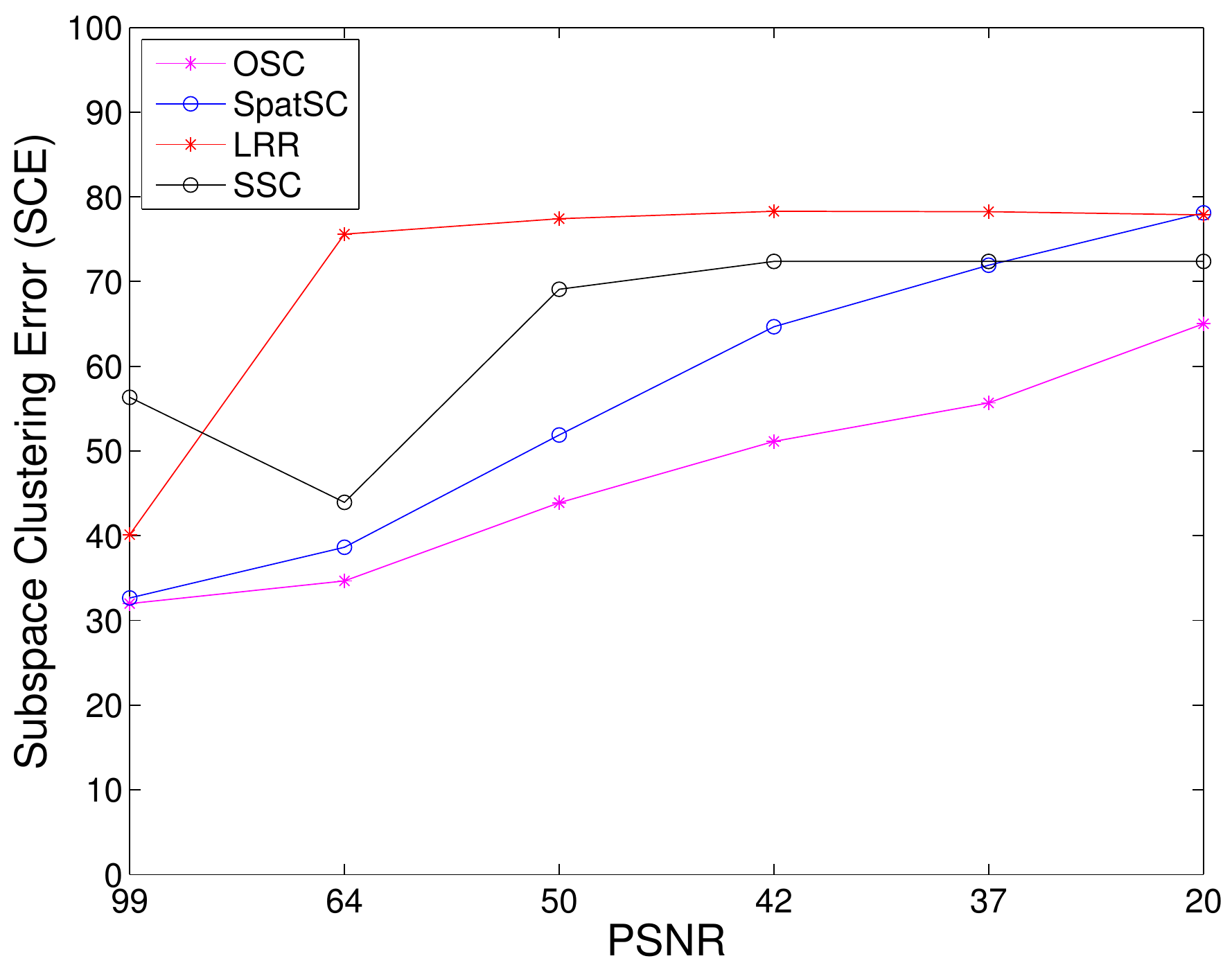}}
\caption{Results for the human activity segmentation experiment with various magnitudes of Gaussian noise. OSC (ours) outperforms SpatSC, LRR and SSC in the majority of cases.}
\label{Plot:MocapStats}
\end{figure*}

%
%
%
%
%

\section{Conclusion}

We have presented and evaluated a novel subspace clustering method, Ordered Subspace Clustering, that exploits the ordered nature of data. OSC produces more interpretable and accurate affinity matrices than other methods. We showed that this method generally outperforms existing state of the art methods in quantitative accuracy, particularly when the data is heavily corrupted with noise. Furthermore we have provided new optimisation schemes for OSC, which have guaranteed convergence, lower computational requirements and  can be computed in parallel.

\section*{Acknowledgements}

The research project is supported by the Australian Research Council (ARC) through the grant DP130100364.

\newpage

\section*{References}
\bibliography{references}
\bibliographystyle{elsarticle-num}

\end{document}